\documentclass{article}

\usepackage[accepted]{icml2026}
\usepackage{mathtools}
\usepackage{amssymb}
\usepackage{amsthm}
\usepackage{bm}
\usepackage{bbm}
\usepackage{mathrsfs}

\usepackage{graphicx}
\usepackage[svgnames]{xcolor}
\usepackage{booktabs}
\usepackage{multirow}
\usepackage{colortbl}
\usepackage{wrapfig}
\usepackage{threeparttable}
\usepackage{longtable} 
\usepackage{etoc}
\usepackage{makecell}
\usepackage{hhline}
\usepackage{tabularray}

\usepackage{framed}
\usepackage{ascmac}

\usepackage{nicefrac}
\usepackage{physics} 
\usepackage{algorithm}
\usepackage[linesnumbered,ruled,vlined,algo2e]{algorithm2e}
\SetAlgoNlRelativeSize{-1}

\makeatletter
\@ifclassloaded{beamer}{%
}{%
  \usepackage{enumitem}
  \usepackage[hypertexnames=false]{hyperref}
}
\makeatother
\usepackage[capitalize]{cleveref}

\usepackage{autonum}
\hypersetup{
  colorlinks,
  linkcolor={violet},
  citecolor={blue},
  urlcolor={black},
}

\usepackage{comment}

\usepackage{soul} 

\theoremstyle{plain}
\newtheorem{thm}{Theorem}[section]

\newtheorem{lem}[thm]{Lemma}
\newtheorem{cor}[thm]{Corollary}
\theoremstyle{definition}
\newtheorem{dfn}[thm]{Definition}

\theoremstyle{remark}
\newtheorem{rem}[thm]{Remark}

\newcommand{\fixcrefthm}[1]{%
  \AddToHook{env/#1/begin}{\crefalias{thm}{#1}}%
}

\fixcrefthm{thm}
\fixcrefthm{prop}
\fixcrefthm{lem}
\fixcrefthm{cor}
\fixcrefthm{dfn}
\fixcrefthm{ass}
\fixcrefthm{rem}
\fixcrefthm{exa}
\fixcrefthm{prty}
 
\crefname{equation}{}{} 
\crefname{function}{Function}{Functions}
\Crefname{equation}{Eq.}{Eqs.}
\crefname{thm}{Theorem}{Theorems}
\crefname{prop}{Proposition}{Propositions}
\crefname{lem}{Lemma}{Lemmas}
\crefname{cor}{Corollary}{Corollaries}
\crefname{dfn}{Definition}{Definitions}
\crefname{ass}{Assumption}{Assumptions}
\crefname{rem}{Remark}{Remarks}
\crefname{exa}{Example}{Examples}
\crefname{prty}{Property}{Properties}

\newcommand{\calA}{\mathcal{A}}
\newcommand{\calB}{\mathcal{B}}
\newcommand{\calC}{\mathcal{C}}
\newcommand{\calD}{\mathcal{D}}
\newcommand{\calE}{\mathcal{E}}
\newcommand{\calF}{\mathcal{F}}

\newcommand{\calK}{\mathcal{K}}

\newcommand{\calM}{\mathcal{M}}

\newcommand{\calS}{\mathcal{S}}
\newcommand{\calT}{\mathcal{T}}

\renewcommand{\hat}{\widehat}
\renewcommand{\tilde}{\widetilde}
\renewcommand{\epsilon}{\varepsilon}

\newcommand{\hatQ}{\hat{Q}}

\newcommand{\hatl}{\hat{\ell}}

\newcommand{\tilpsi}{\tilde{\psi}}

\newcommand{\tilQ}{\tilde{Q}}
\newcommand{\tilO}{\tilde{O}}

\newcommand{\stil}{\tilde{s}}
\newcommand{\till}{\tilde{\ell}}
\newcommand{\tilq}{\tilde{q}}

\newcommand{\unif}{\mathsf{Unif}}
\newcommand{\ber}{\mathsf{Ber}}

\newcommand{\KL}{\mathsf{KL}}
\newcommand{\kl}{\mathsf{kl}}
\newcommand{\Pidet}{\Pi_\mathsf{det}}
\newcommand{\Var}{\mathsf{Var}}
\newcommand{\Varmax}{\mathbb{V}^c}
\newcommand{\Uvar}{U_{\mathsf{Var}}}

\renewcommand{\ln}{\log}

\def\E{\mathbb{E}}
\def\P{\mathbb{P}}
\def\Q{\mathbb{Q}}
\def\I{\mathbb{I}}
\def\R{\mathbb{R}}
\def\Rp{\mathbb{R}_{> 0}}

\def\V{\mathbb{V}}

\newcommand{\ind}{\mathbbm{1}}

\DeclareMathOperator*{\argmax}{arg\,max}
\DeclareMathOperator*{\argmin}{arg\,min}
\DeclarePairedDelimiter{\brk}{[}{]}
\DeclarePairedDelimiter{\set}{\{}{\}}
\DeclarePairedDelimiter{\prn}{(}{)}
\DeclarePairedDelimiter{\nrm}{\|}{\|}
\DeclarePairedDelimiter{\inpr}{\langle}{\rangle}
\DeclarePairedDelimiter{\ceil}{\lceil}{\rceil}
\DeclarePairedDelimiter{\floor}{\lfloor}{\rfloor}
\newcommand{\relmiddle}[1]{\mathrel{}\middle#1\mathrel{}}

\newcommand{\one}{\mathbf{1}}
\newcommand{\Reg}{\text{\rm Reg}}
\newcommand{\Regg}{\mathsf{R}}
\newcommand{\regterm}{\textnormal{\textbf{reg-term}}}
\newcommand{\biasterm}{\textnormal{\textbf{bias-term}}}

\newcommand{\term}{\mathrm{\textbf{term}}}

\newcommand{\polylog}{\mathrm{polylog}}

\newcommand{\pio}{\mathring{\pi}}
\newcommand{\tilpi}{\tilde{\pi}}
\newcommand{\pist}{{\pi^\star}}
\newcommand{\qst}{{q^{\pio}}}

\allowdisplaybreaks[1]

\usepackage[textsize=tiny]{todonotes}

\begin{document}

\twocolumn[
  \icmltitle{Data- and Variance-dependent Regret Bounds for Online Tabular MDPs}

  \icmlsetsymbol{equal}{*}

  \begin{icmlauthorlist}
    \icmlauthor{Mingyi Li}{yyy}
    \icmlauthor{Taira Tsuchiya}{yyy,xxx}
    \icmlauthor{Kenji Yamanishi}{yyy}
  \end{icmlauthorlist}

  \icmlaffiliation{yyy}{Department of Mathematical Informatics, The University of Tokyo, Tokyo, Japan}
  \icmlaffiliation{xxx}{RIKEN AIP, Tokyo, Japan}

  \icmlcorrespondingauthor{Mingyi Li}{mingyi-mike@g.ecc.u-tokyo.ac.jp}

  \icmlkeywords{Machine Learning, ICML}

  \vskip 0.3in
]

\printAffiliationsAndNotice{}

\begin{abstract}
This work studies online episodic tabular Markov decision processes (MDPs) with known transitions and develops best-of-both-worlds algorithms that achieve refined data-dependent regret bounds in the adversarial regime and variance-dependent regret bounds in the stochastic regime. We quantify MDP complexity using a first-order quantity and several new data-dependent measures for the adversarial regime, including a second-order quantity and a path-length measure, as well as variance-based measures for the stochastic regime. To adapt to these measures, we develop algorithms based on global optimization and policy optimization, both built on optimistic follow-the-regularized-leader with log-barrier regularization. For global optimization, our algorithms achieve first-order, second-order, and path-length regret bounds in the adversarial regime, and in the stochastic regime, they achieve a variance-aware gap-independent bound and a variance-aware gap-dependent bound that is polylogarithmic in the number of episodes. For policy optimization, our algorithms achieve the same data- and variance-dependent adaptivity, up to a factor of the episode horizon, by exploiting a new optimistic $Q$-function estimator. Finally, we establish regret lower bounds in terms of data-dependent complexity measures for the adversarial regime and a variance measure for the stochastic regime, implying that the regret upper bounds achieved by the global-optimization approach are nearly optimal.
\end{abstract}

\section{Introduction}
\begin{table*}[t]
  \centering
  \begin{threeparttable}
  \caption{
  Comparison of regret upper bounds based on global optimization. 
  Here, $U = \sum_{s}\sum_{a\neq\pi^\star(s)}\frac{H^2\ln(T)}{\Delta(s,a)}$ and 
  $\Uvar = \sum_{s}\sum_{a\neq\pi^\star(s)}\frac{H\Varmax(s)\ln(T)}{\Delta(s,a)}$.
  We only display leading terms and omit logarithmic and lower-order factors.
  }
  \label{tab:regret-summary_global}
\begin{tabular*}{\linewidth}{lll}
\toprule
\textbf{Reference} & \textbf{Adversarial regime} & \textbf{Stochastic regime with adversarial corruption} \\
\midrule
\citet{zimin2013online}
  & $\sqrt{HSAT}$ &$\sqrt{HSAT}$ \\
\citet{lee2020bias}
  & $\sqrt{SAL^\star}$ &  $\sqrt{SAL^\star}$ \\
\multirow{1}{*}{\citet{jin2021best}}
  & $\sqrt{HSAT}$ 
  & $U_{\text{Jin}} + \sqrt{U_{\text{Jin}}\calC}$ \
    $(U_{\text{Jin}} = U + \frac{H^3S\ln (T)}{\min_{s,a\neq\pist(s)}\Delta(s,a)})$ \\
\midrule
\multirow{1}{*}{\textbf{This work} (\cref{thm:Occup_OPT})}
  & $\sqrt{SA\min\{L^\star,\,HT\!-\!L^\star,\,Q_\infty,\,V_1\}}$ 
  & $\min\set{\sqrt{SA(\V T+\calC)},\, U + \sqrt{U\calC} }$ \\
\multirow{1}{*}{\textbf{This work} (\cref{thm:Occup_OPT2})}
  & $\sqrt{SA\min\{L^\star,\,HT\!-\!L^\star,\,Q_\infty\}}$ 
  & $\min\set{\sqrt{SA(\V T+\calC)},\, \Uvar + \sqrt{\Uvar\calC} }$ \\
\bottomrule
\end{tabular*}
\end{threeparttable}
\end{table*}

\begin{table*}[t]
  \centering
\begin{threeparttable}
  \caption{
  Comparison of regret upper bounds based on policy optimization. 
  Here, $U = \sum_{s}\sum_{a\neq\pi^\star(s)}\frac{H^2\ln^2(T)}{\Delta(s,a)}$ and $\Uvar = \sum_{s}\sum_{a\neq\pi^\star(s)}\frac{H\Varmax(s)\ln^2(T)}{\Delta(s,a)}$. 
  We only display leading terms and omit logarithmic and lower-order factors.
  }
  \label{tab:regret-summary_policy}
\small
\begin{tabular*}{\linewidth}{lll}
\toprule
\textbf{Reference} & \textbf{Adversarial regime} & \textbf{Stochastic regime with adversarial corruption} \\
\midrule
\citet{luo2021policy}
  & $\sqrt{H^3SAT}$ & $\sqrt{H^3SAT}$ \\
\multirow{1}{*}{\citet{dann2023best}}
& $\sqrt{H^2SAL^\star}$
& $U + \sqrt{U \calC}$ \\
\midrule
\multirow{1}{*}{\textbf{This work} (\cref{thm:Policy_OPT})}
& $\sqrt{H^2SA\min\set*{L^\star, HT-L^\star, Q_\infty, V_1}}$
& $\min\set{\sqrt{H^2SA\prn*{\V T + \calC}},\, U + \sqrt{U \calC} }$ \\
\multirow{1}{*}{\textbf{This work} (\cref{thm:Policy_OPT2})}
& $\sqrt{H^2SA\min\set*{L^\star, HT-L^\star, Q_\infty}}$
& $\min\set{\sqrt{H^2SA\prn*{\V T + \calC}},\, \Uvar + \sqrt{\Uvar \calC} }$ \\
\bottomrule
\end{tabular*}
\end{threeparttable}
\end{table*}

We study online learning in finite-horizon episodic tabular Markov decision processes (MDPs), a standard model in reinforcement learning with broad applications, such as robotics~\cite{schulman2017proximal}, games~\cite{mnih2015human}, and healthcare decision-making~\citep{komorowski2018artificial}.
In this setting, a learner interacts with an environment over~$T$ episodes. In each episode, the learner selects a distribution over actions at each state, follows the trajectory induced by the algorithm, and observes the losses incurred along that trajectory.
The goal is to minimize regret, defined as the difference between the learner's cumulative expected loss and that of the best fixed policy in hindsight.

Tabular MDP algorithms are typically built on either \emph{global optimization} or \emph{policy optimization}. Global optimization solves an optimization problem over the set of all occupancy measures and can achieve minimax-optimal regret guarantees \cite{zimin2013online,jin2020learning}, but it can be computationally demanding for large MDPs. 
Policy optimization updates an action distribution at each state, which is often practical and computationally efficient, and the per-state updates can be viewed as instances of multi-armed bandits~\cite{shani2020optimistic,luo2021policy}.

The difficulty of tabular MDPs depends on how the underlying loss sequence is generated.
In the adversarial regime, where losses may be chosen arbitrarily, the minimax-optimal regret typically scales as $\tilde{O}(\sqrt{T})$ \cite{jin2020learning,luo2021policy}, where $T$ is the number of episodes.
By contrast, in the stochastic regime with i.i.d.~losses, one can achieve much faster gap-dependent regret, typically $O(\log T)$ \cite{simchowitz2019non}.

Recent work has shown that these regret upper bounds can be improved in various ways to better adapt to the structure of MDPs.
One line of work develops \emph{best-of-both-worlds} algorithms, which aim to achieve near-optimal regret in both the adversarial and stochastic regimes with a single algorithm~\citep{jin2020simultaneously,jin2021best,dann2023best}, thereby bridging the gap between the two regimes.
As another example, in the adversarial regime, one can derive regret bounds that depend on \emph{first-order complexity measures}: when the optimal policy has a small value function, this benign property yields improved guarantees~\citep{lee2020bias,dann2023best}. 
Furthermore, in the stochastic regime, variance-aware algorithms have been actively studied, including those with \emph{gap-independent} regret bounds~\citep{zanette2019tighter,zhang2021reinforcement} and those with \emph{gap-dependent} regret bounds~\citep{simchowitz2019non,chen2025sharp} with polylogarithmic dependence on $T$.

Despite these developments, existing algorithms remain unsatisfactory.
First, the adaptive guarantees above are typically achieved by different algorithms; in practice, the environment is unknown a priori, making it difficult to choose the most suitable algorithm in advance.
Moreover, in adversarial tabular MDPs, the only known data-dependent guarantees are first-order bounds. This contrasts with the broader online learning literature, which studies many other data-dependent guarantees, including \emph{second-order bounds} that adapt to the magnitude of loss fluctuations and \emph{path-length bounds} that adapt to how much losses change over time (see, \textit{e.g.,}~\citealt{cesa1996worst,allenberg2006hannan,neu2015first,wei2018more,bubeck2019improved}).
This naturally raises the following question:

\emph{Can we design a single best-of-both-worlds algorithm for tabular MDPs that achieves first-order, second-order, and path-length bounds in the adversarial regime and achieves variance-dependent bounds that are gap-independent or gap-dependent in the stochastic regime?}

To address this question, we focus on the known-transition setting. This lets us separate the difficulty of loss estimation, which is already nontrivial under bandit feedback.
Indeed, even when the transition kernel is known, losses are observed only along realized trajectories, and estimation errors at a state-action pair can affect downstream value estimates through the MDP dynamics. Therefore, unlike in multi-armed bandits, these errors cannot be controlled independently across state-action pairs; we also need to account for how they affect later estimates.
This makes it necessary to design loss and $Q$-function estimators whose biases are compatible with refined data-dependent complexity measures.

The unknown-transition setting is also important, but it is beyond the scope of this paper. Handling it would require data-dependent control of transition-estimation errors in addition to the loss-estimation errors studied here.
For global optimization, \citet{lee2020bias} derive first-order bounds under unknown transitions, but extending these techniques to second-order, path-length, variance-dependent, or best-of-both-worlds guarantees remains open.
For policy optimization, even first-order data-dependent guarantees under unknown transitions remain open~\cite{dann2023best}.

\subsection*{Contributions of This Paper}
In \Cref{sec:data_dependent_measures}, we begin by introducing new data-dependent complexity measures.
Specifically, we introduce a second-order quantity $Q_\infty$, which captures how large the losses can fluctuate, as well as a path-length measure $V_1$, which quantifies how much the losses change over time.
In addition, to quantify variance of MDPs in the stochastic regime, we introduce the occupancy-weighted variance $\V$ and the conditional occupancy-weighted variance~$\Varmax$ (see \cref{def:first-order_bound,def:second-order_bound,def:path-length_bound,def:total_variance,def:variance-to-go} for detailed definitions).

In \cref{sec:global_optimization}, we first develop global optimization algorithms whose regret adapts to the data-dependent complexity measures introduced above.
They achieve a regret upper bound of $\tilO(\sqrt{SA\min\set*{L^\star, HT - L^\star, Q_{\infty}, V_1}})$ in the adversarial regime, as well as a variance-aware gap-independent regret bound of $\tilO(\sqrt{SA\V T})$ and a variance-aware gap-dependent regret bound of $\polylog(T)$ in the stochastic regime\footnote{Precisely speaking, for both global optimization and policy optimization, whether we can attain a path-length bound or a variance-aware gap-dependent bound depends on how the loss prediction in OFTRL is chosen (see \Cref{tab:regret-summary_global,tab:regret-summary_policy}).}~(see \cref{thm:Occup_OPT,thm:Occup_OPT2}).
To our knowledge, these are the first second-order and path-length bounds for online episodic tabular MDPs.
Moreover, our gap-dependent bound in the stochastic regime improves over \citet{jin2021best} by adapting the variance and avoiding their additional dependence on ${1}/{\min_{s,a}\Delta(s,a)}$.
The algorithms are based on optimistic follow-the-regularized-leader (OFTRL) over the set of all occupancy measures with a log-barrier regularizer and an adaptive learning rate.
See \cref{tab:regret-summary_global} for a detailed comparison.

In \cref{sec:policy_optimization}, we develop policy optimization-based algorithms, which achieve a regret upper bound of $\tilO(\sqrt{H^2 S A\min\{L^\star,HT-L^\star,Q_\infty,V_1\}})$ in the adversarial regime, as well as a gap-independent variance-dependent regret bound of $\tilO(\sqrt{H^2SA\V T})$ and a gap-dependent variance-dependent regret bound of $\polylog(T)$ in the stochastic regime (see \cref{thm:Policy_OPT,thm:Policy_OPT2}).  
See \cref{tab:regret-summary_policy} for a detailed comparison.
The algorithms are also based on OFTRL with a log-barrier regularizer. 
A particularly notable ingredient is that, to correct a bias induced by the loss predictions in OFTRL, we introduce an even more optimistic $Q$-function estimation scheme than the one used in the existing best-of-both-worlds policy optimization by \citet{dann2023best} (see \Cref{subsec:po_alg} for details).

\begin{table}[t]
  \centering
  \begin{threeparttable}
  \caption{Regret lower bounds for online episodic tabular MDPs. Note that the lower bounds in terms of $L^\star$, $Q_\infty$, and $V_1$ are constructed for adversarial instances.
  }
  \label{tab:regret-lb}
  \begin{tabular*}{\linewidth}{ll}
    \toprule
    \textbf{Reference} & \textbf{Lower bound} \\
    \midrule
    \citet{zimin2013online} & $\Omega(\sqrt{HSAT})$ \\
    \midrule
    \multirow{2}{*}{\textbf{This work} (\cref{sec:lower_bounds})} & 
       $\Omega(\sqrt{SAL^\star})$, $\Omega(\sqrt{SA\,Q_\infty})$, \\  & $\Omega(\sqrt{H\,V_1})$, $\Omega(\sqrt{SA\,\V T})$ \\
    \bottomrule
  \end{tabular*}
  \end{threeparttable}
\end{table}

Finally, in \cref{sec:lower_bounds}, we derive data-dependent regret lower bounds of $\Omega(\sqrt{SAL^\star})$, $\Omega(\sqrt{SAQ_\infty})$, and $\Omega(\sqrt{HV_1})$, as well as a variance-dependent lower bound of $\Omega(\sqrt{SA\V T})$.
This implies that our regret upper bound for global optimization is nearly optimal in terms of $L^\star$, $Q_\infty$, and $V_1$.
See \cref{tab:regret-lb} for a summary.
Due to space limitations, we defer additional related work on MDPs, best-of-both-worlds algorithms, and data-dependent analyses in the adversarial and stochastic regimes to \cref{app:additional_related}.

\section{Preliminaries}
\paragraph{Notation.}
For $N\in\mathbb{N}$, let $[N]\coloneqq\{1,2,\ldots,N\}$.
Given a vector $x$, we write $\norm{x}_p$ to denote the $\ell_p$-norm for $p\in[1,\infty]$.
The set $\Delta(\calK)$ denotes the set of all probability distributions over the set $\calK$,
and the indicator function $\ind[\cdot]$ returns $1$ if the specified condition holds and $0$ otherwise.
For sets $\calA$ and $\calB$, we use $\calA^{\calB}$ to denote the set of all functions from $\calB$ to $\calA$.
Given functions $f$ and $g$ with $g(x)>0$, 
we write $f\lesssim g$ or $f=O(g)$ if there exists a constant $c>0$ such that $f(x)\leq cg(x)$ for all $x$ in the relevant domain and $\tilO(\cdot)$ hides logarithmic factors.

\paragraph{Episodic tabular MDPs.}
We consider a finite-horizon episodic tabular Markov Decision Process (MDP) $\mathcal{M}=(\calS,\calA, P, H, s_0)$, where $\calS$ is a finite state space with $S = \abs{\calS}$, $\calA$ is a finite action space with $A = \abs{\calA}$, and $P: \calS \times \calA \to \Delta(\calS)$ is a known transition function.
Here, $P(s' \mid s, a)$ specifies the probability of transitioning to state $s'$ after taking action $a$ in state $s$.
We adopt the standard layered MDP assumption \citep{neu2010online,jin2020learning,luo2021policy}
that the state space is layered into $H+1$ disjoint sets $\calS_0,\calS_1,\dots,\calS_H$, where $\calS_0=\{s_0\}$ is the initial layer and $\calS_H=\{s_H\}$ is a terminal absorbing layer. For simplicity, we exclude $s_H$ from $\calS$ and note that $H \le S$. Transitions are restricted to proceed from one layer to the next: for any $(s, a) \in \calS_h\times\calA$ with $h\in\{0,\dots, H-1\}$, the distribution $P(\cdot\mid s, a)$ is supported only on $\calS_{h+1}$. We write $h(s)$ for the layer index of state $s$.
The learning proceeds for $T$ episodes indexed by $t=1,\dots, T$. At the beginning of episode $t$, the environment chooses a loss function $\ell_t:\calS\times\calA\to[0,1]$.
A policy $\pi:\calS\to\Delta(\calA)$ assigns a distribution over actions to each state $s$, with $\pi(a\mid s)$ denoting the probability of action $a$ at state $s$. The set of all stochastic policies is denoted by $\Pi=\Delta(\calA)^{\calS}$, and the set of deterministic policies by $\Pidet=\calA^{\calS}$. When $\pi$ is deterministic, we write $\pi(s)\in\calA$ for the unique action chosen in $s$. We assume $T \ge \max\{2,S,A\}$ for convenience.\footnote{The assumptions $T \ge S$ and $T \ge A$ are not essential. If they do not hold, the analysis remains valid with $\ln(T)$ replaced by $\ln(SAT)$ or $\ln(AT)$.}

For a policy $\pi$ and a loss function $\ell$, we define the value functions recursively with the terminal condition $V^\pi(s_H;\ell)=0$. 
The state value function $V^\pi(s; \ell)$ and the state-action value function $Q^\pi(s,a;\ell)$ (a.k.a.~$Q$-function) are defined as
$V^\pi(s;\ell) = \E_{a\sim \pi(\cdot\mid s)}\brk{Q^\pi(s,a;\ell)}$
and $Q^\pi(s,a;\ell) = \ell(s,a) + \E_{s'\sim P(\cdot\mid s,a)}\brk{V^\pi(s';\ell)}$.
Here we may overload the notation by allowing a general function $m:\calS\times\calA\to\R$ to replace the loss function $\ell$, and write $V^\pi(s;m)$ and $Q^\pi(s, a;m)$ accordingly.

In each episode $t \in [T]$, the learner chooses a policy $\pi_t$ based on past observations, executes it from the initial state $s_0$, and observes the losses along the realized trajectory $\{(s_{t,h}, a_{t,h},\ell_t(s_{t,h}, a_{t,h}))\}_{h=0}^{H-1}$.
The goal of the learner is to minimize the regret given by
\begin{equation}
\Reg_T = \max_{\pi \in \Pi} \, \E\brk*{\sum_{t=1}^T V^{\pi_t}(s_0;\ell_t) - \sum_{t=1}^T V^{\pi}(s_0;\ell_t)},
\end{equation}
and denote by $\pio$ one of the optimal fixed comparators for the regret definition.

For a policy $\pi$ and a state-action pair $(s,a)$, the occupancy measure $q^\pi(s,a)$ is the probability of visiting $(s,a)$ within an episode under $\pi$. We also use $q^\pi(s'\mid s,a)$ and $q^\pi(s',a'\mid s,a)$ for the corresponding conditional occupancy measures given that $(s,a)$ has already been visited (note that these quantities are zero whenever $h(s')<h(s)$).
For each state $s$, we set $q^\pi(s) \coloneq \sum_{a} q^\pi(s,a)$, so that $q^\pi(s,a) = q^\pi(s)\pi(a\mid s)$.

\paragraph{Additional notation.}
We denote $\E_t[\cdot] = \E[\cdot\mid \calF_{t-1}]$, where $\{\calF_t\}_{t\ge 0}$ is the natural filtration generated by all observations up to the end of episode $t$.
Let $\I_t(s, a) = \ind\brk*{(s_{t,h},a_{t,h})=(s,a), \exists \, h\in\{0,\ldots,H-1\}}$ be
the indicator function representing whether the state-action pair $(s, a)$ is visited under the policy $\pi_t$ used in episode $t$
and transition kernel $P$, and let $\I_t(s)=\sum_{a} \I_t(s,a)$. We also define the visitation counts $N_t(s,a)\coloneqq \sum_{\tau=1}^t \I_\tau(s,a)$.
We write $\ell_t(h)\in[0,1]^{S_h\times A}$ for the restriction of $\ell_t$ to layer $h$ in episode $t$, and use the same notation for other functions defined on $\calS\times\calA$.

\subsection{Regimes of Environments}\label{subsec:regimes}
We consider three regimes for how the loss functions $\ell_1,\ldots,\ell_T$ are generated.
In the adversarial regime, we make no generative assumption. At the beginning of episode $t$, the environment arbitrarily selects a loss function $\ell_t \in [0,1]^{S \times A}$.
Specifically, $\ell_t$ may adapt to the past history but not to the learner’s fresh randomness in the current episode. In the stochastic regime, the loss functions $\ell_1,\ldots,\ell_T$ are sampled i.i.d.~from a fixed and unknown distribution $\calD$.

The \emph{stochastic regime with adversarial corruption} generalizes both the stochastic and adversarial regimes.
Let $\ell'_1,\ldots,\ell'_T$ be sampled i.i.d.~from a fixed and unknown distribution $\calD$, and let the observed loss functions $\ell_1,\ldots,\ell_T$ be arbitrary corruptions of $\ell'_1,\ldots,\ell'_T$. We quantify the total corruption level $\calC\coloneqq \E\brk[\big]{\sum_{t=1}^T\sum_{h=0}^{H-1}\|\ell'_t(h)-\ell_t(h)\|_\infty} \in [0, HT]$. In particular, when $\calC=0$, the stochastic regime with adversarial corruption reduces to the stochastic regime, whereas when $C = \Omega(T)$ it coincides with the adversarial regime.
For each state-action pair $(s,a)$, let $\mu(s,a)\coloneqq \E_{\ell'\sim\calD}[\ell'(s,a)]$ and $\sigma^2(s,a)\coloneqq \E_{\ell'\sim\calD} \brk{ (\ell'(s,a)-\mu(s,a))^2 }$ denote the mean and variance, respectively.
Let $\pi^\star$ be an optimal policy for the uncorrupted mean loss function $\mu$, and define the suboptimality gap $\Delta\colon\calS\times\calA\to[0,H]$ as $\Delta(s,a)\coloneqq Q^{\pi^\star}(s,a;\mu)-\min_{a'\in\calA}Q^{\pi^\star}(s,a';\mu)$.

\subsection{Optimistic Follow-the-Regularized-Leader}
Our proposed algorithms are based on the \emph{optimistic follow-the-regularized-leader} (OFTRL) framework \citep{chiang2012online,rakhlin2013online,steinhardt2014adaptivity}, which has also been adopted in several existing studies \citep{wei2018more,ito2022adversarially}.

Here, we present OFTRL in the standard online linear optimization setting over a convex set\footnote{In our applications, $\calK=\Omega(P)$ for global optimization and $\calK=\Delta(\calA)$ for policy optimization.} $\calK$.
At each round $t$, the learner outputs $p_t\in\calK$ and incurs linear loss $\langle p_t,c_t\rangle$, where $\{c_t\}_{t=1}^T$ are  loss vectors.\footnote{In our applications, $c_t$ serves as a loss estimator, namely $c_t=\hat{\ell}_t$ for global optimization and $c_t=\hat{Q}_t - B_t$ for policy optimization.}
The OFTRL algorithm with differentiable regularizers $\{ \psi_t \}_{t=1}^T$ and loss predictions $\{ m_t \}_{t=1}^T$ chooses $p_t$ in $\calK$ by
\begin{equation}
    p_t 
    = \argmin_{p \in \calK}
    \set*{
        \inpr*{ p, \sum_{\tau=1}^{t-1} c_\tau + m_t } + \psi_t(p)
    }.
    \label{eq:OFTRL_main}
\end{equation}
The FTRL algorithm is recovered as the special case when $m_t = 0$ for all $t$ in \cref{eq:OFTRL_main}. 
The sequence $\{ m_t \}_{t=1}^T$ serves as an \emph{optimistic prediction} of the upcoming loss vector. When the prediction is accurate, the algorithm improves regret guarantees, while in the worst case, the regret bound remains of the same order as FTRL.

We consider two schemes to obtain $\set{m_t}_{t=1}^T$ in \cref{eq:OFTRL_main}, used for both global optimization and policy optimization. 
The first scheme is based on a gradient descent approach inspired by \citet{ito2021parameter,tsuchiya2023further}:
we initialize $m_1(s,a) = 1/2$ for all $(s,a)$ and update
\begin{align}
    &m_{t+1}(s,a) = \\
    &\begin{cases}
        (1 - \xi)\,m_t(s,a) + \xi\, \ell_t(s,a) & \text{if $\I_t(s,a) = 1$},\\
        m_t(s,a) & \text{if $\I_t(s,a) = 0$},
    \end{cases} 
    \label{def:predictor_sequence}
\end{align}
where $\xi \in (0,1/2)$ is a step size. We set $\xi=1/4$ throughout this paper.
This approach is useful for obtaining path-length regret bounds depending on \cref{def:path-length_bound}.
The second scheme is based on the empirical mean predictor:
\begin{equation}
    m_t(s,a) = \frac{\sum_{\tau=1}^{t - 1} \I_\tau(s,a)\ell_\tau(s,a)}{\max\set{1, N_{t-1}(s,a)}}
    .
    \label{def:predictor_sequence2}
\end{equation}
We will show that this is useful to obtain variance-aware gap-dependent regret bounds depending on \cref{def:variance-to-go}.
\section{Complexity Measures in Online MDPs}\label{sec:data_dependent_measures}
This section introduces several complexity measures for online tabular MDPs.
In our analysis, we derive guarantees that scale with these data-dependent complexity measures, and our algorithms do not need to know these quantities in advance.

\subsection{Complexity in the Adversarial Regime}
The first-order complexity $L^\star \in [0, HT]$ is defined as
\begin{equation}
L^\star \coloneq \min_{\pi\in\Pi}\E\brk*{\sum_{t=1}^T V^\pi(s_0;\ell_t)},
\label{def:first-order_bound}
\end{equation}
which is the cumulative loss of the best fixed policy in hindsight, sometimes referred to as the small-loss quantity, and investigated in \citet{lee2020bias,dann2023best}.

We further introduce new complexity measures for online MDPs.
The following two can be seen as extensions of those used in multi-armed bandits.
The second-order complexity $Q_\infty \in [0, HT / 4]$ is defined as
\begin{equation}
Q_\infty
\coloneq
\min_{\ell^\star \in[0,1]^{S\times A}}\E\brk*{\sum_{t=1}^T\sum_{h=0}^{H-1}\|\ell_t(h)-\ell^\star(h)\|_\infty^2},
\label{def:second-order_bound}
\end{equation}
which becomes small when the losses stay close to a single baseline $\ell^\star$ over time.
The path-length (or total variation) complexity $V_1 \in [0, S A (T - 1)]$ is defined as
\begin{equation}
V_1 \coloneq \E\brk*{\sum_{t=1}^{T-1}\|\ell_{t+1}-\ell_t\|_1},
\label{def:path-length_bound}
\end{equation}
which becomes small when the loss sequence changes slowly over episodes.

\subsection{Complexity in the Stochastic Regime}
We also introduce variance-based complexity measures for the stochastic regime.
The \emph{occupancy-weighted variance} $\V \in [0, H / 4]$ is defined as
\begin{equation}
\V \coloneq 
\max_{\pi\in\Pi}\sum_{s,a} q^{\pi} (s,a)\sigma^2(s,a),
\label{def:total_variance}
\end{equation}
which is the stochastic noise weighted by the occupancy measure.
The \emph{conditional occupancy-weighted variance} $\Varmax(s) \in [0,H/4]$ at state~$s$ is defined as
\begin{equation}
\Varmax(s) \coloneq \max_{\pi\in\Pi,a\in\calA}\sum_{s',a'} q^{\pi}(s',a'\mid s,a)\sigma^2(s',a'),
\label{def:variance-to-go}
\end{equation}
which is the remaining noise after reaching $s$ maximized over the first action $a$ at~$s$.

There are known variance-dependent complexity measures in the literature.
The maximum (unconditional) total variance~\citep{zhou2023sharp,zhang2024settling} and maximum conditional total variance~\citep{chen2025sharp} are defined as
\begin{align}
\Var_{\max} &\coloneq \max_{\pi\in\Pi} \sum_{s,a} q^{\pi}(s,a)\Var^\star(s,a), \label{def:related_uncon}\\
\Var_{\max}^c &\coloneq \max_{\pi\in\Pi,s\in\calS} 
\sum_{s',a'} \bar q^\pi(s',a'\mid s)\Var^\star(s',a'),
\label{def:related_con}
\end{align}
where 
$
\Var^\star(s,a) \coloneq \sigma^2(s,a) \!+\! \Var_{s'\sim P(\cdot\mid s,a)}[V^{\pist}(s')]
$ \citep{zanette2019tighter,simchowitz2019non} and $\bar q^\pi(s',a' \mid s)$ denotes the occupancy of $(s',a')$ over the entire trajectory conditioned on visiting state $s$.
These complexity measures were introduced in the context of a value-based approach for the stochastic regime with unknown transitions.
In our setting, $\V$ and $\Varmax$ are analogous to $\Var_{\max}$ and $\Var_{\max}^c$, respectively. Since we consider known transitions, the second term in $\Var^\star(s,a)$ is unnecessary and can be omitted. 
Moreover, $\Varmax$ is defined using conditional occupancy measures $q^{\pi}(s',a'\mid s,a)$ and captures variance only after visiting $(s,a)$, whereas $\Var_{\max}^c$ aggregates variance over the entire trajectory by conditioning on visiting state $s$.
As a consequence, our variance measures are $H^2$-sharper than those based on $\Var^\star(s,a)$. 
Further discussion is deferred to \cref{app:comparison_variance}.
\begin{algorithm*}[t]
\caption{Global Optimization with Data- and Variance-dependent Bounds} 
\label{alg:OFTRL_GO}
\nl \textbf{Input: }
MDP $\mathcal{M}=(\calS,\calA, P, H, s_0)$, initial learning rate $\frac{1}{\eta_1(s,a)} = \frac{1}{\eta_1} = 2H$, 
loss prediction $m_t$ chosen as in \cref{def:predictor_sequence} or \cref{def:predictor_sequence2}.

\nl \For{$t=1,2,\ldots$} {

    \nl
    Compute the occupancy measure $q^{\pi_t} \in \Omega(P)$ by \nllabel{line:compute_occup}
    \begin{align}\label{eq:epoch_FTRL}
    \\[-22pt]
	    q^{\pi_t} = \argmin_{q\in \Omega(P)} \set*{ \inpr*{q,  \sum_{\tau=1}^{t-1}\hat\ell_\tau + m_t} + \psi_t(q) }
        ,\quad
        \psi_t(q) = \sum_{s,a} \frac{1}{\eta_t(s, a)}\ln\prn*{\frac{1}{q(s,a)}}.
    \\[-22pt]
    \end{align}
    \nl Compute policy $\pi_t$ from $q^{\pi_t}$ by
    $\pi_t(a \mid s) \propto q^{\pi_t}(s,a)$, 
    and obtain a trajectory $\{(s_{t,h}, a_{t,h},\ell_t(s_{t,h}, a_{t,h}))\}_{h=0}^{H-1}$.
    
    \nl Compute the loss estimator $\hat\ell_t(s,a)$ by \nllabel{line:loss_estimator}
    \begin{align}
     \\[-22pt]
        \hat\ell_t(s,a) = m_t(s,a) + \frac{\I_t(s,a)(\ell_t(s,a) - m_t(s,a))}{q^{\pi_t}(s,a)}
        . \label{eq:loss_estimate_occup}
    \\[-22pt]
    \end{align}
    \nl Update the learning rate $\eta_{t+1}(s,a)$ by
    \begin{align}
     \\[-22pt]
        \frac{1}{\eta_{t + 1}(s,a)}
        &= \frac{1}{\eta_{t}(s,a)} + \frac{\eta_t(s,a)\zeta_t(s,a)}{\ln(T)}, \label{eq:eta_update_occup}
    \\[-22pt]
    \end{align} 
    \nl where
        $\zeta_t(s,a) = q^{\pi_t}(s,a)^2\min\set[\big]{(\hat\ell_t(s,a)-m_t(s,a))^2, (\hat\ell_t(s,a) + g_t(s,a)-m_t(s,a))^2}$  with $g_t$ defined in \cref{def:loss-shifting}.
        
     \nl Update the loss prediction $m_{t+1}(s,a)$ by \cref{def:predictor_sequence} or \cref{def:predictor_sequence2}.
}
\end{algorithm*}
\section{Global Optimization}\label{sec:global_optimization}
This section presents an occupancy-measure-based algorithm designed to achieve the data-dependent and variance-adaptive regret guarantees stated in \cref{thm:Occup_OPT,thm:Occup_OPT2}.

\subsection{Algorithm}
In global optimization, we optimize directly over occupancy measures. Let $\Omega(P)$ denote the convex set of valid occupancy measures induced by the transition kernel $P$.
In each episode~$t$, we run OFTRL over $\Omega(P)$ with log-barrier regularizers and loss predictions. 
Our design is inspired by \citet{jin2021best} but adapted to the OFTRL framework, and thus the loss estimator and the corresponding loss-shifting function differ from their FTRL-based construction. The complete algorithm is described in \cref{alg:OFTRL_GO}.

In particular, we run OFTRL over $\Omega(P)$ with $c_t=\hat{\ell}_t$ in \cref{eq:OFTRL_main}, leading to the occupancy-measure update in \cref{eq:epoch_FTRL} (\cref{line:compute_occup}) with the log-barrier regularizer.
Here, we use the optimistic importance-weighted estimator in \cref{eq:loss_estimate_occup} (\cref{line:loss_estimator}), where $\set{m_t}_{t=1}^T$ is chosen as in \cref{def:predictor_sequence} or \cref{def:predictor_sequence2}.
This estimator is unbiased in the sense that $\E_t\brk[\big]{\hatl_t(s,a)} = \ell_t(s,a)$.

To obtain a polylogarithmic regret in the stochastic regime, we use the loss-shifting technique of \citet{jin2021best}. 
Since the stability of OFTRL is controlled by the shifted loss $\till_t \coloneq \hat{\ell}_t - m_t$, we construct the following loss-shifting function
\begin{equation}
g_t(s,a) = Q^{\pi_t}(s,a;\till_t)-V^{\pi_t}(s;\till_t)-\till_t(s,a). \label{def:loss-shifting}
\end{equation}
With this shifting function, the learner equivalently runs OFTRL with the advantage function $Q^{\pi_t}(s,a;\till_t)-V^{\pi_t}(s;\till_t)$, which enables a self-bounding regret analysis in the stochastic regime. Moreover, when $m_t$ is the empirical-mean predictor in \cref{def:predictor_sequence2}, the same shifting construction allows us to control the resulting variance term by $\Varmax$.

\subsection{Regret Upper Bounds}
With the optimistic estimator, shifted losses, and adaptive log-barrier learning rates, we state the following theorem. We defer all technical lemmas and proofs to \cref{app:global_opt_proofs}.
\begin{thm}\label{thm:Occup_OPT}
\cref{alg:OFTRL_GO} with $m_t$ in \cref{def:predictor_sequence} guarantees
\begin{align}
    \Reg_T &\lesssim \sqrt{SA\ln(T)\min\set*{L^\star, HT - L^\star, Q_{\infty}, V_1}} \\
    &\qquad + HSA\ln T.
\end{align}
Under the stochastic regime with adversarial corruption, it simultaneously ensures
\begin{align}
    \Reg_T &\lesssim \sqrt{ SA\ln(T)\prn*{\V T + \calC}}+ HSA \ln T,\\
    \Reg_T &\lesssim  U + \sqrt{U\calC} + HSA \ln T,
\end{align}
where $U = \sum_{s}\sum_{a\neq\pi^\star(s)}\frac{H^2\ln(T)}{\Delta(s,a)}$.
\end{thm}
Our first-order, second-order, and variance-aware gap-independent bounds are minimax optimal up to logarithmic factors (see lower bounds in \cref{thm:adversarial_data_lb,thm:variance_lower_bound}) and also recover the worst-case dependence $\tilde{O}(\sqrt{HSAT})$ in the adversarial regime \cite{zimin2013online}.
Furthermore, our gap-dependent guarantee improves over \citet{jin2021best} by avoiding their additional dependence on ${1}/{\min_{s,a}\Delta(s,a)}$.

\begin{thm}\label{thm:Occup_OPT2}
\cref{alg:OFTRL_GO} with $m_t$ in \cref{def:predictor_sequence2} guarantees
\begin{align}
    \Reg_T &\lesssim \sqrt{SA\ln(T)\min\set*{L^\star, HT - L^\star, Q_{\infty}}} \\
    &\qquad  + HSA\ln(T).
\end{align}
Under the stochastic regime with adversarial corruption, it simultaneously ensures
\begin{align}
    \Reg_T &\lesssim \sqrt{SA\ln(T)\prn*{\V T + \calC}} + HSA\ln(T),\\
    \Reg_T &\lesssim \Uvar + \sqrt{\Uvar \calC} + \sqrt{H S^2 A^2 \calC}\ln(T)\\
    &\qquad + H^{\frac12}S^{\frac32}A^{\frac32}\ln^{\frac32}(T),
\end{align}
where $\Uvar = \sum_{s}\sum_{a\neq\pi^\star(s)}\frac{H\Varmax(s)\ln(T)}{\Delta(s,a)}$.
\end{thm}
\begin{rem}\label{rem:Occup_OPT2}
If the uncorrupted losses are generated independently and are uncorrelated across layers, the variance-aware gap-dependent bound in \cref{thm:Occup_OPT2} improves by a factor of $H$ to $\Uvar = \sum_{s}\sum_{a\neq\pi^\star(s)} \frac{\Varmax(s)\ln(T)}{\Delta(s,a)}$.
\end{rem}
\begin{algorithm*}[t]
\caption{Policy Optimization with Data- and Variance-dependent Bounds} 
\setcounter{AlgoLine}{0}
\label{alg:OFTRL_PO}
\nl \textbf{Input:} MDP $\mathcal{M}=(\calS,\calA, P, H, s_0)$, regularizer $\psi_t(\pi(\cdot\mid s)) = \sum_a \frac{1}{\eta_t(s, a)}\ln(1 / \pi(a\mid s))$,
exploration rate $\gamma_t = \frac{\sqrt{HS}}{t}$, initial learning rate  ${1}/{\eta_1(s,a)} = {1}/{\eta_1} = 180 H^3$, 
loss prediction $m_t$ chosen as in \cref{def:predictor_sequence} or \cref{def:predictor_sequence2}.

\nl \For{$t=1,2,\ldots$} {

    \nl
    Compute the policy $\pi_t(\cdot\mid s)$ at each state $s \in \calS$ by \nllabel{line:optimize_policy}
    \begin{equation}
        \pi_t(\cdot\mid s) = \argmin_{\pi(\cdot\mid s)\in\Delta(\calA)} \set[\bigg]{\inpr[\bigg]{\pi(\cdot\mid s), \sum_{\tau=1}^{t-1}\prn*{\hat{Q}_\tau(s, \cdot) - B_\tau(s, \cdot)} + Q^{\pi_t}(s, \cdot; m_t)} + \psi_t(\pi(\cdot\mid s))} 
        .
        \label{eq:opt_pi_OFTRL}
    \end{equation}
    \nl Compute $Y_t \gets \ind\brk*{\max_{s,a}\frac{\eta_t(s,a)}{q_t(s)}\le \frac{1}{18\sqrt{H^3S}}}$, where $q_t(s) = q^{\pi_t}(s) + \gamma_t$. If $Y_t=0$, we insert a virtual episode and shift the indices of subsequent real episodes.  \nllabel{line:virtual_episode}\\
    \nl
    If $Y_t=1$ (real episode), obtain a trajectory $\{(s_{t,h},a_{t,h},\ell_t(s_{t,h},a_{t,h}))\}_{h=0}^{H-1}$.\\
    \nl Let $q_t(s) = q^{\pi_t}(s) + \gamma_t$, $L_{t,h} = \sum_{h'=h}^{H-1} \ell_t(s_{t,h'}, a_{t,h'})$, $M_{t,h} = \sum_{h'=h}^{H-1} m_t(s_{t,h'}, a_{t,h'})$, and \nllabel{line:Q_estimator}
    \begin{equation}
        \hat{Q}_t(s,a) = Q^{\pi_t}(s,a;m_t) + \frac{\I_t(s,a)(L_{t,h(s)} - M_{t,h(s)})}{q_t(s)\pi_t(a \mid s)}Y_t - \frac{\gamma_t H}{q_t(s)}.
        \label{eq:Q-estimate_policy}
    \vspace{-3pt}
    \end{equation}
    \nl Let $(s_t^\dagger, a_t^\dagger)\in\argmax_{s,a}\frac{\eta_t(s,a)}{q_t(s)}$ (break ties arbitrarily), and update the learning rates $\eta_{t+1}(s,a)$,  
    \begin{equation}
        \frac{1}{\eta_{t+1}(s,a)} = 
        \begin{cases} 
        \frac{1}{\eta_{t}(s,a)} + \frac{\eta_t(s,a)\zeta_t(s,a)}{q_t(s)^2\ln(T)}
        &\text{if $t$ is a real episode}, \\[3pt]
        \frac{1}{\eta_t(s,a)}\left(1+\frac{\ind\{(s^\dagger_t, a^\dagger_t) = (s,a)\}}{324H\ln(T)}\right)
        &\text{if $t$ is a virtual episode},
    \end{cases}\label{eq:eta_update_policy}
    \end{equation}
    \vspace{-5pt}
    \begin{equation}
    \text{where}\qquad\zeta_t(s,a) = (\I_t(s,a) - \pi_t(a\mid s) \I_t(s))^2 (L_{t,h(s)} - M_{t,h(s)})^2. \label{eq:zeta_policy}
    \end{equation}
    \nl Compute $b_t(s)$ by \cref{def:bonus_term}, and then compute $B_t(s,a)$ by \cref{def:dilated_bonus}.\\
    \nl Compute loss prediction $m_{t+1}(s,a)$ by \cref{def:predictor_sequence} or \cref{def:predictor_sequence2}.
}
\end{algorithm*}
\section{Policy Optimization}\label{sec:policy_optimization}
This section presents a policy-optimization algorithm with log-barrier regularization that attains data- and variance-dependent regret bounds.
Policy optimization can be viewed as solving a multi-armed bandit problem at each state, with $\pi_t(\cdot\mid s)$ as the action distribution.
This is formalized by the performance-difference lemma \citep{kakade2002approximately}, which implies
$\Reg_T=\E\brk[\Big]{\sum_{s} \sum_{t} \qst(s)\inpr{\pi_t(\cdot\mid s) - \pio(\cdot\mid s), Q^{\pi_t}(s,\,\cdot\,;\ell_t)}}$
and motivates using the $Q$-function as the loss in OFTRL.

\subsection{Algorithm}\label{subsec:po_alg}
Here, we present the policy-optimization procedure in \cref{alg:OFTRL_PO}.
For each state $s$, we run OFTRL as in \cref{eq:opt_pi_OFTRL} with the log-barrier regularizer
\begin{equation}
\psi_t(\pi(\cdot\mid s))=\sum_{a}\frac{1}{\eta_t(s,a)}\ln\prn*{\frac{1}{\pi(a\mid s)}}, \label{eq:regularizer_policy}
\end{equation}
where $\eta_t(s,a)>0$ are time-varying, data-dependent learning rates updated via \cref{eq:eta_update_policy}.

In \cref{line:optimize_policy}, given the policy $\pi_t$ and the loss prediction $m_t$, we can compute $Q^{\pi_t}(s,a;m_t)$ by backward dynamic programming and then determine $\pi_t(\cdot\mid s)$ for each state, starting from the last layer and proceeding backward over $h=H-1,\dots,0$.
Following \citet{dann2023best}, we choose the exploration rate $\gamma_t = \sqrt{HS}/t$, in order to achieve a polylogarithmic regret in the stochastic regime. 
The loss prediction $m_t$ is updated by the gradient descent in \cref{def:predictor_sequence} or empirical mean in \cref{def:predictor_sequence2}, which allows us to obtain data-dependent regret bounds in the adversarial regime and variance-dependent regret bounds in the stochastic regime.
In what follows, we describe three key technical components of the algorithm: the dilated bonus, virtual episodes, and a novel optimistic $Q$-function estimator.

\paragraph{Dilated bonus.}
In policy optimization, updates are performed locally at each state, which can lead to insufficient exploration.
To enforce global exploration, \citet{luo2021policy} introduced a dilated exploration bonus $B_t(s,a)$ that is constructed in the same form as a $Q$-function, 
\begin{align}
&B_t(s,a) = b_t(s)\\
&\quad+ \prn*{1 + \frac{1}{H}}\, \E_{s'\sim P(\cdot\mid s,a), a'\sim \pi_t(\cdot\mid s')} \brk*{B_t(s',a')} \label{def:dilated_bonus}.
\end{align}
Intuitively, $b_t(s)$ is chosen to scale inversely with the visitation probability $q^{\pi_t}(s)$, so that rarely visited states receive larger exploration incentives (see, \textit{e.g.}, \cref{def:bonus_term} or \citealt[Eq.~(8)]{luo2021policy}). The resulting bonus $B_t(s,a)$ has the same recursive structure as a $Q$-function and is subtracted from the $Q$-estimate in the policy update as in \cref{eq:opt_pi_OFTRL}. 
This construction of bonus $B_t(s,a)$ yields the following lemma, which plays a key role in achieving the best-of-both-worlds guarantees in \citet{dann2023best} and this work.
\begin{lem}[{\citealt[Lemma B.2]{luo2021policy}}]
\label{lem:dilated_bonus_main}
    Suppose that $b_t(s)$ is a nonnegative loss function, $B_t$ satisfies \cref{def:dilated_bonus} for all $(s,a)$, 
    and that, for each $s \in \calS$ and for some $J(s) \ge 0$,
    \begin{align}
        &\E\brk[\Bigg]{\sum_{t,a} \prn*{\pi_t(a\mid s) - \pio(a\mid s) } \prn*{Q^{\pi_t}(s,a;\ell_t) -  B_t(s,a)}}\\
        &\leq  J(s) \!+\! \E\brk*{\sum_{t} b_t(s) \!+\! \frac{1}{H}  \sum_{t,a} \pi_t(a\mid s)B_t(s,a)}. \label{eq:key_lemma_ineq_main}
    \end{align}
    Then, 
    $\Reg_T \leq \sum_s \qst(s) J(s) + 3 \E\brk*{\sum_{t = 1}^T V^{\pi_t}(s_0;b_t)}$.
\end{lem}
The factor $(1+1/H)$ in \cref{def:dilated_bonus} slightly inflates the propagated bonus, so that the error due to the bonus term can be absorbed into
$\frac{1}{H}\sum_{t,a}\pi_t(a\mid s)B_t(s,a)$ in \cref{eq:key_lemma_ineq_main}. Consequently, the overall exploration overhead is bounded by a constant factor of the learner’s own occupancy term $\sum_{t=1}^T V^{\pi_t}(s_0; b_t)$, as formalized in \cref{lem:dilated_bonus_main}.

To make \cref{eq:key_lemma_ineq_main} hold, we use the local bonus $b_t(s)$ given by
\begin{equation}
\!\!\!\!b_t(s) = 6\!\sum_{a}\prn*{\!\frac{1}{\eta_{t+1}(s,a)}\!-\!\frac{1}{\eta_t(s,a)}\!}\! \ln (T) + \frac{5 \gamma_t H}{q_t(s)}.
\label{def:bonus_term}
\end{equation}
The first term is the OFTRL-regret overhead induced by the use of an adaptive learning rate, and the second term arises from the optimism in the $Q$-estimation (explained later).
For further details and intuition behind the bonus term, we refer the reader to \citet{luo2021policy,dann2023best}.

\paragraph{Virtual episodes.}
To motivate the introduction of virtual episodes (\cref{line:virtual_episode}), we first discuss the learning rate design in the OFTRL algorithm with a log-barrier regularizer in~\cref{eq:regularizer_policy}. For a fixed state $s$, the regret of this algorithm can be roughly bounded by
\begin{equation}
    \underbrace{\sum_{t,a}\prn*{\!\frac{1}{\eta_{t+1}(s,a)}\!-\!\frac{1}{\eta_t(s,a)} \!}\ln (T)}_{\text{penalty-term}}+\underbrace{\sum_{t,a}\frac{\eta_t(s,a)\zeta_t(s,a)}{q_t(s)^2} }_{\text{stability-term}},
\end{equation}
where $\zeta_t(s,a)$ is the data-dependent term defined in~\cref{eq:zeta_policy}.
Hence, it is natural to choose a data-dependent learning rate like \citet{dann2023best} that directly balances these two terms, namely
$\frac{1}{\eta_{t+1}(s,a)} = \frac{1}{\eta_t(s,a)} + \frac{\eta_t(s,a)\zeta_t(s,a)}{q_t(s)^2\ln(T)}$.
With this update, the penalty and stability terms evolve on the same scale.
However, to upper bound the error term induced by the bonus by 
$\frac{1}{H}\sum_{t,a} \pi_t(a\mid s) B_t(s,a)$, the analysis additionally requires 
$\eta_t(s,a)\pi_t(a \mid s)\,B_t(s,a)\lesssim\frac{1}{H}$.
Since $B_t(s,a)$ is of order $1/q_t(s)^2$ from \cref{def:bonus_term}, it becomes large when $q_t(s)$ is
small, and the above inequality is not guaranteed by the learning rate schedule alone.

Therefore, following \citet{dann2023best}, we enforce the above condition by inserting virtual episodes (\cref{line:virtual_episode}).
At the start of episode $t$, if $\max_{s,a} \eta_t(s,a)/q_t(s)$ is larger than $1/(18\sqrt{H^3S})$, we set $Y_t=0$ and declare the episode virtual.
In a virtual episode, we set $\ell_t(s,a)=0$ for all $(s,a)$. We then shrink the learning rate at the state-action pair
$(s_t^\dagger, a_t^\dagger)\in\argmax_{s,a}\frac{\eta_t(s,a)}{q_t(s)} $ by a constant factor $1+\frac{1}{324H\ln(T)}$,
and shift the indices of real episodes. The total number of virtual episodes is at most $O(HSA\ln^2(T))$, so we still use $T$ for the total number of episodes and absorb their effect into lower-order terms, while the data-dependent complexity measures are defined over real episodes only.

\paragraph{New $Q$-function estimator.}
A key technical ingredient in our analysis is the construction of our $Q$-function estimator $\hat{Q}_t$ defined in \cref{eq:Q-estimate_policy}, which is used for OFTRL in~\cref{eq:opt_pi_OFTRL} (\cref{line:Q_estimator}).
Since OFTRL updates the policy using the $Q$-function as a loss,  we propagate the loss prediction $m_t$ in~\cref{def:predictor_sequence} or~\cref{def:predictor_sequence2} through the $Q$-recursion and obtain the predicted $Q$-function $Q^{\pi_t}(s,a;m_t)$.

The main difficulty is that simply incorporating this predicted term does not provide enough control over the bias.
Indeed, if we were to use $Q^{\pi_t}(s,a;m_t)$ alone (\textit{i.e.}, the first two terms in \cref{eq:Q-estimate_policy}) as the estimator, then the expected deviation $\E_t\brk{\hat Q_t(s,a)}-Q^{\pi_t}(s,a;\ell_t)$ could be positive or negative, making the bias difficult to control directly.
To resolve this issue, we subtract a margin of the form $\gamma_t H/q_t(s)$ to ensure that $\hatQ_t(s,a)$ is an optimistic estimator of $Q^{\pi_t}(s,a;\ell_t)$.
Indeed, a direct calculation shows that (see \cref{lem:expect_hatQ_proof} for details)
\begin{align}
    &\E_t\brk*{\hat{Q}_t(s,a)} = Q^{\pi_t}(s, a; m_t) \\
    &\qquad\qquad + \frac{q^{\pi_t}(s)}{q_t(s)}Q^{\pi_t}(s, a; \ell_t-m_t)Y_t  - \frac{\gamma_t H}{q_t(s)}. 
    \label{eq:expect_hatQ}
\end{align}
In particular, in a real episode ($Y_t=1$), we have
$0 \leq \E_t\brk[\big]{Q^{\pi_t}(s,a;\ell_t) - \hat{Q}_t(s,a)}
    \leq {2\gamma_t H}/{q_t(s)}$.
Hence, when $\gamma_t=0$, $\hat{Q}_t(s,a)$ is an unbiased estimator of $Q^{\pi_t}(s,a;\ell_t)$.
For $\gamma_t > 0$, the estimator is optimistic in expectation, and this controlled optimism is useful in the regret analysis, as it makes the bias term easy to handle while still benefiting from the variance reduction due to the predictor $Q^{\pi_t}(s,a;m_t)$.

By contrast, in virtual episodes ($Y_t=0$), the loss is zero but the prediction term $Q^{\pi_t}(s,a;m_t)$ still remains in the estimator.
Thus, the estimator is not optimistic in the same sense as in real episodes and may introduce additional bias.
The effect of this bias is limited, however, because the total number of virtual episodes is small, and it contributes only a lower-order term to the regret bound.

\subsection{Regret Upper Bounds}
We now state the resulting regret guarantee, with all proofs deferred to \cref{app:policy_opt_proofs}.
\begin{thm} \label{thm:Policy_OPT}
\cref{alg:OFTRL_PO} with $m_t$ in \cref{def:predictor_sequence} guarantees
\begin{align}
    \Reg_T &\lesssim \sqrt{H^2SA\ln^2(T)\min\set*{L^\star, HT - L^\star, Q_{\infty}, V_1}} \\
    &\qquad + H^{\frac52}S^{\frac32}A\ln^2(T).
\end{align}
Under the stochastic regime with adversarial corruption, it simultaneously ensures
\begin{align}
    \Reg_T &\lesssim \sqrt{H^2SA\ln^2(T)\prn*{\V T + \calC}} + H^{\frac52}S^{\frac32}A\ln^2(T),\\
    \Reg_T &\lesssim U + \sqrt{U\calC} + H^{\frac52}S^{\frac32}A\ln^2(T),
\end{align}
where $U = \sum_{s}\sum_{a\neq\pi^\star(s)}\frac{H^2\ln^2(T)}{\Delta(s,a)}$.
\end{thm}
In the worst case, our bound becomes the known regret bounds based on policy optimization in \citet{luo2021policy,dann2023best}, and the lower-order term $H^3S^2A^2\ln^2(T)$ in \citet[Theorem 4.3]{dann2023best} is improved to $H^{\frac52}S^{\frac32}A\ln^2(T)$.

\begin{thm}\label{thm:Policy_OPT2}
\cref{alg:OFTRL_PO} with $m_t$ in \cref{def:predictor_sequence2} guarantees
\begin{align}
    \Reg_T &\lesssim \sqrt{H^2SA\ln^2(T)\min\set*{L^\star, HT - L^\star, Q_{\infty}}}\\
    &\qquad + H^{\frac52}S^{\frac32}A\ln^2(T).
\end{align}
Under the stochastic regime with adversarial corruption, it simultaneously ensures
\begin{align}
    \Reg_T &\lesssim \sqrt{H^2SA\ln^2(T)\prn*{\V T + \calC}} + H^{\frac52}S^{\frac32}A\ln^2(T), \\
    \Reg_T &\lesssim \Uvar +  \sqrt{\Uvar\calC} + \sqrt{H S^2 A^2 \calC}\ln^{\frac32}(T) \\
    &\qquad + H^{\frac12}S^{\frac32}A(H^{2} + \sqrt{A})\ln^2(T),
\end{align}
where $\Uvar = \sum_{s}\sum_{a\neq\pi^\star(s)}\frac{H\Varmax(s)\ln^2(T)}{\Delta(s,a)}$.
\end{thm}
\begin{rem}\label{rem:Policy_OPT2}
    If the uncorrupted losses are generated independently and are uncorrelated across layers, the variance-aware gap-dependent bound in \cref{thm:Policy_OPT2} improves by a factor of $H$ to $\Uvar = \sum_{s}\sum_{a\neq\pi^\star(s)} \frac{\Varmax(s)\ln^2(T)}{\Delta(s,a)}$.
\end{rem}
\section{Regret Lower Bounds}\label{sec:lower_bounds}
We complement our regret upper bounds with information-theoretic regret lower bounds for MDPs with bandit feedback.
In multi-armed bandits, refined lower bounds such as first-order, second-order, and path-length bounds were developed by \citet{gerchinovitz2016refined} and \citet{bubeck2019improved}.
For MDPs, the data-independent minimax lower bound $\Omega(\sqrt{HSAT})$ is already known \citep{zimin2013online,tsuchiya2025reinforcement}.
Accordingly, our focus is on data-dependent lower bounds in MDPs, identifying the optimal dependence on measures such as $L^\star$, $Q_\infty$, $V_1$, and $\V$. All proofs are deferred to \cref{app:lower_bounds}.

The refined adversarial lower bounds below are obtained via a simple truncation reduction:
we run an instance that induces $\Omega(\sqrt{HSAT})$ regret for only a prefix of episodes and set all losses to zero thereafter.
This ensures that the corresponding complexity measure is small, while preserving a nontrivial regret contribution from the active phase.
\begin{thm}
\label{thm:adversarial_data_lb}
Suppose that $H \geq 3$, $A \geq 3$, $T \geq \frac{SA}{8H}$ and $\alpha \in \brk*{\frac{\ceil*{SA/8H}}{T}, 1}$.
Then, for any policy $\{\pi_t\}_{t=1}^T$, there exists an episodic MDP with adversarial losses satisfying
$L^\star \leq \alpha HT$, $Q_\infty\leq \alpha HT$, $V_1 \leq \alpha SAT$ such that $\Reg_T = \Omega\prn{\sqrt{\alpha HSAT}}$.
\end{thm}
Consequently, choosing $\alpha$ appropriately for each complexity measure in \cref{thm:adversarial_data_lb} yields $\Reg_T = \Omega\prn{\sqrt{SA L^\star}}$, $\Omega\prn{\sqrt{SA Q_\infty}}$, and $\Omega\prn{\sqrt{H V_1}}$.

These lower bounds imply that the regret bounds in \cref{sec:global_optimization} are optimal up to logarithmic factors, except for the path-length bound.
For the path-length bounds, our upper bound leaves an $\sqrt{{SA}/{H}}$-dependent gap.
This gap is consistent with that in multi-armed bandits: the best-known upper bounds come with an explicit $\sqrt{A}$ dependence on the number of actions $A$, whereas the lower bounds scale as $\Omega(\sqrt{V_1})$ and do not require any dependence on $A$ \citep{bubeck2019improved}.
By contrast, policy optimization often introduces an additional dependence on $H$.
In particular, as in \citet{luo2021policy,dann2023best}, the resulting data-independent guarantees can be worse by a factor of $H$ compared to the best-known bounds. 
Closing this $H$-gap in minimax regret remains an important open problem.

Finally, we turn to the stochastic regime and consider the occupancy-weighted variance $\V$.
\begin{thm}
\label{thm:variance_lower_bound}
Suppose that $H \geq 3$, $A \geq 3$, $T \geq \frac{SA}{8H}$, and $\alpha \in (0, 1/4]$.
Then, for any policy $\{\pi_t\}_{t=1}^T$, there exists an episodic MDP with stochastic losses satisfying
$\V \leq \alpha H$ such that
$\Reg_T = \Omega\prn{\sqrt{\alpha HSAT}} = \Omega\prn{\sqrt{SA \V T}}$.
\end{thm}
The above lower bound implies that the regret bound of $\tilO(\sqrt{SA\V T})$ in \cref{sec:global_optimization} is optimal up to logarithmic factors.
In contrast, policy optimization typically incurs a multiplicative factor of $H$ here as well.

\section{Conclusion}
\label{sec:conclusion}
In this work, we introduced refined complexity measures for tabular MDPs, including the second-order measure $Q_\infty$, the path-length measure $V_1$, and the variance measures $\V$ and $\Varmax$.
For the known-transition setting, we developed both global optimization and policy optimization algorithms that achieve best-of-both-worlds guarantees with these refined data-dependent bounds.
Our lower bounds further show that the guarantees obtained by the global optimization approach are nearly optimal.
An important direction for future work is to remove the known-transition assumption and extend our results to unknown transitions, where transition-estimation errors would also need to be controlled in a data-dependent way.


\newpage
\section*{Acknowledgements}
TT was supported by JSPS KAKENHI Grant Number JP24K23852 and partially supported by JSPS KAKENHI Grant Number JP26K21297.
KY was partially supported by JSPS KAKENHI Grant Number JP24H00703.

\section*{Impact Statement}
This paper presents work whose goal is to advance the field of machine learning. There are many potential societal consequences of our work, none of which we feel must be specifically highlighted here.

\bibliography{reference}

@inproceedings{dann2023best,
  title={Best of both worlds policy optimization},
  author={Dann, Christoph and Wei, Chen-Yu and Zimmert, Julian},
  booktitle={International Conference on Machine Learning},
  pages={6968--7008},
  year={2023},
  organization={PMLR}
}

@inproceedings{ito2022adversarially,
  title={Adversarially robust multi-armed bandit algorithm with variance-dependent regret bounds},
  author={Ito, Shinji and Tsuchiya, Taira and Honda, Junya},
  booktitle={Conference on Learning Theory},
  pages={1421--1422},
  year={2022},
  organization={PMLR}
}

@inproceedings{lancewicki2025near,
  title = 	 {Near-optimal Regret Using Policy Optimization in Online {MDP}s with Aggregate Bandit Feedback},
  author =       {Lancewicki, Tal and Mansour, Yishay},
  booktitle = 	 {Proceedings of the 42nd International Conference on Machine Learning},
  pages = 	 {32467--32491},
  year = 	 {2025},
  volume = 	 {267},
  publisher =    {PMLR}
}

@inproceedings{zanette2019tighter,
  title={Tighter problem-dependent regret bounds in reinforcement learning without domain knowledge using value function bounds},
  author={Zanette, Andrea and Brunskill, Emma},
  booktitle={International Conference on Machine Learning},
  pages={7304--7312},
  year={2019},
  organization={PMLR}
}

@InProceedings{azar2017minimax,
  title = 	 {Minimax Regret Bounds for Reinforcement Learning},
  author =       {Mohammad Gheshlaghi Azar and Ian Osband and R{\'e}mi Munos},
  booktitle = 	 {Proceedings of the 34th International Conference on Machine Learning},
  pages = 	 {263--272},
  year = 	 {2017},
  volume = 	 {70},
  publisher =    {PMLR},
}

@InProceedings{zhang2021reinforcement,
  title = 	 {Is Reinforcement Learning More Difficult Than Bandits? A Near-optimal Algorithm Escaping the Curse of Horizon},
  author =       {Zhang, Zihan and Ji, Xiangyang and Du, Simon},
  booktitle = 	 {Proceedings of Thirty Fourth Conference on Learning Theory},
  pages = 	 {4528--4531},
  year = 	 {2021},
  volume = 	 {134},
  publisher =    {PMLR},
}

@inproceedings{zhou2023sharp,
  title={Sharp variance-dependent bounds in reinforcement learning: Best of both worlds in stochastic and deterministic environments},
  author={Zhou, Runlong and Zihan, Zhang and Du, Simon Shaolei},
  booktitle={International Conference on Machine Learning},
  pages={42878--42914},
  year={2023},
  organization={PMLR}
}

@inproceedings{chen2025sharp,
title={Sharp Gap-Dependent Variance-Aware Regret Bounds for Tabular {MDP}s},
author={Shulun Chen and Runlong Zhou and Zihan Zhang and Maryam Fazel and Simon Shaolei Du},
booktitle = {Advances in Neural Information Processing Systems},
publisher = {Curran Associates, Inc.},
volume = {38},
year={2025},
}

@inproceedings{simchowitz2019non,
 author = {Simchowitz, Max and Jamieson, Kevin G},
 booktitle = {Advances in Neural Information Processing Systems},
 pages = {1153--1162},
 publisher = {Curran Associates, Inc.},
 title = {Non-Asymptotic Gap-Dependent Regret Bounds for Tabular {MDPs}},
 volume = {32},
 year = {2019}
}

@inproceedings{zimin2013online,
 author = {Zimin, Alexander and Neu, Gergely},
 booktitle = {Advances in Neural Information Processing Systems},
 pages = {1583--1591},
 publisher = {Curran Associates, Inc.},
 title = {Online learning in episodic {Markovian} decision processes by relative entropy policy search},
 volume = {26},
 year = {2013}
}

@inproceedings{jin2018q,
 author = {Jin, Chi and Allen-Zhu, Zeyuan and Bubeck, Sebastien and Jordan, Michael I},
 booktitle = {Advances in Neural Information Processing Systems},
 pages = {4863--4873},
 publisher = {Curran Associates, Inc.},
 title = {Is {Q}-Learning Provably Efficient?},
 volume = {31},
 year = {2018}
}

@InProceedings{jin2020learning,
  title = 	 {Learning Adversarial {M}arkov Decision Processes with Bandit Feedback and Unknown Transition},
  author =       {Jin, Chi and Jin, Tiancheng and Luo, Haipeng and Sra, Suvrit and Yu, Tiancheng},
  booktitle = 	 {Proceedings of the 37th International Conference on Machine Learning},
  pages = 	 {4860--4869},
  year = 	 {2020},
  volume = 	 {119},
  publisher =    {PMLR},
}

@inproceedings{luo2021policy,
 author = {Luo, Haipeng and Wei, Chen-Yu and Lee, Chung-Wei},
 booktitle = {Advances in Neural Information Processing Systems},
 pages = {22931--22942},
 publisher = {Curran Associates, Inc.},
 title = {Policy Optimization in Adversarial {MDPs}: Improved Exploration via Dilated Bonuses},
 volume = {34},
 year = {2021}
}

@book{tsybakov2009non,
    author = {Tsybakov, Alexandre B},
    title = {Introduction to Nonparametric Estimation},
    publisher = {Springer},
    year = {2009}
}

@inproceedings{gerchinovitz2016refined,
 author = {Gerchinovitz, S\'{e}bastien and Lattimore, Tor},
 booktitle = {Advances in Neural Information Processing Systems},
 pages = {1198--1206},
 publisher = {Curran Associates, Inc.},
 title = {Refined Lower Bounds for Adversarial Bandits},
 volume = {29},
 year = {2016}
}

@inproceedings{audibert2007tuning,
    author="Audibert, Jean-Yves
            and Munos, R{\'e}mi
            and Szepesv{\'a}ri, Csaba",
    title="Tuning Bandit Algorithms in Stochastic Environments",
    booktitle="Algorithmic Learning Theory",
    year="2007",
    publisher="Springer Berlin Heidelberg",
    pages="150--165",
}

@InProceedings{rosenberg2019online,
  title = 	 {Online Convex Optimization in Adversarial {M}arkov Decision Processes},
  author =       {Rosenberg, Aviv and Mansour, Yishay},
  booktitle = 	 {Proceedings of the 36th International Conference on Machine Learning},
  pages = 	 {5478--5486},
  year = 	 {2019},
  volume = 	 {97},
  publisher =    {PMLR},
}

@InProceedings{shani2020optimistic,
  title = 	 {Optimistic Policy Optimization with Bandit Feedback},
  author =       {Shani, Lior and Efroni, Yonathan and Rosenberg, Aviv and Mannor, Shie},
  booktitle = 	 {Proceedings of the 37th International Conference on Machine Learning},
  pages = 	 {8604--8613},
  year = 	 {2020},
  volume = 	 {119},
  publisher =    {PMLR},
}

@article{even2009online,
  title={Online {Markov} decision processes},
  author={Even-Dar, Eyal and Kakade, Sham M and Mansour, Yishay},
  journal={Mathematics of Operations Research},
  volume={34},
  number={3},
  pages={726--736},
  year={2009},
  publisher={INFORMS}
}

@inproceedings{lancewicki2022learning,
  title={Learning adversarial {Markov} decision processes with delayed feedback},
  author={Lancewicki, Tal and Rosenberg, Aviv and Mansour, Yishay},
  booktitle={Proceedings of the AAAI Conference on Artificial Intelligence},
  volume={36},
  pages={7281--7289},
  year={2022}
}

@inproceedings{jin2022near,
  title={Near-optimal regret for adversarial {MDP} with delayed bandit feedback},
  author={Jin, Tiancheng and Lancewicki, Tal and Luo, Haipeng and Mansour, Yishay and Rosenberg, Aviv},
  publisher = {Curran Associates, Inc.},
  booktitle={Advances in Neural Information Processing Systems},
  volume={35},
  pages={33469--33481},
  year={2022}
}

@InProceedings{lancewicki2023delay,
  title = 	 {Delay-Adapted Policy Optimization and Improved Regret for Adversarial {MDP} with Delayed Bandit Feedback},
  author =       {Lancewicki, Tal and Rosenberg, Aviv and Sotnikov, Dmitry},
  booktitle = 	 {Proceedings of the 40th International Conference on Machine Learning},
  pages = 	 {18482--18534},
  year = 	 {2023},
  volume = 	 {202},
  publisher =    {PMLR}
}

@article{jaksch2010near,
  author  = {Thomas Jaksch and Ronald Ortner and Peter Auer},
  title   = {Near-optimal Regret Bounds for Reinforcement Learning},
  journal = {Journal of Machine Learning Research},
  year    = {2010},
  volume  = {11},
  number  = {51},
  pages   = {1563--1600},
}

@InProceedings{bubeck2012best,
  title = 	 {The Best of Both Worlds: Stochastic and Adversarial Bandits},
  author = 	 {Bubeck, Sébastien and Slivkins, Aleksandrs},
  booktitle = 	 {Proceedings of the 25th Annual Conference on Learning Theory},
  pages = 	 {42.1--42.23},
  year = 	 {2012},
  volume = 	 {23},
  publisher =    {PMLR}
}

@inproceedings{jin2020simultaneously,
 author = {Jin, Tiancheng and Luo, Haipeng},
 booktitle = {Advances in Neural Information Processing Systems},
 pages = {16557--16566},
 publisher = {Curran Associates, Inc.},
 title = {Simultaneously Learning Stochastic and Adversarial Episodic {MDPs} with Known Transition},
 volume = {33},
 year = {2020}
}

@inproceedings{jin2021best,
 author = {Jin, Tiancheng and Huang, Longbo and Luo, Haipeng},
 booktitle = {Advances in Neural Information Processing Systems},
 pages = {20491--20502},
 publisher = {Curran Associates, Inc.},
 title = {The best of both worlds: stochastic and adversarial episodic {MDPs} with unknown transition},
 volume = {34},
 year = {2021}
}

@InProceedings{wei2018more,
  title = 	 {More Adaptive Algorithms for Adversarial Bandits},
  author =       {Wei, Chen-Yu and Luo, Haipeng},
  booktitle = 	 {Proceedings of the 31st Conference On Learning Theory},
  pages = 	 {1263--1291},
  year = 	 {2018},
  volume = 	 {75},
  publisher =    {PMLR}
}

@article{zimmert2021tsallis,
  title={Tsallis-{INF}: An Optimal Algorithm for Stochastic and Adversarial Bandits},
  author={Zimmert, Julian and Seldin, Yevgeny},
  journal={Journal of Machine Learning Research},
  volume={22},
  number={28},
  pages={1--49},
  year={2021}
}

@InProceedings{masoudian2021improved,
  title = 	 {Improved Analysis of the {Tsallis-INF} Algorithm in Stochastically Constrained Adversarial Bandits and Stochastic Bandits with Adversarial Corruptions},
  author =       {Masoudian, Saeed and Seldin, Yevgeny},
  booktitle = 	 {Proceedings of Thirty Fourth Conference on Learning Theory},
  pages = 	 {3330--3350},
  year = 	 {2021},
  volume = 	 {134},
  publisher =    {PMLR}
}

@inproceedings{ito2021parameter,
  title = 	 {Parameter-Free Multi-Armed Bandit Algorithms with Hybrid Data-Dependent Regret Bounds},
  author =       {Ito, Shinji},
  booktitle = 	 {Proceedings of Thirty Fourth Conference on Learning Theory},
  pages = 	 {2552--2583},
  year = 	 {2021},
  volume = 	 {134},
  organization={PMLR}
}

@inproceedings{erez2021towards,
 author = {Erez, Liad and Koren, Tomer},
 booktitle = {Advances in Neural Information Processing Systems},
 pages = {28511--28521},
 title = {Towards Best-of-All-Worlds Online Learning with Feedback Graphs},
 volume = {34},
 year = {2021},
 publisher = {Curran Associates, Inc.}
}

@inproceedings{seldin2014one,
  title = 	 {One Practical Algorithm for Both Stochastic and Adversarial Bandits},
  author = 	 {Seldin, Yevgeny and Slivkins, Aleksandrs},
  booktitle = 	 {Proceedings of the 31st International Conference on Machine Learning},
  pages = 	 {1287--1295},
  year = 	 {2014},
  volume = 	 {32},
  organization={PMLR}
}

@inproceedings{auer2016algorithm,
  title = 	 {An algorithm with nearly optimal pseudo-regret for both stochastic and adversarial bandits},
  author = 	 {Auer, Peter and Chiang, Chao-Kai},
  booktitle = 	 {29th Annual Conference on Learning Theory},
  pages = 	 {116--120},
  year = 	 {2016},
  volume = 	 {49},
  organization={PMLR}
}

@inproceedings{seldin2017improved,
  title = 	 {An Improved Parametrization and Analysis of the {EXP3++} Algorithm for Stochastic and Adversarial Bandits},
  author = 	 {Seldin, Yevgeny and Lugosi, Gábor},
  booktitle = 	 {Proceedings of the 2017 Conference on Learning Theory},
  pages = 	 {1743--1759},
  year = 	 {2017},
  volume = 	 {65},
  organization={PMLR}
}

@inproceedings{ito2025adapting,
title={Adapting to Stochastic and Adversarial Losses in Episodic {MDPs} with Aggregate Bandit Feedback},
author={Ito, Shinji and Jamieson, Kevin and Luo, Haipeng and Maiti, Arnab and Tsuchiya, Taira},
booktitle = {Advances in Neural Information Processing Systems},
publisher = {Curran Associates, Inc.},
volume = {38},
year={2025},
}

@inproceedings{dann2023blackbox,
  title = 	 {A Blackbox Approach to Best of Both Worlds in Bandits and Beyond},
  author =       {Dann, Chris and Wei, Chen-Yu and Zimmert, Julian},
  booktitle = 	 {Proceedings of Thirty Sixth Conference on Learning Theory},
  pages = 	 {5503--5570},
  year = 	 {2023},
  volume = 	 {195},
  organization={PMLR}
}

@InProceedings{ito2023best,
  title = 	 {Best-of-Three-Worlds Linear Bandit Algorithm with Variance-Adaptive Regret Bounds},
  author =       {Ito, Shinji and Takemura, Kei},
  booktitle = 	 {Proceedings of Thirty Sixth Conference on Learning Theory},
  pages = 	 {2653--2677},
  year = 	 {2023},
  volume = 	 {195},
  publisher =    {PMLR},
}

@InProceedings{tsuchiya2023further,
  title = 	 {Further Adaptive Best-of-Both-Worlds Algorithm for Combinatorial Semi-Bandits},
  author =       {Tsuchiya, Taira and Ito, Shinji and Honda, Junya},
  booktitle = 	 {Proceedings of The 26th International Conference on Artificial Intelligence and Statistics},
  pages = 	 {8117--8144},
  year = 	 {2023},
  volume = 	 {206},
  publisher =    {PMLR},
}

@InProceedings{lee2021achieving,
  title = 	 {Achieving Near Instance-Optimality and Minimax-Optimality in Stochastic and Adversarial Linear Bandits Simultaneously},
  author =       {Lee, Chung-Wei and Luo, Haipeng and Wei, Chen-Yu and Zhang, Mengxiao and Zhang, Xiaojin},
  booktitle = 	 {Proceedings of the 38th International Conference on Machine Learning},
  pages = 	 {6142--6151},
  year = 	 {2021},
  volume = 	 {139},
  publisher =    {PMLR},
}

@inproceedings{ito2022nearly,
  title={Nearly optimal best-of-both-worlds algorithms for online learning with feedback graphs},
  author={Ito, Shinji and Tsuchiya, Taira and Honda, Junya},
  booktitle={Advances in Neural Information Processing Systems},
  volume={35},
  pages={28631--28643},
  year={2022},
  publisher = {Curran Associates, Inc.}
}

@inproceedings{jin2023no,
  title={No-regret online reinforcement learning with adversarial losses and transitions},
  author={Jin, Tiancheng and Liu, Junyan and Rouyer, Chlo{\'e} and Chang, William and Wei, Chen-Yu and Luo, Haipeng},
  booktitle={Advances in Neural Information Processing Systems},
  volume={36},
  pages={38520--38585},
  year={2023},
  publisher = {Curran Associates, Inc.}
}

@inproceedings{allenberg2006hannan,
  title={Hannan consistency in on-line learning in case of unbounded losses under partial monitoring},
  author={Allenberg, Chamy and Auer, Peter and Gy{\"o}rfi, L{\'a}szl{\'o} and Ottucs{\'a}k, Gy{\"o}rgy},
  booktitle={International Conference on Algorithmic Learning Theory},
  pages={229--243},
  year={2006},
  organization={Springer}
}

@inproceedings{lee2020bias,
 author = {Lee, Chung-Wei and Luo, Haipeng and Wei, Chen-Yu and Zhang, Mengxiao},
 booktitle = {Advances in Neural Information Processing Systems},
 pages = {15522--15533},
 publisher = {Curran Associates, Inc.},
 title = {Bias no more: high-probability data-dependent regret bounds for adversarial bandits and MDPs},
 volume = {33},
 year = {2020}
}

@article{hazan2011better,
  author  = {Elad Hazan and Satyen Kale},
  title   = {Better Algorithms for Benign Bandits},
  journal = {Journal of Machine Learning Research},
  year    = {2011},
  volume  = {12},
  number  = {35},
  pages   = {1287--1311},
}

@InProceedings{bubeck2019improved,
  title = 	 {Improved Path-length Regret Bounds for Bandits},
  author =       {Bubeck, S{\'e}bastien and Li, Yuanzhi and Luo, Haipeng and Wei, Chen-Yu},
  booktitle = 	 {Proceedings of the Thirty-Second Conference on Learning Theory},
  pages = 	 {508--528},
  year = 	 {2019},
  volume = 	 {99},
  publisher =    {PMLR},
}

@inproceedings{neu2015first,
  title = 	 {First-order regret bounds for combinatorial semi-bandits},
  author = 	 {Neu, Gergely},
  booktitle = 	 {Proceedings of The 28th Conference on Learning Theory},
  pages = 	 {1360--1375},
  year = 	 {2015},
  volume = 	 {40},
  organization={PMLR}
}

@inproceedings{zheng2025gap,
    title={Gap-Dependent Bounds for {Q-Learning} using Reference-Advantage Decomposition},
    author={Zhong Zheng and Haochen Zhang and Lingzhou Xue},
    booktitle={The Thirteenth International Conference on Learning Representations},
    year={2025}
}

@inproceedings{rakhlin2013online,
    title = 	 {Online Learning with Predictable Sequences},
  author = 	 {Rakhlin, Alexander and Sridharan, Karthik},
  booktitle = 	 {Proceedings of the 26th Annual Conference on Learning Theory},
  pages = 	 {993--1019},
  year = 	 {2013},
  volume = 	 {30},
  organization={PMLR}
}

@InProceedings{steinhardt2014adaptivity,
  title = 	 {Adaptivity and Optimism: An Improved Exponentiated Gradient Algorithm},
  author = 	 {Steinhardt, Jacob and Liang, Percy},
  booktitle = 	 {Proceedings of the 31st International Conference on Machine Learning},
  pages = 	 {1593--1601},
  year = 	 {2014},
  volume = 	 {32},
  publisher =    {PMLR},
}

@article{auer2002nonstochastic,
  author    = {Auer, Peter and Cesa-Bianchi, Nicoló and Freund, Yoav and Schapire, Robert E},
  title     = {The Nonstochastic Multiarmed Bandit Problem},
  journal   = {SIAM Journal on Computing},
  year      = {2002},
  volume    = {32},
  number    = {1},
  pages     = {48--77},
  publisher = {SIAM},
}

@InProceedings{tsuchiya2024corrupted,
  title = 	 {Corrupted Learning Dynamics in Games},
  author =       {Tsuchiya, Taira and Ito, Shinji and Luo, Haipeng},
  booktitle = 	 {Proceedings of Thirty Eighth Conference on Learning Theory},
  pages = 	 {5506--5552},
  year = 	 {2025},
  volume = 	 {291},
  publisher =    {PMLR},
}

@inproceedings{kakade2002approximately,
  title={Approximately optimal approximate reinforcement learning},
  author={Kakade, Sham and Langford, John},
  booktitle = {Proceedings of the Nineteenth International Conference on Machine Learning},
  pages={267--274},
  year={2002}
}

@inproceedings{neu2010online,
  title={The Online Loop-free Stochastic Shortest-Path Problem.},
  author={Neu, Gergely and Gy{\"o}rgy, Andr{\'a}s and Szepesv{\'a}ri, Csaba},
  booktitle = 	 {Proceedings of the 23rd Conference on Learning Theory},
  pages={231--243},
  year={2010}
}

@article{cesa1996worst,
  title={Worst-case quadratic loss bounds for prediction using linear functions and gradient descent},
  author={Cesa-Bianchi, Nicolo and Long, Philip M and Warmuth, Manfred K},
  journal={IEEE Transactions on Neural Networks},
  volume={7},
  number={3},
  pages={604--619},
  year={1996},
  publisher={IEEE}
}

@inproceedings{zinkevich2003online,
    author = {Zinkevich, Martin},
    title = {Online Convex Programming and Generalized Infinitesimal Gradient Ascent},
    year = {2003},
    booktitle = {the Twentieth International Conference on Machine Learning},
    pages = {928-935},
}

@article{littlestone1994weighted,
  title={The weighted majority algorithm},
  author={Littlestone, Nick and Warmuth, Manfred K},
  journal = {Information and Computation},
  volume={108},
  number={2},
  pages={212--261},
  year={1994},
  publisher={Elsevier}
}

@article{schulman2017proximal,
  title={Proximal Policy Optimization Algorithms},
  author={Schulman, John and Wolski, Filip and Dhariwal, Prafulla and Radford, Alec and Klimov, Oleg},
  journal={arXiv preprint arXiv:1707.06347},
  year={2017}
}

@article{mnih2015human,
  author  = {Mnih, Volodymyr and Kavukcuoglu, Koray and Silver, David and Rusu, Andrei A. and Veness, Joel and Bellemare, Marc G. and Graves, Alex and Riedmiller, Martin and Fidjeland, Andreas K. and Ostrovski, Georg and Petersen, Stig and Beattie, Charles and Sadik, Amir and Antonoglou, Ioannis and King, Helen and Kumaran, Dharshan and Wierstra, Daan and Legg, Shane and Hassabis, Demis},
  title   = {Human-level control through deep reinforcement learning},
  journal = {Nature},
  year    = {2015},
  volume  = {518},
  number  = {7540},
  pages   = {529--533},
}

@article{komorowski2018artificial,
  title={The Artificial Intelligence Clinician learns optimal treatment strategies for sepsis in intensive care},
  author={Komorowski, Matthieu and Celi, Leo A. and Badawi, Omar and Gordon, Anthony C. and Faisal, A. Aldo},
  journal={Nature Medicine},
  volume={24},
  number={11},
  pages={1716--1720},
  year={2018},
}

@inproceedings{maurer2009empirical,
  title = {Empirical {Bernstein} Bounds and Sample-Variance Penalization},
  author = {Maurer, Andreas and Pontil, Massimiliano},
  booktitle = {Conference on Learning Theory},
  year = {2009}
}

@InProceedings{chiang2012online,
  title = 	 {Online Optimization with Gradual Variations},
  author = 	 {Chiang, Chao-Kai and Yang, Tianbao and Lee, Chia-Jung and Mahdavi, Mehrdad and Lu, Chi-Jen and Jin, Rong and Zhu, Shenghuo},
  booktitle = 	 {Proceedings of the 25th Annual Conference on Learning Theory},
  pages = 	 {6.1--6.20},
  year = 	 {2012},
  volume = 	 {23},
  publisher =    {PMLR}
}

@article{tsuchiya2025reinforcement,
  title={Reinforcement Learning from Adversarial Preferences in Tabular {MDPs}},
  author={Tsuchiya, Taira and Ito, Shinji and Luo, Haipeng},
  journal={arXiv preprint arXiv:2507.11706},
  year={2025}
}

@InProceedings{zhang2024settling,
  title = 	 {Settling the sample complexity of online reinforcement learning},
  author =       {Zhang, Zihan and Chen, Yuxin and Lee, Jason D and Du, Simon S},
  booktitle = 	 {Proceedings of Thirty Seventh Conference on Learning Theory},
  pages = 	 {5213--5219},
  year = 	 {2024},
  volume = 	 {247},
  publisher =    {PMLR},
}

@InProceedings{yu2009Markov,
author="Yu, Jia Yuan
and Mannor, Shie
and Shimkin, Nahum",
title="{M}arkov Decision Processes with Arbitrary Reward Processes",
booktitle="Recent Advances in Reinforcement Learning",
year="2008",
publisher="Springer Berlin Heidelberg",
pages="268--281",
}

@InProceedings{zhao2023variance,
  title = 	 {Variance-Dependent Regret Bounds for Linear Bandits and Reinforcement Learning: Adaptivity and Computational Efficiency},
  author =       {Zhao, Heyang and He, Jiafan and Zhou, Dongruo and Zhang, Tong and Gu, Quanquan},
  booktitle = 	 {Proceedings of Thirty Sixth Conference on Learning Theory},
  pages = 	 {4977--5020},
  year = 	 {2023},
  volume = 	 {195},
  publisher =    {PMLR},
}

@InProceedings{wagenmaker2022first,
  title = 	 {First-Order Regret in Reinforcement Learning with Linear Function Approximation: A Robust Estimation Approach},
  author =       {Wagenmaker, Andrew J and Chen, Yifang and Simchowitz, Max and Du, Simon and Jamieson, Kevin},
  booktitle = 	 {Proceedings of the 39th International Conference on Machine Learning},
  pages = 	 {22384--22429},
  year = 	 {2022},
  volume = 	 {162},
  publisher =    {PMLR},
}
\bibliographystyle{icml2026}


\crefalias{section}{appendix}
\crefalias{subsection}{appendix}
\crefalias{subsubsection}{appendix}

\newpage
\appendix
\onecolumn

\tableofcontents
\newpage
\section{Summary of Notation}

\begin{table}[H]
\centering
\caption{Summary of notation.}
\label{tab:notation-summary-main}
\begin{tabular*}{\linewidth}{ll}
\toprule
\textbf{Symbol} & \textbf{Meaning} \\
\midrule
\textbf{Online tabular MDPs} & \\
$\mathcal{M}=(\mathcal{S},\mathcal{A},P,H,s_0)$ & Episodic finite-horizon MDP with known transition \\
$\mathcal{S}, S$ & State space and its size $S=|\mathcal{S}|$ \\
$\mathcal{A}, A$ & Action space and its size $A=|\mathcal{A}|$ \\
$P$ & Transition kernel \\
$H$ & Horizon length \\
$T$ & Number of episodes \\
$h(s)$ & Layer index of state $s$ \\
$s_{t,h},\,a_{t,h}$ & State / action at step $h$ in episode $t$ \\
$\ell_t(s,a)$ & Loss assigned to $(s,a)$ in episode $t$ \\
$\pi_t$ & Policy in episode $t$ \\
$\Reg_T$ & Regret over $T$ episodes \\
$\I_t(s,a)$ & Visitation indicator of $(s,a)$ in episode $t$ \\
$N_t(s,a)$ & Number of visits to $(s,a)$ up to  episode $t$ \\
$V^{\pi}(s;\ell)$ & Value function under policy $\pi$ from state $s$ with loss $\ell$ \\
$Q^{\pi}(s,a;\ell)$ & $Q$-function under policy $\pi$ from $(s,a)$ with loss $\ell$ \\
$\pio, \qst$ & Optimal policy and its occupancy measure \\
$q^{\pi}(s),\ q^{\pi}(s,a)$ & Occupancy measure under policy $\pi$ \\
$q^{\pi}(s',a'\mid s,a)$ & Conditional occupancy from $(s,a)$ to $(s',a')$ under $\pi$ \\
$\ell'_t(s,a)$ & Uncorrupted i.i.d.~loss  \\
$\calC = \E\brk{\sum_{t=1}^T\sum_{h=0}^{H-1}\|\ell'_t(h)-\ell_t(h)\|_\infty}$ & Corruption budget \\
$\mu(s,a) = \E_{\ell'\sim\calD}[\ell'(s,a)]$ & Mean of $\ell'_t(s,a)$ \\
$\sigma^2(s,a) = \E_{\ell'\sim\calD}\brk{(\ell'(s,a)-\mu(s,a))^2}$ & Variance of $\ell'_t(s,a)$ \\
$\pist$ & Optimal deterministic policy under $\mu$\\
$\Delta(s,a) = Q^{\pist}(s,a;\mu)-\min_{a'\in\calA}Q^{\pist}(s,a';\mu)$ & Suboptimality gap at $(s,a)$ \\
\midrule
\textbf{Data-dependent complexity measures} & \\
$L^\star = \E\brk{\sum_{t=1}^T V^{\pio}(s_0;\ell_t)}$ & First-order complexity in \cref{def:first-order_bound} \\
$Q_\infty = \min_{\ell^\star \in[0,1]^{S\times A}}\E\brk{\sum_{t=1}^T\sum_{h=0}^{H-1} \nrm{\ell_{t}(h)-\ell^\star(h) }_\infty^2}$ & Second-order complexity in \cref{def:second-order_bound} \\
$V_1 = \E\brk{\sum_{t = 1}^{T-1}\nrm*{\ell_{t + 1} - \ell_t}_1}$ & Path-length complexity \cref{def:path-length_bound} \\
$\V = \max_{\pi}\E\brk{\sum_{s,a} q^{\pi}(s,a)\sigma^2(s,a)}$ & Occupancy-weighted variance in \cref{def:total_variance} \\
$\Varmax(s) = \max_{a, \pi}\E\brk{\sum_{s',a'} q^{\pi}(s',a'\mid s,a)\sigma^2(s',a')}$ & Conditional occupancy-weighted variance  at state $s$ in \cref{def:variance-to-go} \\
\midrule
\textbf{Common notation for \cref{alg:OFTRL_GO,alg:OFTRL_PO}} & \\
$\eta_t(s,a)$ & Learning rate for $(s,a)$ in episode $t$  \\
$m_t(s,a)$ & Loss prediction for $(s,a)$ \\
$\zeta_t(s,a)$ & Data-dependent term for updating $\eta_t(s,a)$  \\
\midrule
\textbf{Notation only for \cref{alg:OFTRL_GO}} (global optimization) & \\
$\hat \ell_t(s,a)$ & Loss estimator \\
\midrule
\textbf{Notation only for \cref{alg:OFTRL_PO}} (policy optimization) & \\
$b_t(s)$ & Bonus term at state $s$ in episode $t$  \\
$B_t(s,a)$ & Dilated bonus-to-go at $(s,a)$ in episode $t$  \\
$Y_t\in\{0,1\}$ & Episode indicator ($Y_t=1$ real, $Y_t=0$ virtual)  \\
$\mathcal{T}_r,\ \mathcal{T}_v$ & Sets of real and virtual episodes \\
$\hat Q_t(s,a)$ & $Q$-function estimator \\
$\gamma_t, q_t(s)=q^{\pi_t}(s)+\gamma_t$ & Exploration rate and smoothed state occupancy  \\
$L_{t,h(s)}, M_{t,h(s)}$ &  Realized / predicted suffix loss from layer $h(s)$  \\
\bottomrule
\end{tabular*}
\end{table}

For the reader's convenience, \cref{tab:notation-summary-main} collects the main notation used throughout the paper.

We formalize the conditional occupancy measure $q_{\pi_t}(s',a'\mid s,a)$ as follows:
\begin{align}
q^{\pi_t}(s',a' \mid s,a)
=\begin{cases}
0 & \text{if } h(s') < h(s),\\
0 & \text{if } h(s') = h(s) \text{ and } (s',a') \neq (s,a),\\
1 & \text{if } (s',a') = (s,a),\\
\Pr\brk*{(s_{t,h(s')},a_{t,h(s')})=(s',a') \mid (s_{t,h(s)},a_{t,h(s)})=(s,a)} 
& \text{if } h(s') > h(s).
\end{cases}
\end{align}
\section{Further Discussion of Related Work}
\subsection{Additional Related Work}
\label{app:additional_related}

\paragraph{Online MDPs.}
Adversarial MDPs were first studied by \citet{even2009online,yu2009Markov} and later extended to the episodic setting by \citet{zimin2013online}. 
Episodic MDPs with \emph{bandit feedback}, where the learner observes losses only for the visited state-action pairs, have been extensively studied.
In this setting, a line of work has established minimax-optimal regret via \emph{global optimization}, which solves an optimization problem over the set of all occupancy measures. 
In particular, when the transition dynamics are known, global optimization achieves the minimax regret $\tilde{\mathcal{O}}(\sqrt{HSAT})$ \citep{zimin2013online}, while under unknown transitions, global optimization achieves the regret $\tilde{\mathcal{O}}(\sqrt{H^2S^2AT})$ \citep{rosenberg2019online,jin2020learning}.
While global optimization enjoys optimal regret guarantees, it requires solving a large-scale convex optimization problem over the feasible occupancy-measure polytope at each episode, which can be computationally demanding and limits scalability in practice.

This has motivated a complementary line of \emph{policy optimization}, which typically reduces the problem to separate instances of the multi-armed bandit problem for each state. Policy optimization was first shown to achieve a $\tilO(T^{2/3})$ regret upper bound under bandit feedback by \citet{shani2020optimistic}. Later \citet{luo2021policy} attained the optimal $\tilO(\sqrt{T})$ dependence on the number of episodes by combining a dilated exploration bonus with a refined $Q$-function estimator, achieving $\tilde{\mathcal{O}}(\sqrt{H^3SAT})$ regret under known transitions and $\tilde{\mathcal{O}}(\sqrt{H^4S^2AT})$ under unknown transitions. Compared with global optimization, these guarantees incur an additional factor of $H$ in the leading term, and closing this dependence gap remains open. Subsequent work has extended policy optimization to more challenging feedback models, including delayed and aggregate feedback \citep{lancewicki2022learning,jin2022near,lancewicki2023delay,lancewicki2025near}.

In parallel, in the stochastic setting, both \emph{model-based} algorithms, which learn the transition dynamics and construct confidence sets over the transition and loss functions \citep{jaksch2010near,azar2017minimax}, and \emph{value-based} methods, which add exploration bonuses directly to the $Q$-function \citep{jin2018q,zanette2019tighter}, have been developed and achieve near-optimal regret guarantees.

\paragraph{Best-of-both-worlds algorithms.}
The \emph{best-of-both-worlds} guarantee, which aims to achieve near-optimal regret in both adversarial and stochastic regimes with a single algorithm, was first investigated in the multi-armed bandit setting by \citet{bubeck2012best}.
Subsequent research has refined the analysis through a variety of techniques \citep{seldin2014one,auer2016algorithm,seldin2017improved,wei2018more,zimmert2021tsallis,masoudian2021improved,ito2021parameter}.
A prominent line of work builds on follow-the-regularized-leader (FTRL), showing that suitable regularization yields algorithms that are automatically adaptive between adversarial and stochastic regimes \citep{wei2018more,zimmert2021tsallis,ito2021parameter}.
In these algorithms, stochastic-regime bounds are obtained via the self-bounding technique \citep{zimmert2021tsallis,masoudian2021improved}, which also extends naturally to adversarially corrupted stochastic settings.
FTRL-based approaches have also been developed in other settings, including linear bandits \citep{lee2021achieving,dann2023blackbox,ito2023best},
contextual bandits \citep{dann2023blackbox},
combinatorial semi-bandits~\citep{tsuchiya2023further}, and online learning with feedback graphs~\citep{erez2021towards,ito2022nearly}.

Beyond multi-armed bandits, best-of-both-worlds algorithms have been extended to MDPs. For global optimization, \citet{jin2020simultaneously,jin2021best} developed best-of-both-worlds algorithms.
In particular, \citet{jin2021best} introduced the loss-shifting technique, which served as a key component in obtaining stochastic regret bounds. 
This idea has been further applied to more challenging settings, including adversarial transitions and aggregate feedback \citep{jin2023no,ito2025adapting}.
For policy optimization, \citet{dann2023best} established best-of-both-worlds guarantees under bandit feedback, covering Tsallis entropy, Shannon entropy, and log-barrier regularizers.

\paragraph{Data-dependent bounds in the adversarial regime.}
The worst-case analysis, which is driven by the worst-case instance within a problem class, can be overly pessimistic for practical environments. 
Accordingly, in the broader online learning literature--including learning with expert advice \citep{littlestone1994weighted}, multi-armed bandits \citep{auer2002nonstochastic}, and online convex optimization \citep{zinkevich2003online}--it has been shown that regret can often be upper bounded by refined data-dependent complexity measures \citep{cesa1996worst,allenberg2006hannan,neu2015first}.
There are several data-dependent complexity measures for adversarial regimes. 
The \emph{first-order} complexity $L^\star$ scales with the cumulative loss of the best action, yielding $\tilO(\sqrt{L^\star})$-type regret \citep{allenberg2006hannan,neu2015first,zimmert2021tsallis,ito2021parameter}. 
The \emph{second-order} complexity $Q_\infty$ quantifies the magnitude of loss fluctuations, leading to $\tilO(\sqrt{Q_\infty})$-type regret \citep{hazan2011better,wei2018more,ito2022adversarially}, path-length complexity $V_1=\sum_{t=1}^{T-1}\|\ell_{t+1}-\ell_t\|_1$ depends on the cumulative variation of the loss sequence, giving $\tilO(\sqrt{V_1})$-type regret \citep{wei2018more,ito2022adversarially}. 

Extending data-dependent complexity measures from bandits to Markov decision processes remains an active and challenging direction. 
For MDPs, first-order regret complexity was introduced by \citet{lee2020bias}, who showed that under unknown transitions one can achieve a first-order bound $\tilde{\mathcal{O}}(\sqrt{HS^2A L^\star})$ in global optimization. 
For policy optimization, \citet{dann2023best} established best-of-both-worlds guarantees with a first-order bound under known transitions. 

In contrast, second-order and path-length measures have been much less studied in MDPs. In this work, we introduce these notions (see \cref{def:first-order_bound,def:second-order_bound,def:path-length_bound}) and prove regret bounds for them when the transition is known. Our results cover both global optimization and policy optimization in the best-of-both-worlds setting.
The case with unknown transitions is still largely open. The main obstacle is the error caused by estimating the transition kernel. It is not clear how to control this error with data-dependent quantities. As a result, even first-order guarantees in the best-of-both-worlds regime are not known when transitions are unknown. For policy optimization with unknown transitions, it is not even known whether any data-dependent guarantee is possible without best-of-both-worlds adaptivity \cite{dann2023best}.

\paragraph{Variance-dependent bounds in the stochastic regime.}
In stochastic regimes, variance-aware and gap-dependent bounds originate from UCB-V \citep{audibert2007tuning}, which augments UCB with empirical variance-based bonuses and yields tighter regret when variances are small.
In the context of best-of-both-worlds algorithms, FTRL-based methods can incorporate variance information into the optimization, achieving variance-dependent regret guarantees while retaining robustness in adversarial settings \citep{ito2022adversarially,tsuchiya2023further,ito2023best}.

For episodic stochastic MDPs, \citet{zanette2019tighter} were among the first to obtain variance-dependent guarantees for model-based optimistic methods, introducing a maximum per-step conditional variance. Using this type of variance, \citet{simchowitz2019non} derive variance-aware gap-dependent guarantees. 
From a value-based perspective, the Monotonic Value Propagation (MVP) algorithm provides a baseline via optimistic value iteration with Bernstein-type bonuses \citep{zhang2021reinforcement}, and subsequent works establish variance-dependent bounds in terms of the maximum total variance \citep{zhou2023sharp,zhang2024settling,zheng2025gap}. 
Continuing this line, \citet{chen2025sharp} derive gap-dependent guarantees by the maximum conditional total variance, explicitly conditioning variance on the state.
Related variance-based guarantees have also been explored in linear contextual bandits and linear MDPs \citep{wagenmaker2022first,zhao2023variance}.
To our knowledge, variance-dependent guarantees have not been established within FTRL-based analyses for MDPs.

\subsection{Comparison with Existing Variance Measures}
\label{app:comparison_variance}
We restate our variance measures in \cref{def:total_variance,def:variance-to-go}.
\begin{equation}
\V = 
\max_{\pi\in\Pi}\sum_{s,a} q^{\pi} (s,a)\sigma^2(s,a) \in \brk*{0,\frac{H}{4}},
\end{equation}
\begin{equation}
\Varmax(s) = \max_{\pi\in\Pi,a\in\calA}\sum_{s',a'} q^{\pi}(s',a'\mid s,a)\sigma^2(s',a') \in \brk*{0,\frac{H}{4}}.
\end{equation}

Most model-based and value-based algorithms focus on the stochastic setting with unknown transitions and without corruption. 
In particular, \citet{zanette2019tighter,simchowitz2019non} introduce
\begin{align}
\Var^\star(s,a) = \sigma^2(s,a) + \Var_{s'\sim P(\cdot\mid s,a)}[V^{\pi^*}(s')]  \in \brk*{0,\frac{H^2}{4}},
\end{align}
and $\Q^\star = \max_{s,a}\Var^\star(s,a)$, where $\pi^*$ differs slightly from our definition of $\pist$, in that it is defined as a policy that simultaneously achieves the minimum of $Q(s,a)$ and $V(s)$ for all state-action pairs $(s,a)$ with uncorrupted loss.

Subsequent works further aggregate $\Var^\star$ along trajectories and consider total-variance measures.
\citet{zhou2023sharp,zhang2024settling} define an unconditional total variance, while \citet{chen2025sharp} introduce a conditional total variance,
\begin{align}
\Var_{\max} &= \max_{\pi\in\Pi} \sum_{s,a} q^{\pi}(s,a)\Var^\star(s,a) \in \brk*{0,\frac{H^3}{4}},\\
\Var_{\max}^c &= \max_{\pi\in\Pi, s\in\calS} \sum_{s',a'} \bar q^\pi(s',a'\mid s)\Var^\star(s',a') \in \brk*{0,\frac{H^3}{4}},
\end{align}
where $\bar q^\pi(s',a' \mid s)$ denotes the occupancy of $(s',a')$ over the entire trajectory conditioned on visiting state $s$.
In our setting, $\V$ and $\Varmax$ are analogous to  $\Var_{\max}$ and $\Var_{\max}^c$, respectively, and can be interpreted as natural variance measures for MDPs with known transitions.
Since we consider known transitions, the second term of $\Var^\star(s,a)$ is unnecessary and can be omitted. 
Moreover, $\Varmax$ is defined using conditional occupancy measures $q^{\pi}(s',a'\mid s,a)$ and captures variance only after visiting $(s,a)$, whereas $\Var_{\max}^c$ aggregates variance over the entire trajectory by conditioning on the event of visiting state $s$.
As a consequence, our variance measures are $H^2$-sharper than those based on $\Var^\star(s,a)$.
Moreover, unlike $\Var_{\max}^c$, our $\Varmax(s)$ is state-dependent, resulting in a more refined and potentially smaller regret bound.

The correspondence between $\V$ and $\Var_{\max}$, $\Varmax$ and $\Var_{\max}^c$ can also be seen from regret bounds.
Our variance-dependent leading term scales as $O(\sqrt{\V T})$, mirroring the $O(\sqrt{\Var_{\max}T})$ dependence in \citet{zhou2023sharp}.
In gap-dependent regret bounds, variance typically appears as a multiplicative coefficient of the suboptimality gap:
in our bound this coefficient is $H\Varmax(s)$ (whose worst-case value matches that of~\citealt{dann2023best}),
while \citet{simchowitz2019non} use $H\Q^\star$ and \citet{chen2025sharp} use $\min\{H^2,\Var_{\max}^c\}$.

\section{Regret Analysis of Optimistic Follow-the-Regularized-Leader}
In this section, we provide a regret analysis of optimistic follow-the-regularized-leader (OFTRL) for the MDP setting.
General OFTRL bounds of this type appear in \citet[Lemma 1]{ito2022adversarially} and \citet[Lemma 16]{tsuchiya2024corrupted}. For completeness, we restate the argument here. 
We then specialize the bound to our two instances, \cref{lem:OFTRL_log_barrier_occup} for global optimization and \cref{lem:OFTRL_log_barrier_policy} for policy optimization. These lemmas will be used in \cref{app:global_opt_proofs,app:policy_opt_proofs}, respectively.
Given a strictly convex function $\psi$, we use $D_{\psi}(y,x)\coloneqq \psi(y)-\psi(x)-\inpr*{\nabla\psi(x), y-x}$ to denote the Bregman divergence induced by $\psi$.

\begin{lem} 
\label{lem:OFTRL_general}
Let $\set{\ell_t}_{t=1}^T$ be a sequence of loss vectors.
Suppose that $p_t$ is defined by the OFTRL algorithm over a convex set $\calK$ and differentiable regularizers $\{\psi_t\}_{t=1}^{T+1}$ and loss predictions $\{m_t\}_{t=1}^{T+1}$:
\begin{align}
    p_t = \argmin_{p\in\calK}\set*{ \inpr*{ p, \sum_{\tau=1}^{t-1}\ell_\tau + m_t} + \psi_t(p) } 
    .
    \label{def:oftrl}
\end{align}
Then, for any $u\in\calK$ it holds  that
\begin{align}
    \sum_{t=1}^{T}\inpr*{ p_t-u, \ell_t }  
    &\leq \psi_{1}(u) - \psi_1(p_1) + \sum_{t=1}^T(\psi_{t+1}(u) - \psi_t(u) + \psi_t(p_{t+1}) - \psi_{t+1}(p_{t+1})) \\
    &\qquad + \sum_{t=1}^{T}\prn*{\inpr*{ p_t-p_{t+1}, \ell_t - m_t } -  D_{\psi_t}(p_{t+1}, p_t)} + \inpr*{ u - p_{T+1}, m_{T+1}}
    .
\end{align}
\end{lem}

\begin{proof}
Define $\tilde{\psi}_t(p) = \psi_t(p) + \inpr*{ p,m_t }$, and let 
\begin{align}
    F_t(p) = \inpr*{ p, \sum_{\tau=1}^{t-1}\ell_\tau + m_t}+\psi_t(p) = \inpr*{ p, \sum_{\tau=1}^{t-1}\ell_\tau}+\tilde{\psi}_t(p).
\end{align}
Since $-\sum_{t=1}^T\inpr*{ u, \ell_t} = \tilpsi_{T+1}(u) - F_{T+1}(u)$, we have
\begin{align}
    \sum_{t=1}^T\inpr*{ p_t-u, \ell_t} 
    &= \sum_{t=1}^T\inpr*{ p_t, \ell_t} + \tilpsi_{T+1}(u) - F_{T+1}(u)\\
    &= \sum_{t=1}^T\inpr*{ p_t, \ell_t} + \tilpsi_{T+1}(u) - F_{T+1}(u) - F_1(p_1) + F_1(p_1) - F_{T+1}(p_{T+1}) + F_{T+1}(p_{T+1})\\
    &= \sum_{t=1}^T\inpr*{ p_t, \ell_t} + \tilpsi_{T+1}(u) - F_{T+1}(u) - F_1(p_1) + \sum_{t=1}^T(F_t(p_t)-F_{t+1}(p_{t+1})) + F_{T+1}(p_{T+1})\\
    &\leq \tilpsi_{T+1}(u) - \tilpsi_1(p_1) + \sum_{t=1}^T(F_t(p_t)-F_{t+1}(p_{t+1}) + \inpr*{ p_t, \ell_t})
    ,
    \label{eq:OFTRL1}
\end{align}
where the last inequality follows from $F_{T+1}(p_{T+1}) - F_{T+1}(u) \leq 0$ because $p_{T+1}$ is the minimizer of $F_{T+1}$ and $F_1(p_1) = \tilpsi_1(p_1)$ by definition.

Hence,
\begin{align}
    &F_t(p_t)-F_{t+1}(p_{t+1}) + \inpr*{ p_t, \ell_t} \\
    &= F_t(p_t)-F_{t}(p_{t+1}) + F_{t}(p_{t+1})-F_{t+1}(p_{t+1}) + \inpr*{ p_t, \ell_t}\\
    &= F_t(p_t)-F_{t}(p_{t+1}) - \inpr*{ p_{t+1}, \ell_t} + \tilpsi_t(p_{t+1}) - \tilpsi_{t+1}(p_{t+1}) + \inpr*{ p_t, \ell_t} \\
    &= F_t(p_t)-F_{t}(p_{t+1}) + \inpr*{ p_t - p_{t+1}, \ell_t} + \psi_t(p_{t+1}) - \psi_{t+1}(p_{t+1}) + \inpr*{ p_{t+1}, m_t - m_{t+1} } \\
    & \leq -D_{F_t}(p_{t+1}, p_t) + \inpr*{ p_t - p_{t+1}, \ell_t} + \psi_t(p_{t+1}) - \psi_{t+1}(p_{t+1}) + \inpr*{ p_{t+1}, m_t - m_{t+1} } , \label{eq:OFTRL_sum}
\end{align}
where the last inequality follows from first-order optimality. 
Since $F_t$ is convex and differentiable and $p_t\in\arg\min_{p\in\calK}F_t(p)$, we have
$\langle\nabla F_t(p_t),\,y-p_t\rangle\ge 0$ for all $y\in\calK$.
Using the definition of the Bregman divergence
$D_{F_t}(p_{t+1},p_t)
= F_t(p_{t+1})-F_t(p_t)-\langle p_{t+1}-p_t,\,\nabla F_t(p_t)\rangle$, 
we obtain
\begin{align}
F_t(p_t)-F_t(p_{t+1})
= -D_{F_t}(p_{t+1},p_t)-\langle\nabla F_t(p_t),\,p_{t+1}-p_t\rangle
\le -D_{F_t}(p_{t+1},p_t).
\end{align}

Then, from \cref{eq:OFTRL1,eq:OFTRL_sum} and using that adding a linear term does not change the Bregman divergence ($D_{F_t}= D_{\psi_t}$), we obtain
\begin{align}
    \sum_{t=1}^T\inpr*{ p_t-u, \ell_t} &\leq \tilpsi_{T+1}(u) - \tilpsi_1(p_1) + \sum_{t=1}^T(\psi_t(p_{t+1}) - \psi_{t+1}(p_{t+1}) -D_{\psi_t}(p_{t+1}, p_t)) \\
    &\qquad + \sum_{t=1}^T \inpr*{ p_t - p_{t+1}, \ell_t} + \sum_{t=1}^T \inpr*{ p_{t+1}, m_t - m_{t+1} } \\
    &= \psi_{T+1}(u) - \psi_1(p_1) + \sum_{t=1}^T(\psi_t(p_{t+1}) - \psi_{t+1}(p_{t+1}) -D_{\psi_t}(p_{t+1}, p_t))\\
    &\qquad + \sum_{t=1}^T \inpr*{ p_t - p_{t+1}, \ell_t} + \sum_{t=1}^T \inpr*{ p_{t+1} - p_t, m_t } + \inpr*{ u - p_{T+1}, m_{T+1}} \\
    &= \psi_{1}(u) - \psi_1(p_1) + \sum_{t=1}^T(\psi_{t+1}(u) - \psi_t(u) + \psi_t(p_{t+1}) - \psi_{t+1}(p_{t+1}) -D_{\psi_t}(p_{t+1}, p_t))\\
    &\qquad + \sum_{t=1}^T \inpr*{ p_t - p_{t+1}, \ell_t - m_t } + \inpr*{ u - p_{T+1}, m_{T+1}}
    ,
\end{align}
which completes the proof.
\end{proof}

\begin{lem}
\label{lem:OFTRL_stab_bregman}
Let $D_{\phi}$ denote the Bregman divergence associated with $\phi(x)=-\ln x$ and define $g(x)=x-\ln(1+x)$.
Then, for any $x > 0$ and $a \geq -1/x$, it holds that
\begin{align} 
\max_{y \in \Rp} \set*{a(x-y) - D_{\phi}(y,x)} &= g(ax).
\end{align}
\end{lem}

\begin{proof}
This can be proven by simply considering the worst-case w.r.t.~$y$ and the proof can be found \textit{e.g.,}~in \citet[Lemma 5]{ito2022adversarially}.
\end{proof}

The following lemma will be used in the regret analysis for global optimization. There are a few prior works that analyze global optimization with time-varying log-barrier learning rates (see \citealt[Appendix C.7.2]{jin2023no} for a related approach).
\begin{lem}[OFTRL with global optimization]
\label{lem:OFTRL_log_barrier_occup}  
Suppose that a sequence of occupancy measures 
$p_1, \dots, p_T \in \Omega(P)$
is given by OFTRL in \cref{def:oftrl} with regularizer $\psi_t$ given by 
$\psi_t(p) = \sum_{s,a} \frac{1}{\eta_t(s, a)}\ln\prn[\big]{\frac{1}{p(s,a)}}$ as in \cref{eq:epoch_FTRL},
for some nonincreasing learning rate $\eta_t(s, a)$ with $\eta_1(s, a)=\eta_1$ for all $s,a$,
and let losses $\set{\ell_t}_{t=1}^T$ and loss predictions $\set{m_t}_{t=1}^{T+1}$ satisfy
\begin{align}
    \eta_t(s,a) p_t(s,a)(\ell_t(s,a) - m_t(s,a)) &\geq -\frac{1}{2}
    \label{eq:OFTRL_constraint}
\end{align}
 for all $t,s,a$.
Then, for any $u\in\Omega(P)$, it holds that
\begin{align}
    \sum_{t=1}^T \inpr*{ p_t-u, \ell_t} 
    &\leq  \frac{SA\ln(SAT)}{\eta_1} + \sum_{t=1}^T \sum_{s,a} \prn*{\frac{1}{\eta_{t+1}(s,a)} - \frac{1}{\eta_t(s,a)}}\ln(SAT)\\
    &\qquad +  \sum_{t=1}^T  \sum_{s,a}  \eta_t(s,a)p_t(s,a)^2(\ell_t(s,a)-m_t(s,a))^2\\
    &\qquad + \frac{1}{T}\sum_{t=1}^T \inpr*{ -u + \frac{1}{SA}\sum_{s,a} p^{\max}_{s,a}, \ell_t } + 2H\nrm*{m_{T+1}}_{\infty}, 
\end{align}
where $p^{\max}_{s,a}$ denotes the occupancy measure induced by a policy that maximizes the probability of visiting state-action pair $(s,a)$ in transition $P$.
\end{lem}
\begin{proof}
    Let
    \begin{align}
    u' = \prn*{1-\frac{1}{T}}u + \frac{1}{TSA}\sum_{s,a} p^{\max}_{s,a}.
    \end{align}
    Then we have $u'\in \Omega(P)$, since $u\in\Omega(P)$, each $p^{\max}_{s,a}\in\Omega(P)$, and $\Omega(P)$ is convex. We also define $\phi(x) = \ln\prn*{\frac{1}{x}}$.
    Thus, by the definition of the regularizer, we have
    \begin{align}
        \psi_t(u') &= \sum_{s,a} \frac{1}{\eta_t(s,a)}\ln\prn*{\frac{1}{u'(s,a)}} \leq \sum_{s,a} \frac{1}{\eta_t(s,a)}\ln\prn*{\frac{SAT}{p^{\max}_{s,a}(s,a)}}.
    \end{align}
    First, for any $u \in \Omega(P)$, we decompose the regret as 
    \begin{align}
        \sum_{t=1}^T \inpr*{ p_t-u, \ell_t}  
        &= \sum_{t=1}^T \inpr*{ p_t-u', \ell_t} + \sum_{t=1}^T \inpr*{ u'-u, \ell_t} \\
        &= \sum_{t=1}^T \inpr*{ p_t-u', \ell_t} + \frac{1}{T}\sum_{t=1}^T \inpr*{ -u + \frac{1}{SA}\sum_{s,a} p^{\max}_{s,a}, \ell_t }.\label{eq:OFTRL_decom_occup}
    \end{align}
    Using \cref{lem:OFTRL_general}, the first term in the last equality is upper-bounded by
    \begin{align}
        \sum_{t=1}^{T}\inpr*{ p_t-u', \ell_t }  
        &\leq \underbrace{\psi_{1}(u') - \psi_1(p_1) + \sum_{t=1}^T(\psi_{t+1}(u') - \psi_t(u') + \psi_t(p_{t+1}) - \psi_{t+1}(p_{t+1}))}_{\text{penalty-term}} \\
        &\qquad + \underbrace{\sum_{t=1}^{T}\prn*{\inpr*{ p_t-p_{t+1}, \ell_t - m_t } -  D_{\psi_t}(p_{t+1}, p_t)} + \inpr*{ u' - p_{T+1}, m_{T+1}}}_{\text{stability-term}}
        .
    \end{align}
    The penalty-term can be upper bounded by
    \begin{align}
        \text{penalty-term} &\leq \sum_{s,a}\frac{1}{\eta_1}\ln\prn*{\frac{p_1(s,a)}{u'(s,a)}} + \sum_{t=1}^T \sum_{s,a} \prn*{\frac{1}{\eta_{t+1}(s,a)} - \frac{1}{\eta_t(s,a)}}\prn*{\phi(u'(s,a)) - \phi(p_{t+1}(s,a))}\\
        &\leq \sum_{s,a}\frac{1}{\eta_1}\ln\prn*{\frac{p_1(s,a)}{u'(s,a)}} + \sum_{t=1}^T \sum_{s,a}\prn*{\frac{1}{\eta_{t+1}(s,a)} - \frac{1}{\eta_t(s,a)}}\ln\prn*{\frac{p_{t+1}(s,a)}{u'(s,a)}}\\
        &\leq \sum_{s,a}\frac{1}{\eta_1}\ln\prn*{\frac{SATp_1(s,a)}{p^{\max}_{s,a}(s,a)}} + \sum_{t=1}^T \sum_{s,a}\prn*{\frac{1}{\eta_{t+1}(s,a)} - \frac{1}{\eta_t(s,a)}}\ln\prn*{\frac{SATp_{t+1}(s,a)}{p^{\max}_{s,a}(s,a)}}\\
        &\leq \frac{SA\ln\prn*{SAT}}{\eta_1} + \sum_{t=1}^T \sum_{s,a}\prn*{\frac{1}{\eta_{t+1}(s,a)} - \frac{1}{\eta_t(s,a)}}\ln\prn*{SAT},
        \label{eq:OFTRL_log_penalty}
    \end{align}
    where the last inequality follows from $p_\tau(s,a) \leq p^{\max}_{s,a}(s,a)$ for all $\tau \leq T + 1$.
    
    Next, we bound the stability term. The Bregman divergence $D_{\psi_t}(p_{t+1},p_t)$ can be written as
    \begin{align}
        D_{\psi_t}(p_{t+1}(s,a), p_t(s,a)) &= \sum_{s,a} \frac{1}{\eta_t(s,a)}D_{\phi}(p_{t+1}(s,a), p_t(s,a))
        ,
    \end{align}
    where we recall that $\phi(x)=-\ln x$ and $g(x)=x-\ln(1+x)$. 
    By using \cref{lem:OFTRL_stab_bregman}, we have
    \begin{align}
        &\inpr*{ p_t- p_{t+1}, \ell_t - m_t} - D_{\psi_t}(p_{t+1}, p_t) \\
        &\leq \sum_{s,a} \prn*{(\ell_t(s,a) - m_t(s,a))(p_t(s,a)-p_{t+1}(s,a)) - \frac{1}{\eta_t(s,a)}D_{\phi}(p_{t+1}(s,a), p_t(s,a))}\\
        &\leq \sum_{s,a}\frac{1}{\eta_t(s,a)} g(\eta_t(s,a)(\ell_t(s,a)-m_t(s,a))p_t(s,a)) \tag{by \cref{lem:OFTRL_stab_bregman} and \cref{eq:OFTRL_constraint}}\\
        &\leq \sum_{s,a}\eta_t(s,a)p_t(s,a)^2(\ell_t(s,a)-m_t(s,a))^2,
    \end{align}
    where the last inequality follows from
    $g(x) = x - \ln(1+x) \leq x^2$ for $x\geq -\frac{1}{2}$ and \cref{eq:OFTRL_constraint}.
    
    Therefore,
    \begin{align}
        \text{stability-term} 
        &=  \sum_{t=1}^{T}\prn*{\inpr*{ p_t-p_{t+1}, \ell_t - m_t} -  D_{\psi_t}(p_{t+1}, p_t)} + \inpr*{ u - p_{T+1}, m_{T+1}} \\
        &\leq \sum_{t=1}^T  \sum_{s,a}  \eta_t(s,a)p_t(s,a)^2(\ell_t(s,a)-m_t(s,a))^2 + 2H\nrm*{m_{T+1}}_{\infty}, \label{eq:OFTRL_log_stability}
    \end{align}
    where the last bound follows from H\"older's inequality 
    $\inpr*{u-p_{T+1},\,m_{T+1}}\leq \nrm*{u-p_{T+1}}_{1}\,\nrm*{m_{T+1}}_{\infty}\le 2H\nrm*{m_{T+1}}_{\infty}$,
    since $u,p_{T+1}\in\Omega(P)$ imply $\nrm*{u-p_{T+1}}_{1}\le 2H$.
    Combining \cref{eq:OFTRL_decom_occup,eq:OFTRL_log_penalty,eq:OFTRL_log_stability} completes the proof.
\end{proof}

The following lemma will be used in the regret analysis for policy optimization.
\begin{lem}[OFTRL with policy optimization]
\label{lem:OFTRL_log_barrier_policy}  
Suppose that a sequence of probability vectors
$p_1, \dots, p_T \in \triangle(\calA)$
is given by OFTRL in \cref{def:oftrl} with regularizer $\psi_t(p)$ in \cref{eq:regularizer_policy} for some nonincreasing learning rate $\eta_t(a)$ with $\eta_1(a)=\eta_1$ for all $a$, and let losses $\set{\ell_t}_{t=1}^T$, loss predictions $\set{m_t}_{t=1}^{T+1}$ and $\set{x_t}_{t=1}^T$ be such that
\begin{align}
    \eta_t(a) p_t(a)(\ell_t(a) - m_t(a) + x_t) &\geq -\frac{1}{2} \label{eq:OFTRL_constraint_policy}
\end{align}
for all $t,a$.
Then for any $u\in\triangle(\mathcal{A})$, the OFTRL algorithm achieves
\begin{align}
    \sum_{t=1}^T \inpr*{ p_t-u, \ell_t} 
    &\leq  \frac{A\ln(AT^2)}{\eta_1} + \sum_{t=1}^T \sum_a \prn*{\frac{1}{\eta_{t+1}(a)} - \frac{1}{\eta_t(a)}}\ln(AT^2)\\
    &\qquad +  \sum_{t=1}^T  \sum_a  \eta_t(a)p_t(a)^2(\ell_t(a)-m_t(a) + x_t)^2\\
    &\qquad + \frac{1}{T^2}\sum_{t=1}^T \inpr*{ -u + \frac{1}{A}\one, \ell_t } + 2\nrm*{m_{T+1}}_{\infty}. 
\end{align}
\end{lem}
\begin{proof}
    Let $u'=\prn*{1-\frac{1}{T^2}}u + \frac{1}{AT^2}\one \in  \triangle(\calA)$ and $\phi(x) = \ln\prn*{\frac{1}{x}}$.
    Then, we have
    \begin{align}
        \psi_t(u') &= \sum_{a} \frac{1}{\eta_t(a)}\ln\prn*{\frac{1}{u'(a)}} \leq \sum_{a} \frac{1}{\eta_t(a)}\ln\prn*{AT^2} \leq \sum_{a} \frac{1}{\eta_t(a)}\ln\prn*{AT^2}.
    \end{align}
    First, for any $u \in  \triangle(\calA)$, we decompose the regret as 
    \begin{align}
        \sum_{t=1}^T \inpr*{ p_t-u, \ell_t}  
        &= \sum_{t=1}^T \inpr*{ p_t-u', \ell_t} + \sum_{t=1}^T \inpr*{ u'-u, \ell_t} \\
        &= \sum_{t=1}^T \inpr*{ p_t-u', \ell_t} + \frac{1}{T^2}\sum_{t=1}^T \inpr*{ -u + \frac{1}{A}\one, \ell_t }.\label{eq:OFTRL_decomp_policy}
    \end{align}
    For the first term, using \cref{lem:OFTRL_general}, we obtain
    \begin{align}
        \sum_{t=1}^{T}\inpr*{ p_t-u', \ell_t }  
        &\leq \underbrace{\psi_{1}(u') - \psi_1(p_1) + \sum_{t=1}^T(\psi_{t+1}(u') - \psi_t(u') + \psi_t(p_{t+1}) - \psi_{t+1}(p_{t+1}))}_{\text{penalty-term}} \\
        &\qquad + \underbrace{\sum_{t=1}^{T}\prn*{\inpr*{ p_t-p_{t+1}, \ell_t - m_t } -  D_{\psi_t}(p_{t+1}, p_t)} + \inpr*{ u' - p_{T+1}, m_{T+1}}}_{\text{stability-term}}
        .
    \end{align}
    
    Then, 
    \begin{align}
        \text{penalty-term} &\leq \frac{A\ln(AT^2)}{\eta_1} + \sum_{t=1}^T \sum_a \prn*{\frac{1}{\eta_{t+1}(a)} - \frac{1}{\eta_t(a)}}\prn*{\phi(u'(a)) - \phi(p_{t+1}(a))}\\
        &\leq \frac{A\ln(AT^2)}{\eta_1} + \sum_{t=1}^T \sum_a \prn*{\frac{1}{\eta_{t+1}(a)} - \frac{1}{\eta_t(a)}}\ln(AT^2),
        \label{eq:OFTRL_log_penalty_policy}
    \end{align}
    where the last inequality follows from $ \phi(p_{t+1}(a)) \geq 0$.
    
    Next, we bound the stability term. The Bregman divergence $D_{\psi_t}(p_{t+1},p_t)$ can be written as
    \begin{align}
        D_{\psi_t}(p_{t+1}(a), p_t(a)) &= \sum_{a} \frac{1}{\eta_t(a)}D_{\phi}(p_{t+1}(a), p_t(a))
        ,
    \end{align}
    where $D_{\phi}$ denotes the Bregman divergence associated with $\phi(x)=-\ln (x)$. 
    Recall that $g(x)=x-\ln(1+x)$ by \cref{lem:OFTRL_stab_bregman}, we have
    \begin{align}
        &\inpr*{ p_t- p_{t+1}, \ell_t - m_t} - D_{\psi_t}(p_{t+1}, p_t) \\
        &= \inpr*{ p_t- p_{t+1}, \ell_t - m_t + x_t\one} - D_{\psi_t}(p_{t+1}, p_t) \\
        &\leq \sum_{a} \prn*{(\ell_t(a) - m_t(a) + x_t)(p_t(a)-p_{t+1}(a) ) - \frac{1}{\eta_t(a)}D_{\phi}(p_{t+1}(a), p_t(a))}\\
        &\leq \sum_{a}\frac{1}{\eta_t(a)} g(\eta_t(a)(\ell_t(a)-m_t(a) + x_t)p_t(a)) \tag{by \cref{lem:OFTRL_stab_bregman} and \cref{eq:OFTRL_constraint_policy}}\\
        &\leq \sum_{a}\eta_t(a)p_t(a)^2(\ell_t(a)-m_t(a)+x_t)^2,
    \end{align}
    where the last inequality follows from
    $g(x) = x - \ln(1+x) \leq x^2$ for $x\geq -\frac{1}{2}$ and \cref{eq:OFTRL_constraint_policy}.
    
    Therefore,
    \begin{align}
        \text{stability-term} 
        &=  \sum_{t=1}^{T}\prn*{\inpr*{ p_t-p_{t+1}, \ell_t - m_t} -  D_{\psi_t}(p_{t+1}, p_t)} + \inpr*{ u - p_{T+1}, m_{T+1}} \\
        &\leq \sum_{t=1}^T  \sum_a  \eta_t(a)p_t(a)^2(\ell_t(a)-m_t(a) + x_t)^2 + 2\nrm*{m_{T+1}}_{\infty}, \label{eq:OFTRL_log_stability_policy}
    \end{align}
    where the last bound follows from H\"older's inequality 
    $\inpr*{u-p_{T+1},\,m_{T+1}}\leq \nrm*{u-p_{T+1}}_{1}\,\nrm*{m_{T+1}}_{\infty}\le 2\nrm*{m_{T+1}}_{\infty}$,
    since $u,p_{T+1}\in \triangle(\calA)$ implies $\nrm*{u-p_{T+1}}_{1}\le 2$.
    
    Combining \cref{eq:OFTRL_decomp_policy,eq:OFTRL_log_penalty_policy,eq:OFTRL_log_stability_policy} completes the proof.
\end{proof}

\section{Regret Analysis of Global Optimization (deferred from \Cref{sec:global_optimization})}
\label{app:global_opt_proofs}
In this section, we provide the details omitted from \cref{sec:global_optimization}
and complete the regret analysis for \cref{thm:Occup_OPT,thm:Occup_OPT2}.

\subsection{Auxiliary Lemmas}
We first note the unbiasedness of the estimator in \cref{eq:loss_estimate_occup}.
For any state-action pair $(s,a)$,
\begin{align}
    \E_t{\brk*{\hat{\ell}_t(s,a)}} = m_t(s,a) + \frac{\ell_t(s,a) - m_t(s,a)}{q^{\pi_t}(s,a)}\E_t\brk*{\I_t(s,a)} = \ell_t(s,a).
\end{align}
Hence,
using this and $V^\pi(s_0;\ell_t)=\sum_{s,a}q^\pi(s,a)\ell_t(s,a)=\langle q^\pi,\ell_t\rangle$,
we can rewrite the regret as follows:
\begin{align}
\Reg_T
&= \E\brk*{\sum_{t=1}^T V^{\pi_t}(s_0;\ell_t)-\sum_{t=1}^T V^{\pio}(s_0;\ell_t)}\\
&= \E\brk*{\sum_{t=1}^T \inpr*{q^{\pi_t}-\qst, \ell_t}}\\
&= \E\brk*{\sum_{t=1}^T \inpr*{q^{\pi_t}-\qst, \hat{\ell}_t}}. 
\label{eq:regret_unbiased_hatell}
\end{align}

We also recall the loss-shifting technique introduced by \citet{jin2021best}, which is useful to prove logarithmic regret bounds in the stochastic regime.
\begin{lem}[special case of {\citealt[Lemma A.1.1]{jin2021best}}]
	\label{lem:loss_shifting_invariant}
	Fix the transition function $P$. 
    For any policy $\pi$ and loss function $\mathring \ell \colon \calS \times \calA \rightarrow \R$, define the invariant function $g \colon \calS \times \calA \rightarrow \R$ as
	\begin{equation}
	g^{ \pi }(s,a;\mathring{\ell}) \coloneq  Q^{\pi}(s,a;\mathring{\ell}) - V^{\pi}(s;\mathring{\ell}) - {\mathring \ell}(s,a). \label{eq:invariant_func}
	\end{equation}
	Then, it holds for any policy $\pi'$ that 
	\begin{align}
	\inpr*{q^{P, \pi'}(\cdot, \cdot), g^{\pi}(\cdot, \cdot;\mathring{\ell})}
	\coloneq \sum_{s, a} q^{\pi'}(s,a)\, g^{\pi }(s,a;\mathring{\ell})
	=  - V^{\pi}(s_0;\mathring{\ell}),
	\end{align}
	where $ V^{\pi}(s_0;\mathring{\ell})$ only depends on $\pi$ and $\mathring  \ell$ (but not $\pi'$). 
\end{lem}

The following lemma extends \citet[Lemma A.1.2]{jin2021best} from standard FTRL to OFTRL. It immediately follows from \cref{lem:loss_shifting_invariant}.
\begin{lem}
	\label{lem:invariant_with_oftrl}
	Consider the occupancy measure $q^{\pi_t}$ selected by OFTRL with regularizer $\psi_t$, loss sequence $\set{\hat\ell_\tau}_{\tau < t}$, and predictor $m_t$ over the decision set $\Omega(P)$. Then,
	\begin{align}
	q^{\pi_t} = \argmin_{q \in \Omega(P)} \set*{ \inpr*{q, \sum_{\tau=1}^{t-1} \hat\ell_\tau + m_t} + \psi_t(q) }
    = \argmin_{q \in \Omega(P)} \set*{ \inpr*{q, \sum_{\tau=1}^{t-1} (\hat\ell_\tau + g_\tau) + m_t} + \psi_t(q) }.
	\end{align}
	for any sequence of invariant functions $\set{g_\tau}_{\tau < t}$ which are constructed with hypothesized losses $\set{\mathring{\ell}_\tau}_{\tau < t}$ and policies $\set{\pi'_\tau}_{\tau < t}$. 
\end{lem}

Since we use the OFTRL framework, we slightly modify the loss-shifting construction of \citet{jin2021best}.
In their analysis, the invariant function $g^{\pi_t}(s,a;\hat\ell)$ is defined using the estimated loss $\hat\ell_t$. 
By contrast, as can be seen from \cref{lem:OFTRL_log_barrier_occup,lem:OFTRL_log_barrier_policy}, the stability of OFTRL is controlled by the shifted loss $\till_t \coloneq \hat{\ell}_t - m_t$.
Accordingly, we construct the invariant function from $\hat\ell_t - m_t$ so that the invariance property in \cref{lem:invariant_with_oftrl} holds under OFTRL.

To this end, we define $\till_t$ by
\begin{align}
    \till_t(s,a) \coloneq \frac{\I_t(s,a)\prn*{\ell_t(s,a) - m_t(s,a)}}{q^{\pi_t}(s,a)} = \hat\ell_t(s,a) - m_t(s,a).
\end{align}
We then define the corresponding loss-shifting (invariant) function induced by $\till_t$ as
\begin{align}
    g_t(s,a) \coloneq Q^{\pi_t}(s,a;\till_t) - V^{\pi_t}(s;\till_t) - \till_t(s,a).
\end{align}
In what follows, we collect several basic properties of $g_t$ for bounding the regret. All of them hold for an arbitrary loss prediction $m_t\in[0,1]^{S\times A}$.

\begin{lem}
	\label{lem:loss_shifting_estimator1} 
    For any loss prediction $m_t\in[0,1]^{S\times A}$, it holds that
	\begin{align}
		\hatl_t(s,a) + g_t(s,a) - m_t(s,a) \geq \frac{-H}{q^{\pi_t}(s, a)}
	\end{align} 
	for all state-action pairs $(s,a)$.
\end{lem}

\begin{proof}
Fix any state-action pair $(s,a)$. By the definitions of $g_t$ and $\till_t$, we have
\begin{align}
     \hatl_t(s,a) + g_t(s,a) - m_t(s,a)
    &= \hatl_t(s,a) + Q^{\pi_t}(s,a;\till_t) - V^{\pi_t}(s;\till_t) - \till_t(s,a) - m_t(s,a)\\
    &= Q^{\pi_t}(s,a;\till_t) - V^{\pi_t}(s;\till_t)\\
    &= (1 - \pi_t(a\mid s))Q^{\pi_t}(s,a;\till_t) - \sum_{b\neq a} \pi_t(b \mid s)Q^{\pi_t}(s,b;\till_t), \label{eq:loss_shifting_decomp1}
\end{align}
where the last line uses $V^{\pi_t}(s;\till_t)=\sum_{b}\pi_t(b\mid s)Q^{\pi_t}(s,b;\till_t)$.

We first lower bound $Q^{\pi_t}(s,a;\till_t)$. By the definition of the $Q$-function,
\begin{align}
    Q^{\pi_t}(s,a;\till_t) &= \sum_{h = h(s)}^{H-1}\sum_{(s',a') \in \calS_h \times  \calA}q^{\pi_t}(s',a'\mid s,a)\till_t(s',a')\\
    &= \sum_{h = h(s)}^{H-1}\sum_{(s',a') \in \calS_h \times  \calA}q^{\pi_t}(s',a'\mid s,a)\frac{\I_t(s',a')(\ell_t(s',a') - m_t(s',a'))}{q^{\pi_t}(s',a')}\\
    &\geq \frac{-1}{q^{\pi_t}(s,a)}\sum_{h = h(s)}^{H-1}\sum_{(s',a') \in \calS_h \times  \calA}\frac{q^{\pi_t}(s,a) q^{\pi_t}(s',a'\mid s,a)}{q^{\pi_t}(s',a')}\I_t(s',a')\\
    &\geq \frac{-1}{q^{\pi_t}(s,a)}\sum_{h = h(s)}^{H-1}\sum_{(s',a') \in \calS_h \times  \calA} \I_t(s',a') \geq \frac{-H}{q^{\pi_t}(s,a)},
\end{align}
where the third line uses $\ell_t(s',a') - m_t(s',a') \geq -1$,
the fourth line uses that $\frac{q^{\pi_t}(s,a)q^{\pi_t}(s',a'\mid s,a)}{q^{\pi_t}(s',a')}\le 1$,
and the last inequality follows from
$\sum_{(s',a')\in\calS_h\times\calA}\I_t(s',a')\le 1$ for each $h$.

Next, we evaluate the second term of \cref{eq:loss_shifting_decomp1}. 
By the similar argument as above, we have
\begin{align}
    \sum_{b\neq a} \pi_t(b \mid s)Q^{\pi_t}(s,b;\till_t) &= \sum_{b\neq a} \pi_t(b \mid s) \sum_{h = h(s)}^{H-1}\sum_{(s',a') \in \calS_h \times  \calA}q^{\pi_t}(s',a'\mid s,b)\till_t(s',a')\\
    &= \sum_{h = h(s)}^{H-1}\sum_{(s',a') \in \calS_h \times  \calA} \sum_{b\neq a} \pi_t(b \mid s) q^{\pi_t}(s',a'\mid s,b) \frac{\I_t(s',a')(\ell_t(s',a') - m_t(s',a'))}{q^{\pi_t}(s',a')}\\
    &\leq \frac{1}{q^{\pi_t}(s)}\sum_{h = h(s)}^{H-1}\sum_{(s',a') \in \calS_h \times  \calA} \sum_{b\neq a} \frac{q^{\pi_t}(s,b) q^{\pi_t}(s',a'\mid s,b)}{q^{\pi_t}(s',a')}\I_t(s',a')\\
    &\leq \frac{1}{q^{\pi_t}(s)}\sum_{h = h(s)}^{H-1}\sum_{(s',a') \in \calS_h \times  \calA} \I_t(s',a') \leq \frac{H}{q^{\pi_t}(s)}
    .
\end{align}
Combining the two bounds yields
\begin{align}
    \hatl_t(s,a) + g_t(s,a) - m_t(s,a) &\geq (1 - \pi_t(a \mid s))\frac{-H}{q^{\pi_t}(s,a)} - \frac{H}{q^{\pi_t}(s)} \\
    &= \frac{-H}{q^{\pi_t}(s,a)}\prn*{(1 - \pi_t(a \mid s)) + \pi_t(a \mid s)} = \frac{-H}{q^{\pi_t}(s,a)}.
\end{align}
\end{proof}

\begin{lem}
	\label{lem:loss_shifting_estimator2} 
    For an arbitrary loss prediction $m_t\in[0,1]^{S\times A}$, it holds that
	\begin{align}
		\E_t\brk*{ \prn*{ \hatl_t(s,a) + g_t(s,a) - m_t(s,a) }^2 } \leq \frac{2H^2}{q^{\pi_t}(s,a)}\prn*{1 - \pi_t(a\mid s)}
	\end{align} 
	for all state-action pairs $(s,a)$.
\end{lem}

\begin{proof}
Fix any $(s,a)$. By the definitions of $g_t$ and $\till_t$, we have
\begin{align}
    \E_t\brk*{\prn*{ \hatl_t(s,a) + g_t(s,a) - m_t(s,a) }^2}
    &= \E_t\brk*{\prn*{ \hatl_t(s,a) + Q^{\pi_t}(s,a;\till_t) - V^{\pi_t}(s;\till_t) - \till_t(s,a) - m_t(s,a) }^2 }\\
    &= \E_t\brk*{\prn*{ Q^{\pi_t}(s,a;\till_t) - V^{\pi_t}(s;\till_t) }^2 }\\
    &= \E_t\brk*{\prn*{ (1 - \pi_t(a\mid s))Q^{\pi_t}(s,a;\till_t) - \sum_{b\neq a} \pi_t(b \mid s)Q^{\pi_t}(s,b;\till_t)}^2} \\
    &\leq 2\E_t\brk*{(1 - \pi_t(a\mid s))^2Q^{\pi_t}(s,a;\till_t)^2 + \prn*{\sum_{b\neq a} \pi_t(b \mid s)Q^{\pi_t}(s,b;\till_t) }^2}, \label{eq:loss_shifting_decomp}
\end{align}
where the last inequality follows from $(x-y)^2\leq 2(x^2+y^2)$.

By the definition of the $Q$-function, the first term in \cref{eq:loss_shifting_decomp} is evaluated as
\begin{align}
    \E_t\brk*{Q^{\pi_t}(s,a;\till_t)^2}
    &= \E_t\brk*{\prn*{\sum_{h = h(s)}^{H-1}\sum_{(s',a') \in \calS_h \times \calA}q^{\pi_t}(s',a'\mid s,a)\till_t(s',a')}^2} \\ 
    &= \E_t\brk*{\prn*{\sum_{h = h(s)}^{H-1}\sum_{(s',a') \in \calS_h \times  \calA}q^{\pi_t}(s',a'\mid s,a)\frac{\I_t(s',a')(\ell_t(s',a') - m_t(s',a'))}{q^{\pi_t}(s',a')}}^2} \\
    &\leq H\E_t\brk*{\sum_{h = h(s)}^{H-1}\sum_{(s',a') \in \calS_h \times  \calA}q^{\pi_t}(s',a'\mid s,a)^2\frac{\I_t(s',a')(\ell_t(s',a') - m_t(s',a'))^2}{q^{\pi_t}(s',a')^2}},
\end{align}
where the last inequality applies the Cauchy--Schwarz inequality across at most $H$ stages combined with the fact that $\sum_{(s,a) \in \calS_h\times \calA} \I_t(s,a) \leq 1$. Then, this can be further bounded as
 \begin{align}
    & H\E_t\brk*{\sum_{h = h(s)}^{H-1}\sum_{(s',a') \in \calS_h \times  \calA}q^{\pi_t}(s',a'\mid s,a)^2\frac{\I_t(s',a')(\ell_t(s',a') - m_t(s',a'))^2}{q^{\pi_t}(s',a')^2}}\\
    &\leq H\sum_{h = h(s)}^{H-1}\sum_{(s',a') \in \calS_h \times  \calA}\frac{q^{\pi_t}(s',a'\mid s,a)^2}{q^{\pi_t}(s',a')}\\
    &\leq \frac{H}{q^{\pi_t}(s,a)}\sum_{h = h(s)}^{H-1}\sum_{(s',a') \in \calS_h \times  \calA}\frac{q^{\pi_t}(s,a)q^{\pi_t}(s',a'\mid s,a)}{q^{\pi_t}(s',a')}\cdot q^{\pi_t}(s',a'\mid s,a)\\
    &\leq \frac{H}{q^{\pi_t}(s,a)}\sum_{h = h(s)}^{H-1}\sum_{(s',a') \in \calS_h \times  \calA} q^{\pi_t}(s',a'\mid s,a) \leq \frac{H^2}{q^{\pi_t}(s,a)}, 
\end{align}
where the last inequality follows from
$\sum_{h = h(s)}^{H-1}\sum_{(s',a') \in \calS_h \times \calA} q^{\pi_t}(s',a'\mid s,a)\leq H$.
Consequently, we obtain
\begin{align}
    \E_t\brk*{Q^{\pi_t}(s,a;\till_t)^2} \leq \frac{H^2}{q^{\pi_t}(s,a)}. \label{eq:first_term_bound_loss_shifting_estimator}
\end{align}
For the second term in \cref{eq:loss_shifting_decomp}, by repeating the similar arguments, we have
\begin{align}
    &\E_t\brk*{\prn*{\sum_{b\neq a} \pi_t(b \mid s)Q^{\pi_t}(s,b;\till_t) }^2} \\ 
    &= \E_t\brk*{\prn*{\sum_{h = h(s)}^{H-1}\sum_{(s',a') \in \calS_h \times  \calA}\prn*{\sum_{b\neq a} \pi_t(b \mid s) q^{\pi_t}(s',a'\mid s,b)}\till_t(s',a')}^2} \\
    &= \E_t\brk*{\prn*{\sum_{h = h(s)}^{H-1}\sum_{(s',a') \in \calS_h \times  \calA}\prn*{\sum_{b\neq a} \pi_t(b \mid s) q^{\pi_t}(s',a'\mid s,b)}\frac{\I_t(s',a')(\ell_t(s',a') - m_t(s',a'))}{q^{\pi_t}(s',a')}}^2} \\
    &\leq H\E_t\brk*{\sum_{h = h(s)}^{H-1}\sum_{(s',a') \in \calS_h \times  \calA}\prn*{\sum_{b\neq a} \pi_t(b \mid s) q^{\pi_t}(s',a'\mid s,b)}^2\frac{\I_t(s',a') (\ell_t(s',a') - m_t(s',a'))^2}{q^{\pi_t}(s',a')^2}},
\end{align}
where the last inequality applies the Cauchy--Schwarz inequality across at most $H$ stages, 
combined with the fact that $\sum_{(s,a) \in \calS_h\times \calA} \I_t(s,a) \leq 1$.
This can be further bounded as
 \begin{align}
    & H\E_t\brk*{\sum_{h = h(s)}^{H-1}\sum_{(s',a') \in \calS_h \times  \calA}\prn*{\sum_{b\neq a} \pi_t(b \mid s) q^{\pi_t}(s',a'\mid s,b)}^2\frac{\I_t(s',a') (\ell_t(s',a') - m_t(s',a'))^2}{q^{\pi_t}(s',a')^2}}\\
    &\leq H\sum_{h = h(s)}^{H-1}\sum_{(s',a') \in \calS_h \times  \calA}\frac{\prn*{\sum_{b\neq a} \pi_t(b \mid s) q^{\pi_t}(s',a'\mid s,b)}^2}{q^{\pi_t}(s',a')}\\
    &= \frac{H}{q^{\pi_t}(s)}\sum_{h = h(s)}^{H-1}\sum_{(s',a') \in \calS_h \times  \calA}\frac{\sum_{c\neq a} q^{\pi_t}(s,c) q^{\pi_t}(s',a'\mid s,c)}{q^{\pi_t}(s',a')}
    \prn*{\sum_{b\neq a} \pi_t(b \mid s) q^{\pi_t}(s',a'\mid s,b)}\\
    &\leq \frac{H}{q^{\pi_t}(s)}\sum_{h = h(s)}^{H-1}\sum_{(s',a') \in \calS_h \times  \calA}\sum_{b\neq a} \pi_t(b \mid s) q^{\pi_t}(s',a'\mid s,b)\\
    &= \frac{H}{q^{\pi_t}(s)}\sum_{b\neq a} \pi_t(b \mid s) \sum_{h = h(s)}^{H-1}\sum_{(s',a') \in \calS_h \times  \calA}q^{\pi_t}(s',a'\mid s,b)\\
    &\leq \frac{H^2}{q^{\pi_t}(s)}\sum_{b\neq a} \pi_t(b \mid s) = \frac{H^2}{q^{\pi_t}(s)}(1 - \pi_t(a \mid s)).
\end{align}
Consequently, we obtain
\begin{align}
    \E_t\brk*{\prn[\Bigg]{\sum_{b\neq a} \pi_t(b \mid s)Q^{\pi_t}(s,b;\till_t) }^2}
    &\leq \frac{H^2}{q^{\pi_t}(s)}(1 - \pi_t(a \mid s)) . \label{eq:second_term_bound_loss_shifting_estimator}
\end{align}

Finally, combining \cref{eq:loss_shifting_decomp,eq:first_term_bound_loss_shifting_estimator,eq:second_term_bound_loss_shifting_estimator}, we have
\begin{align}
    \E_t\brk*{\prn*{ \hatl_t(s,a) + g_t(s,a) - m_t(s,a) }^2}
    &\leq 2\prn*{(1 - \pi_t(a\mid s))^2\frac{H^2}{q^{\pi_t}(s,a)} + \frac{H^2}{q^{\pi_t}(s)}(1 - \pi_t(a \mid s))} \\
    &\leq \frac{2H^2}{q^{\pi_t}(s,a)}\prn*{(1 - \pi_t(a\mid s))^2 + \pi_t(a\mid s) (1 - \pi_t(a \mid s))} \\
    &= \frac{2H^2}{q^{\pi_t}(s,a)}\prn*{1 - \pi_t(a\mid s)},
\end{align}
which is the desired bound.
\end{proof}

We next extend \cref{lem:loss_shifting_estimator2} to derive a variance-aware upper bound. The additional terms that arise can be controlled, and become $\polylog(T)$ when $m_t$ is chosen as in \cref{def:predictor_sequence2} (see \cref{lem:lower_order_stochastic}).

\begin{lem}
	\label{lem:loss_shifting_estimator3} 
    Under the stochastic regime with adversarial corruption, for an arbitrary loss prediction $m_t\in[0,1]^{S\times A}$, it holds that
	\begin{align}
		&\E_t\brk*{ \prn*{ \hatl_t(s,a) + g_t(s,a) - m_t(s,a) }^2 }\\
        &\leq \frac{4H\Varmax(s)}{q^{\pi_t}(s,a)}(1 - \pi_t(a \mid s)) +  \frac{4H}{q^{\pi_t}(s,a)^2}\sum_{s',a'}q^{\pi_t}(s',a')\E_t\brk*{(\ell_t(s', a') - \ell'_t(s',a') + \mu(s',a') -  m_t(s',a'))^2}
	\end{align} 
	for all state-action pairs $(s,a)$, where we recall that $\Varmax(s)$ is defined in \cref{def:variance-to-go}.
\end{lem}

\begin{proof}
Fix any $(s,a)$. Define 
$\kappa_t(s,a) = \ell'_t(s,a) - \mu(s,a)$ and $\lambda_t(s,a) = \ell_t(s, a) - \ell'_t(s,a) + \mu(s,a) -  m_t(s,a)$ so that 
\begin{align}
    \ell_t(s,a) - m_t(s,a) = \kappa_t(s,a) + \lambda_t(s,a).
\end{align}
By the definitions of $g_t$ and $\till_t$, we have
\begin{align}
    \E_t\brk*{\prn*{ \hatl_t(s,a) + g_t(s,a) - m_t(s,a) }^2}
    &= \E_t\brk*{\prn*{ \hatl_t(s,a) + Q^{\pi_t}(s,a;\till_t) - V^{\pi_t}(s;\till_t) - \till_t(s,a) - m_t(s,a) }^2 }\\
    &= \E_t\brk*{\prn*{ Q^{\pi_t}(s,a;\till_t) - V^{\pi_t}(s;\till_t) }^2 }\\
    &= \E_t\brk*{\prn*{ (1 - \pi_t(a\mid s))Q^{\pi_t}(s,a;\till_t) - \sum_{b\neq a} \pi_t(b \mid s)Q^{\pi_t}(s,b;\till_t)}^2} \\
    &\leq 2\E_t\brk*{(1 - \pi_t(a\mid s))^2Q^{\pi_t}(s,a;\till_t)^2 + \prn*{\sum_{b\neq a} \pi_t(b \mid s)Q^{\pi_t}(s,b;\till_t) }^2}, \label{eq:loss_shifting_decomp_sto}
\end{align}
where the last inequality follows from $(x-y)^2\leq 2(x^2+y^2)$.

Then, we bound the two expectations on the right-hand side of \cref{eq:loss_shifting_decomp_sto} in turn.
\paragraph{Bounding $\E_t\brk*{Q^{\pi_t}(s,a;\till_t)^2}$ (the first term in \cref{eq:loss_shifting_decomp_sto}).}
Using the definition of the $Q$-function,  we obtain
\begin{align}
    &\E_t\brk*{Q^{\pi_t}(s,a;\till_t)^2}\\
    &= \E_t\brk*{\prn*{\sum_{s',a'}q^{\pi_t}(s',a'\mid s,a)\till_t(s',a')}^2} \\ 
    &= \E_t\brk*{\prn*{\sum_{s',a'}q^{\pi_t}(s',a'\mid s,a)\frac{\I_t(s',a')(\ell_t(s',a') - m_t(s',a'))}{q^{\pi_t}(s',a')}}^2} \\
    &= \E_t\brk*{\prn*{\sum_{s',a'}\frac{q^{\pi_t}(s',a'\mid s,a)}{q^{\pi_t}(s',a')}\I_t(s',a')(\kappa_t(s',a') + \lambda_t(s',a'))}^2} \\
    &\leq 2\E_t\brk*{\prn*{\sum_{s',a'}\frac{q^{\pi_t}(s',a'\mid s,a)}{q^{\pi_t}(s',a')}\I_t(s',a')\kappa_t(s',a')}^2} + 2\E_t\brk*{\prn*{\sum_{s',a'}\frac{q^{\pi_t}(s',a'\mid s,a)}{q^{\pi_t}(s',a')}\I_t(s',a')\lambda_t(s',a')}^2}, \label{eq:Q-func_decomp_sto}
\end{align}
where the last inequality follows from $(x+y)^2\leq 2(x^2+y^2)$.

For the first term ($\kappa$-term) of \cref{eq:Q-func_decomp_sto}, by using  $\E_t\brk*{\kappa_t(s,a)^2} = \sigma^2(s,a)$, we have
\begin{align}
    2\E_t\brk*{\prn*{\sum_{s',a'}\frac{q^{\pi_t}(s',a'\mid s,a)}{q^{\pi_t}(s',a')}\I_t(s',a')\kappa_t(s',a')}^2} 
    &\leq 2H\E_t\brk*{\sum_{s',a'}\frac{q^{\pi_t}(s',a'\mid s,a)^2}{q^{\pi_t}(s',a')^2}\I_t(s',a')\kappa_t(s',a')^2}\\
    &= 2H\sum_{s',a'}\frac{q^{\pi_t}(s',a'\mid s,a)^2}{q^{\pi_t}(s',a')}\sigma^2(s',a')\\
    &\leq \frac{2H}{q^{\pi_t}(s,a)}\sum_{s',a'}\frac{q^{\pi_t}(s,a)q^{\pi_t}(s',a'\mid s,a)}{q^{\pi_t}(s',a')}q^{\pi_t}(s',a'\mid s,a)\sigma^2(s',a')\\
    &\leq \frac{2H}{q^{\pi_t}(s,a)}\sum_{s',a'}q^{\pi_t}(s',a'\mid s,a)\sigma^2(s',a')\\
    &\leq \frac{2H\Varmax(s)}{q^{\pi_t}(s,a)}, \label{eq:Q-func_kappa}
\end{align}
where the first inequality applies the Cauchy--Schwarz inequality across at most $H$ stages, combined with the fact that $\sum_{(s,a) \in \calS_h\times \calA} \I_t(s,a) \leq 1$.

Similarly, for the second term (the $\lambda$-term) of \cref{eq:Q-func_decomp_sto}, we have
\begin{align}
    &2\E_t\brk*{\prn*{\sum_{s',a'}\frac{q^{\pi_t}(s',a'\mid s,a)}{q^{\pi_t}(s',a')}\I_t(s',a')\lambda_t(s',a')}^2}\\
    &\leq 2H\E_t\brk*{\sum_{s',a'}\frac{q^{\pi_t}(s',a'\mid s,a)^2}{q^{\pi_t}(s',a')^2}\I_t(s',a')\lambda_t(s',a')^2} \\ 
    &= 2H\sum_{s',a'}\frac{q^{\pi_t}(s',a'\mid s,a)^2}{q^{\pi_t}(s',a')}\E_t\brk*{\lambda_t(s',a')^2}\\
    &= \frac{2H}{q^{\pi_t}(s,a)}\sum_{s',a'}\frac{q^{\pi_t}(s,a)q^{\pi_t}(s',a'\mid s,a)}{q^{\pi_t}(s',a')}q^{\pi_t}(s',a'\mid s,a)\E_t\brk*{\lambda_t(s',a')^2}\\
    &\leq \frac{2H}{q^{\pi_t}(s,a)}\sum_{s',a'}q^{\pi_t}(s',a'\mid s,a)\E_t\brk*{\lambda_t(s',a')^2}\\
    &\leq \frac{2H}{q^{\pi_t}(s,a)^2}\sum_{s',a'}q^{\pi_t}(s',a')\E_t\brk*{\lambda_t(s',a')^2}, \label{eq:Q-func_lambda}
\end{align}
where the first inequality applies the Cauchy--Schwarz inequality, and the last inequality follows from $q^{\pi_t}(s',a'\mid s,a)\leq \frac{q^{\pi_t}(s',a')}{q^{\pi_t}(s,a)}$.

Consequently, combining \cref{eq:Q-func_decomp_sto,eq:Q-func_kappa,eq:Q-func_lambda} yields
\begin{align}
    &\E_t\brk*{Q^{\pi_t}(s,a;\till_t)^2}\leq \frac{2H\Varmax(s)}{q^{\pi_t}(s,a)} + \frac{2H}{q^{\pi_t}(s,a)^2}\sum_{s',a'}q^{\pi_t}(s',a')\E_t\brk*{\lambda_t(s',a')^2}. \label{eq:Q-func_result}
\end{align}
\paragraph{Bounding $\E_t\brk*{\prn*{\sum_{b\neq a} \pi_t(b \mid s)Q^{\pi_t}(s,b;\till_t) }^2}$ (the second term in \cref{eq:loss_shifting_decomp_sto}).} By repeating the similar arguments,
\begin{align}
    \E_t\brk*{\prn*{\sum_{b\neq a} \pi_t(b \mid s)Q^{\pi_t}(s,b;\till_t) }^2} 
    &= \E_t\brk*{\prn*{\sum_{s',a'}\prn*{\sum_{b\neq a} \pi_t(b \mid s) q^{\pi_t}(s',a'\mid s,b)}\till_t(s',a')}^2} \\
    &= \E_t\brk*{\prn*{\sum_{s',a'}\prn*{\sum_{b\neq a} \pi_t(b \mid s) q^{\pi_t}(s',a'\mid s,b)}\frac{\I_t(s',a')(\ell_t(s',a') - m_t(s',a'))}{q^{\pi_t}(s',a')}}^2} \\
    &= \E_t\brk*{\prn*{\sum_{s',a'}\frac{\sum_{b\neq a} \pi_t(b \mid s) q^{\pi_t}(s',a'\mid s,b)}{q^{\pi_t}(s',a')}\I_t(s',a') (\kappa_t(s',a') + \lambda_t(s',a'))}^2}\\
    &\leq 2\E_t\brk*{\prn*{\sum_{s',a'}\frac{\sum_{b\neq a} \pi_t(b \mid s) q^{\pi_t}(s',a'\mid s,b)}{q^{\pi_t}(s',a')}\I_t(s',a') \kappa_t(s',a')}^2}\\
    &\qquad + 2\E_t\brk*{\prn*{\sum_{s',a'}\frac{\sum_{b\neq a} \pi_t(b \mid s) q^{\pi_t}(s',a'\mid s,b)}{q^{\pi_t}(s',a')}\I_t(s',a') \lambda_t(s',a')}^2}, \label{eq:V-func_decomp_sto}
\end{align}

For the first term (the $\kappa$-term) in \cref{eq:V-func_decomp_sto}, 
repeating the same argument as in \cref{eq:Q-func_kappa}, we obtain
\begin{align}
    &2\E_t\brk*{\prn*{\sum_{s',a'}\frac{\sum_{b\neq a} \pi_t(b \mid s) q^{\pi_t}(s',a'\mid s,b)}{q^{\pi_t}(s',a')}\I_t(s',a') \kappa_t(s',a')}^2}\\ 
    &\leq 2H\E_t\brk*{\sum_{s',a'}\frac{\prn*{\sum_{b \neq a} \pi_t(b \mid s)q^{\pi_t}(s',a'\mid s,b)}^2}{q^{\pi_t}(s',a')^2}\I_t(s',a')\kappa_t(s',a')^2}\\
    &= 2H\sum_{s',a'}\frac{\prn*{\sum_{b \neq a} \pi_t(b \mid s)q^{\pi_t}(s',a'\mid s,b)}^2}{q^{\pi_t}(s',a')}\sigma^2(s',a')\\
    &\leq \frac{2H}{q^{\pi_t}(s)}\sum_{s',a'}\frac{\sum_{b \neq a} q^{\pi_t}(s,b)q^{\pi_t}(s',a'\mid s,b)}{q^{\pi_t}(s',a')}\prn*{\sum_{b \neq a} \pi_t(b \mid s)q^{\pi_t}(s',a'\mid s,b)}\sigma^2(s',a')\\
    &\leq \frac{2H}{q^{\pi_t}(s)}\sum_{s',a'}\sum_{b \neq a} \pi_t(b \mid s)q^{\pi_t}(s',a'\mid s,b)\sigma^2(s',a')\\
    &\leq \frac{2H}{q^{\pi_t}(s)}\sum_{b \neq a} \pi_t(b \mid s)\sum_{s',a'}q^{\pi_t}(s',a'\mid s,b)\sigma^2(s',a')\\
    &\leq \frac{2H}{q^{\pi_t}(s)}\sum_{b \neq a} \pi_t(b \mid s)\Varmax(s)\\
    &\leq \frac{2H\Varmax(s)}{q^{\pi_t}(s)}(1 - \pi_t(a \mid s)). \label{eq:V-func_kappa}
\end{align}
The second term (the $\lambda$-term) in \cref{eq:V-func_decomp_sto} can be bounded similarly by
\begin{align}
    &2\E_t\brk*{\prn*{\sum_{s',a'}\frac{\sum_{b\neq a} \pi_t(b \mid s) q^{\pi_t}(s',a'\mid s,b)}{q^{\pi_t}(s',a')}\I_t(s',a') \lambda_t(s',a')}^2}\\
    &\leq 2H\E_t\brk*{\sum_{s',a'}\frac{\prn*{\sum_{b \neq a} \pi_t(b \mid s)q^{\pi_t}(s',a'\mid s,b)}^2}{q^{\pi_t}(s',a')^2}\I_t(s',a')\lambda_t(s',a')^2}\\
    &= 2H\sum_{s',a'}\frac{\prn*{\sum_{b \neq a} \pi_t(b \mid s)q^{\pi_t}(s',a'\mid s,b)}^2}{q^{\pi_t}(s',a')}\E_t\brk*{\lambda_t(s',a')^2}\\
    &= \frac{2H}{q^{\pi_t}(s)}\sum_{s',a'}\frac{\sum_{c \neq a} q^{\pi_t}(s,c)q^{\pi_t}(s',a'\mid s,c)}{q^{\pi_t}(s',a')}
    \prn*{\sum_{b \neq a} \pi_t(b \mid s)q^{\pi_t}(s',a'\mid s,b)}\E_t\brk*{\lambda_t(s',a')^2}\\
    &\leq \frac{2H}{q^{\pi_t}(s)}\sum_{s',a'}\sum_{b \neq a} \pi_t(b \mid s)q^{\pi_t}(s',a'\mid s,b)\E_t\brk*{\lambda_t(s',a')^2}\\
    &\leq \frac{2H}{q^{\pi_t}(s)^2}\sum_{s',a'}q^{\pi_t}(s',a')\E_t\brk*{\lambda_t(s',a')^2}, \label{eq:V-func_lambda}
\end{align}
where the last inequality follows from $\sum_{b \neq a} \pi_t(b \mid s)q^{\pi_t}(s',a'\mid s,b)\leq \frac{q^{\pi_t}(s',a')}{q^{\pi_t}(s)}$.

Consequently, combining \cref{eq:V-func_decomp_sto,eq:V-func_kappa,eq:V-func_lambda} yields
\begin{align}
   \E_t\brk*{\prn*{\sum_{b\neq a} \pi_t(b \mid s)Q^{\pi_t}(s,b;\till_t) }^2}
    &\leq \frac{2H\Varmax(s)}{q^{\pi_t}(s)}(1 - \pi_t(a \mid s)) + \frac{2H}{q^{\pi_t}(s)^2}\sum_{s',a'}q^{\pi_t}(s',a')\E_t\brk*{\lambda_t(s',a')^2}. \label{eq:V-func_result}
\end{align}
Therefore, by \cref{eq:loss_shifting_decomp_sto,eq:Q-func_result,eq:V-func_result},
\begin{align}
    &\E_t\brk*{\prn*{ \hatl_t(s,a) + g_t(s,a) - m_t(s,a) }^2}\\
    &\leq \frac{4H\Varmax(s)}{q^{\pi_t}(s,a)}(1 - \pi_t(a\mid s))^2 + \frac{4H(1 - \pi_t(a\mid s))^2}{q^{\pi_t}(s,a)^2}\sum_{s',a'}q^{\pi_t}(s',a')\E_t\brk*{\lambda_t(s',a')^2}\\
    &\qquad + \frac{4H\Varmax(s)}{q^{\pi_t}(s)}(1 - \pi_t(a \mid s)) + \frac{4H}{q^{\pi_t}(s)^2}\sum_{s',a'}q^{\pi_t}(s',a')\E_t\brk*{\lambda_t(s',a')^2}\\
    &= \frac{4H\Varmax(s)}{q^{\pi_t}(s,a)}(1 - \pi_t(a \mid s))(1 - \pi_t(a\mid s) + \pi_t(a \mid s))\\
    &\qquad +  \frac{4H((1 - \pi_t(a\mid s))^2 + \pi_t(a\mid s)^2)}{q^{\pi_t}(s,a)^2}\sum_{s',a'}q^{\pi_t}(s',a')\E_t\brk*{\lambda_t(s',a')^2}\\
    &\leq \frac{4H\Varmax(s)}{q^{\pi_t}(s,a)}(1 - \pi_t(a \mid s)) +  \frac{4H}{q^{\pi_t}(s,a)^2}\sum_{s',a'}q^{\pi_t}(s',a')\E_t\brk*{\lambda_t(s',a')^2},
\end{align}
and this completes the proof.
\end{proof}

\begin{cor}\label{cor:loss_shifting_estimator4} 
    In the stochastic regime with adversarial corruption, suppose that the uncorrupted losses are generated independently and are uncorrelated across layers. Then, it holds that
    \begin{align}
		&\E_t\brk*{ \prn*{ \hatl_t(s,a) + g_t(s,a) - m_t(s,a) }^2 }\\
        &\leq \frac{4\Varmax(s)}{q^{\pi_t}(s,a)}(1 - \pi_t(a \mid s)) +  \frac{4H}{q^{\pi_t}(s,a)^2}\sum_{s',a'}q^{\pi_t}(s',a')\E_t\brk*{(\ell_t(s', a') - \ell'_t(s',a') + \mu(s',a') -  m_t(s',a'))^2}
	\end{align} 
	for all state-action pairs $(s,a)$.
\end{cor}
\begin{proof}
    This corollary can be viewed as a simple variant of~\cref{lem:loss_shifting_estimator3}.
    Let $\kappa_t(s,a) \coloneqq \ell'_t(s,a)-\mu(s,a)$.
    Since the uncorrupted losses are generated independently and are uncorrelated
    across layers, it holds that for any $(s_1,a_1)\neq(s_2,a_2)$,
    \begin{align}
        \E_t\brk*{\I_t(s_1,a_1)\I_t(s_2,a_2)\kappa_t(s_1,a_1)\kappa_t(s_2,a_2)} = 0.
    \end{align}
    Then, for any function $\alpha : \calS\times \calA\to\R$, we have
    \begin{align}
    \prn*{\sum_{s,a}\alpha(s,a)\I_t(s,a)\kappa_t(s,a)}^2
    &= \sum_{s_1,a_1}\sum_{s_2,a_2} \alpha(s_1,a_1)\alpha(s_2,a_2)\I_t(s_1,a_1)\I_t(s_2,a_2)\kappa_t(s_1,a_1)\kappa_t(s_2,a_2)\\
    &= \sum_{s,a} \alpha(s,a)^2\I_t(s,a)\kappa_t(s,a)^2. \label{eq:uncorrupt_kappa}
    \end{align}
    Thus, for the first term ($\kappa$-term) of \cref{eq:Q-func_decomp_sto}, we have
    \begin{align}
        2\E_t\brk*{\prn*{\sum_{s',a'}\frac{q^{\pi_t}(s',a'\mid s,a)}{q^{\pi_t}(s',a')}\I_t(s',a')\kappa_t(s',a')}^2} 
        &\leq 2\E_t\brk*{\sum_{s',a'}\frac{q^{\pi_t}(s',a'\mid s,a)^2}{q^{\pi_t}(s',a')^2}\I_t(s',a')\kappa_t(s',a')^2} \tag{by \cref{eq:uncorrupt_kappa}}\\
        &= 2\sum_{s',a'}\frac{q^{\pi_t}(s',a'\mid s,a)^2}{q^{\pi_t}(s',a')}\sigma^2(s',a')\\
        &\leq \frac{2}{q^{\pi_t}(s,a)}\sum_{s',a'}\frac{q^{\pi_t}(s,a)q^{\pi_t}(s',a'\mid s,a)}{q^{\pi_t}(s',a')}q^{\pi_t}(s',a'\mid s,a)\sigma^2(s',a')\\
        &\leq \frac{2}{q^{\pi_t}(s,a)}\sum_{s',a'}q^{\pi_t}(s',a'\mid s,a)\sigma^2(s',a')\\
        &\leq \frac{2\Varmax(s)}{q^{\pi_t}(s,a)}. \label{eq:Q-func_kappa_uncorrupt}
    \end{align}

For the first term (the $\kappa$-term) in \cref{eq:V-func_decomp_sto}, repeating the same argument as in \cref{eq:Q-func_kappa_uncorrupt}, we also obtain
\begin{align}
    &2\E_t\brk*{\prn*{\sum_{s',a'}\frac{\sum_{b\neq a} \pi_t(b \mid s) q^{\pi_t}(s',a'\mid s,b)}{q^{\pi_t}(s',a')}\I_t(s',a') \kappa_t(s',a')}^2}\\ 
    &\leq 2\E_t\brk*{\sum_{s',a'}\frac{\prn*{\sum_{b \neq a} \pi_t(b \mid s)q^{\pi_t}(s',a'\mid s,b)}^2}{q^{\pi_t}(s',a')^2}\I_t(s',a')\kappa_t(s',a')^2}\tag{by \cref{eq:uncorrupt_kappa}}\\
    &\leq \frac{2\Varmax(s)}{q^{\pi_t}(s)}(1 - \pi_t(a \mid s)). \label{eq:V-func_kappa_uncorrupt}
\end{align}
Therefore, compared with \cref{lem:loss_shifting_estimator3}, we obtain an
$H$-times sharper bound in \cref{eq:Q-func_kappa_uncorrupt,eq:V-func_kappa_uncorrupt}
than \cref{eq:Q-func_kappa,eq:V-func_kappa}. As a consequence, the corresponding
$\Varmax$ term is also improved by a factor of $H$.
\end{proof}

\begin{lem}
\label{lem:learning_rate_occup}
Suppose that the learning rates are updated according to \cref{eq:eta_update_occup}. Then, it holds that 
\begin{align}
    \eta_t(s,a) &\leq \frac{\sqrt{\ln(T)}}{\sqrt{2H^2\ln(T) + \sum_{\tau = 1}^t\zeta_\tau(s,a)}}
\end{align}
for any episode $t$ and state-action pair $(s,a)$.
\end{lem}

\begin{proof}
By the update rule of the learning rate \cref{eq:eta_update_occup},
\begin{align}
    \frac{1}{\eta_{t + 1}(s,a)^2}
    &= \prn*{\frac{1}{\eta_{t}(s,a)} + \frac{\eta_t(s,a)}{\ln(T)}\zeta_t(s,a)}^2 \\
    &\geq \frac{1}{\eta_{t}(s,a)^2} + \frac{2}{\ln(T)}\zeta_t(s,a).
\end{align}
Repeatedly applying the above inequality yields
\begin{align}
    \frac{1}{\eta_{t}(s,a)^2} \geq \frac{1}{\eta_1^2} + \sum_{\tau = 1}^{t-1}\frac{2}{\ln(T)}\zeta_\tau(s,a).
\end{align}
Taking reciprocals and then taking square roots yields
\begin{align}
    \eta_t(s,a) &\leq \frac{1}{\sqrt{\eta_1^{-2} + \sum_{\tau = 1}^{t-1}\frac{2}{\ln(T)}\zeta_\tau(s,a)}} \\
    &\leq \frac{\sqrt{\ln(T)}}{\sqrt{2}\sqrt{\frac{1}{2}\eta_1^{-2}\ln(T) + \sum_{\tau = 1}^{t - 1}\zeta_\tau(s,a)}}\\
    &\leq \frac{\sqrt{\ln(T)}}{\sqrt{\frac{1}{2}\eta_1^{-2}\ln(T) + \sum_{\tau = 1}^{t}\zeta_\tau(s,a)}},
\end{align}
where the last inequality follows from $\zeta_t(s,a)\leq 1 \leq \frac{1}{2}\eta_1^{-2}\ln(T)=2H^2\ln(T)$ for $T\geq 2$.
Finally, using $\frac{1}{2}\eta_1^{-2}\ln(T)=2H^2\ln(T)$, we obtain
\begin{align}
    \eta_t(s,a)
    \leq \frac{\sqrt{\ln(T)}}{\sqrt{2H^2\ln(T) + \sum_{\tau = 1}^{t}\zeta_\tau(s,a)}}.
\end{align}
\end{proof}

\subsection{Common Regret Analysis}

\begin{lem}
\label{lem:occup_regterm}
\cref{alg:OFTRL_GO} guarantees 
\begin{align}
     \Reg_T \lesssim HSA\ln(T) + \sum_{s,a} \sqrt{\ln(T)\E\brk*{\sum_{t = 1}^{T}\zeta_t(s,a)}},
\end{align}
where $\zeta_t(s,a) = q^{\pi_t}(s,a)^2\min\set*{(\hat\ell_t(s,a)-m_t(s,a))^2, (\hat\ell_t(s,a) + g_t(s,a)-m_t(s,a))^2}$.
\end{lem}
\begin{proof}
From \cref{eq:regret_unbiased_hatell}, we have 
$
\Reg_T
=
\E\brk[\Big]{\sum_{t=1}^T \inpr{q^{\pi_t}-\qst, \hat{\ell}_t}},
$
and we will apply \cref{lem:OFTRL_log_barrier_occup} with $p_t = q^{\pi_t}$ and $\ell_t \in \set{\hat{\ell}_t, \hat{\ell}_t + g_t}$ combined with \Cref{lem:invariant_with_oftrl}.
To do so, we will check the conditions of \Cref{lem:OFTRL_log_barrier_occup}.
For any $(s,a)$, we have
\begin{align}
    \eta_t(s,a)q^{\pi_t}(s,a)(\hatl_t(s,a) - m_t(s,a))
    &= \eta_t(s,a)q^{\pi_t}(s,a)\frac{\I_t(s,a)(\ell_t(s,a) - m_t(s,a))}{q^{\pi_t}(s,a)} \\
    &\geq -\eta_t(s,a) \tag{by $\ell_t(s,a) - m_t(s,a) \geq -1$}\\
    &\geq -\eta_1 \geq -\frac{1}{2}, \tag{by $\eta_1 = \frac{1}{2H}$}
\end{align}
and
\begin{align}
    \eta_t(s,a)q^{\pi_t}(s,a)(\hatl_t(s,a) + g_t(s,a) - m_t(s,a))
    &\geq \eta_t(s,a)q^{\pi_t}(s,a)\frac{-H}{q^{\pi_t}(s, a)} \tag{by \cref{lem:loss_shifting_estimator1}} \\
    &= -H\eta_t(s,a) \\
    &\geq -H\eta_1 = -\frac{1}{2}. \tag{by $\eta_1 = \frac{1}{2H}$}
\end{align}
Moreover, define
\begin{align}
    \tilq = \frac{1}{SA}\sum_{s,a} q^{\max}_{s,a}
    \in \Omega(P),
\end{align}
where $q^{\max}_{s,a}$ denotes the occupancy measure induced by a policy that maximizes the probability of visiting the state-action pair $(s,a)$ under transition kernel $P$.

Therefore, by \cref{lem:OFTRL_log_barrier_occup,lem:invariant_with_oftrl}, we obtain
\begin{align}
    &\E\brk*{\sum_{t=1}^T\inpr*{q^{\pi_t} - \qst, \hatl_t}}\\
    &\leq \frac{SA\ln(SAT)}{\eta_1} 
    + \E\brk*{\sum_{t=1}^T \sum_{s,a} \prn*{\frac{1}{\eta_{t+1}(s,a)} - \frac{1}{\eta_t(s,a)}}\ln(SAT)}\\
    &\qquad +  \E\brk*{\sum_{t=1}^T  \sum_{s,a}  \eta_t(s,a)q^{\pi_t}(s,a)^2 \min\set*{(\hat\ell_t(s,a)-m_t(s,a))^2, (\hat\ell_t(s,a) + g_t(s,a)-m_t(s,a))^2}}\\
    &\qquad + \E\brk*{\frac{1}{T}\sum_{t=1}^T \inpr*{ -\qst + \tilq, \hat\ell_t }} + 2H\E\brk*{\nrm*{m_{T+1}}_{\infty}}\\
    &\leq \frac{3SA\ln(T)}{\eta_1} + 2H + 2H + 3\E\brk*{\sum_{t=1}^T \sum_{s,a} \prn*{\frac{1}{\eta_{t+1}(s,a)} - \frac{1}{\eta_t(s,a)}}\ln(T)} 
    \tag{by $T\geq S$, $T\geq A$, and $\nrm*{m_{T+1}}_{\infty} \leq 1$} \\
    &\qquad +  \E\brk*{\sum_{t=1}^T  \sum_{s,a}  \eta_t(s,a)\zeta_t(s,a)}\\
    &\leq \frac{3SA\ln(T)}{\eta_1} + 4H + 4\E\brk*{\sum_{t=1}^T \sum_{s,a}  \eta_t(s,a)\zeta_t(s,a)}.
\end{align}
Here, the second inequality follows from
\begin{align}
    \E\brk*{\frac{1}{T}\sum_{t=1}^T \inpr*{ -\qst + \tilq, \hat\ell_t }} 
    = 
    \frac{1}{T}\inpr*{ -\qst + \tilq, \sum_{t=1}^T\ell_t } 
    \leq \frac{1}{T}\nrm*{-\qst + \tilq}_1\nrm*{\sum_{t=1}^T\ell_t }_\infty \leq \frac{2HT}{T} = 2H
    ,
\end{align}
where we used $\nrm*{-\qst + \tilq}_1\leq 2H$ and $\nrm*{\sum_{t=1}^T\ell_t }_\infty\leq T$,
and the last inequality follows from the update rule of the learning rate \cref{eq:eta_update_occup} and the definition of $\zeta_t(s,a)$.

It remains to bound $\sum_{t=1}^T \sum_{s,a} \eta_t(s,a)\zeta_t(s,a)$.
From \cref{lem:learning_rate_occup},
\begin{align}
    &\sum_{t=1}^T \sum_{s,a} \eta_t(s,a)\zeta_t(s,a)\\
    &\leq \sqrt{\ln(T)}\sum_{t=1}^T \sum_{s,a} \frac{\zeta_t(s,a)}{\sqrt{2H^2\ln(T) + \sum_{\tau = 1}^t\zeta_\tau(s,a)}}\\
    &\leq 2\sqrt{\ln(T)}\sum_{t=1}^T \sum_{s,a} \prn*{\sqrt{2H^2\ln(T) + \sum_{\tau = 1}^{t}\zeta_\tau(s,a)} - \sqrt{2H^2\ln(T) + \sum_{\tau = 1}^{t - 1}\zeta_\tau(s,a)}}\\
    &= 2\sqrt{\ln(T)}\sum_{s,a} \prn*{\sqrt{2H^2\ln(T) + \sum_{\tau = 1}^{T}\zeta_\tau(s,a)} - \sqrt{2H^2\ln(T)}} \\
    &\leq 2\sqrt{\ln(T)}\sum_{s,a} \sqrt{\sum_{t = 1}^{T}\zeta_t(s,a)},
\end{align}
where the second inequality follows from
\begin{align}
    &2\prn*{\sqrt{2H^2\ln(T) + \sum_{\tau = 1}^{t}\zeta_\tau(s,a)} - \sqrt{2H^2\ln(T) + \sum_{\tau = 1}^{t - 1}\zeta_\tau(s,a)}}\\
    &= \frac{2\zeta_t(s,a)}{\sqrt{2H^2\ln(T) + \sum_{\tau = 1}^{t}\zeta_\tau(s,a)} + \sqrt{2H^2\ln(T) + \sum_{\tau = 1}^{t - 1}\zeta_\tau(s,a)}}\\
    &\geq \frac{\zeta_t(s,a)}{\sqrt{2H^2\ln(T) + \sum_{\tau = 1}^{t}\zeta_\tau(s,a)}}.
\end{align}
Therefore, combining the above argument with \cref{eq:regret_unbiased_hatell}, we obtain
\begin{align}
     \Reg_T &= \E\brk*{\sum_{t=1}^T\inpr*{q^{\pi_t} - \qst, \hatl_t}} \\
     &\leq \frac{3SA\ln(T)}{\eta_1} + 4H + 8\sum_{s,a} \sqrt{\ln(T)\E\brk*{\sum_{t = 1}^{T}\zeta_t(s,a)}}\\
     &\lesssim HSA\ln(T) + \sum_{s,a} \sqrt{\ln(T)\E\brk*{\sum_{t = 1}^{T}\zeta_t(s,a)}},
\end{align}
which completes the proof.
\end{proof}

\subsection{Proof of \protect\cref{thm:Occup_OPT}}

Now we are ready to prove \cref{thm:Occup_OPT}.
\begin{thm}[Restatement of \cref{thm:Occup_OPT}]
\cref{alg:OFTRL_GO} with the loss prediction $m_t$ defined in \cref{def:predictor_sequence} guarantees
\begin{align}
    \Reg_T \lesssim \sqrt{SA\ln(T)\min\set*{L^\star, HT - L^\star, Q_{\infty}, V_1}} + HSA\ln(T).
\end{align}
Under the stochastic regime with adversarial corruption, it simultaneously ensures
\begin{align}
    \Reg_T \lesssim \sqrt{SA\ln(T)\prn*{\V T + \calC}} + HSA\ln(T),
\end{align}
and
\begin{align}
    \Reg_T \lesssim U + \sqrt{U\calC} + HSA\ln(T),
\end{align}
where $U = \sum_{s}\sum_{a\neq\pi^\star(s)}\frac{H^2\ln(T)}{\Delta(s,a)}$.
\end{thm}

\begin{proof}
We start from \cref{lem:occup_regterm}, which gives
\begin{align}
    \Reg_T
    \lesssim HSA\ln(T) + \sum_{s,a}\sqrt{\ln(T)\E\brk*{\sum_{t = 1}^{T}\zeta_t(s,a)}} 
    .
    \label{eq:reg_bound_occup_in_proof}
\end{align}
By the definition of $\zeta_t(s,a)$, we have
\begin{align}
    &\sum_{s,a}\sqrt{\ln(T)\E\brk*{\sum_{t = 1}^{T}\zeta_t(s,a)}}\\
    &= \sum_{s,a}\sqrt{\ln(T)\E\brk*{\sum_{t = 1}^{T}q^{\pi_t}(s,a)^2\min\set*{(\hat\ell_t(s,a)-m_t(s,a))^2, (\hat\ell_t(s,a) + g_t(s,a)-m_t(s,a))^2}}}\\
    &\leq \sqrt{SA\ln(T)\E\brk*{\sum_{t = 1}^{T}\sum_{s,a}q^{\pi_t}(s,a)^2(\hatl_t(s,a)-m_t(s,a))^2}} \tag{by the Cauchy--Schwarz inequality}\\
    &= \sqrt{SA\ln(T)\E\brk*{\sum_{t = 1}^{T}\sum_{s,a}\I_t(s,a)(\ell_t(s,a)-m_t(s,a))^2}}, \label{eq:common_bound_occup}
\end{align}
where the last equality uses $(\hat\ell_t(s,a)-m_t(s,a))^2
= \frac{\I_t(s,a)(\ell_t(s,a)-m_t(s,a))^2}{q^{\pi_t}(s,a)^2}$.

\paragraph{1. Bounds for the adversarial regime.}
By \cref{lem:general_predict_result}, we can evaluate \cref{eq:common_bound_occup} as
\begin{align}
    &\sqrt{SA\ln(T)\E\brk*{\sum_{t = 1}^{T}\sum_{s,a}\I_t(s,a)(\ell_t(s,a)-m_t(s,a))^2}} \\
    &\lesssim \sqrt{SA\ln(T)\prn*{\min\set*{L^\star  + \Reg_T, HT - L^\star  - \Reg_T, Q_{\infty}, V_1} +  SA}}\\
    &\leq \sqrt{SA\ln(T)\min\set*{L^\star  + \Reg_T, HT - L^\star  - \Reg_T, Q_{\infty}, V_1}} + SA\sqrt{\ln(T)}.
\end{align}
Absorbing the lower-order term into $HSA\ln(T)$, we obtain
\begin{align}
    \Reg_T &\lesssim \sqrt{SA\ln(T)\prn*{L^\star  + \Reg_T}} + HSA\ln(T) 
    ,
    \label{eq:occup_mid_regret_first1}\\
    \Reg_T &\lesssim \sqrt{SA\ln(T)\prn*{HT - L^\star  - \Reg_T}} + HSA\ln(T) , \label{eq:occup_mid_regret_first2}\\
    \Reg_T &\lesssim \sqrt{SA\ln(T)Q_{\infty}} + HSA\ln(T),  \label{eq:occup_regret_second}\\
    \Reg_T &\lesssim \sqrt{SA\ln(T)V_1} + HSA\ln(T).
    \label{eq:occup_regret_path}
\end{align}
From \cref{eq:occup_mid_regret_first1}, 
\begin{align}
    \Reg_T &\leq c\sqrt{SA\ln(T)L^\star} + c\sqrt{SA\ln(T)\Reg_T} + cHSA\ln(T) \tag{for some absolute constant $c$}\\ 
    &\leq c\sqrt{SA\ln(T)L^\star} + \frac{c^2}{2}SA\ln(T) + \frac{1}{2}\Reg_T + cHSA\ln(T) \\
    &\leq \frac{1}{2}\Reg_T + O(\sqrt{SA\ln(T)L^\star} + HSA\ln(T)),
\end{align}
where the second line follows from the AM--GM inequality. Therefore, 
\begin{align}
    \Reg_T &\lesssim \sqrt{SA\ln(T)L^\star} + HSA\ln(T) 
    . 
    \label{eq:occup_regret_first1}
\end{align}
From \cref{eq:occup_mid_regret_first2}, we also have
\begin{align}
    \Reg_T &\lesssim \sqrt{SA\ln(T)\prn*{HT - L^\star  - \Reg_T}} + HSA\ln(T)\\
    &\leq \sqrt{SA\ln(T)\prn*{HT - L^\star}} + HSA\ln(T) \label{eq:occup_regret_first2}
\end{align}

Combining \cref{eq:occup_regret_first1,eq:occup_regret_first2,eq:occup_regret_second,eq:occup_regret_path}, we obtain
\begin{align}
    \Reg_T \lesssim \sqrt{SA\ln(T)\min\set*{L^\star, HT - L^\star, Q_{\infty}, V_1}} + HSA\ln(T)
\end{align}

\paragraph{2. Stochastic variance bound.}
In the stochastic regime, combining \cref{eq:reg_bound_occup_in_proof,eq:common_bound_occup} with \cref{lem:general_predict_result} implies
\begin{align}
    \Reg_T \lesssim \sqrt{SA\ln(T)\prn*{\V T + \calC}} + HSA\ln(T).
\end{align}

\paragraph{3. Stochastic gap-dependent bound.}
We can evaluate \cref{eq:common_bound_occup} as
\begin{align}
    &\sum_{s,a}\sqrt{\ln(T)\E\brk*{\sum_{t = 1}^{T}\zeta_t(s,a)}}\\
    &= \sum_{s,a}\sqrt{\ln(T)\E\brk*{\sum_{t = 1}^{T}q^{\pi_t}(s,a)^2\min\set*{(\hat\ell_t(s,a)-m_t(s,a))^2, (\hat\ell_t(s,a) + g_t(s,a)-m_t(s,a))^2}}}\\
    &\leq \sum_{s,a}\sqrt{\ln(T)\E\brk*{\sum_{t = 1}^{T}q^{\pi_t}(s,a)^2(\hat\ell_t(s,a) + g_t(s,a)-m_t(s,a))^2}}\\
    &\leq \sum_{s,a}\sqrt{\ln(T)\E\brk*{\sum_{t = 1}^{T}q^{\pi_t}(s,a)^2 {\frac{2H^2}{q^{\pi_t}(s,a)}\prn*{1 - \pi_t(a\mid s)}}}} \tag{by \cref{lem:loss_shifting_estimator2}}\\
    &= \sqrt{2}H\sum_{s,a}\sqrt{\ln(T)\E\brk*{\sum_{t = 1}^{T}q^{\pi_t}(s,a)\prn*{1 - \pi_t(a\mid s)}}}\\
    &\leq \sqrt{2}H\sum_{s}\sum_{a \neq \pist(s)}\sqrt{\ln(T)\E\brk*{\sum_{t = 1}^{T}q^{\pi_t}(s,a)}} + \sqrt{2}H\sum_{s}\sqrt{\ln(T)\E\brk*{\sum_{t = 1}^{T}q^{\pi_t}(s)\prn*{1 - \pi_t(\pist(s)\mid s)}}}\\
    &= \sqrt{2}H\sum_{s}\sum_{a \neq \pist(s)}\sqrt{\ln(T)\E\brk*{\sum_{t = 1}^{T}q^{\pi_t}(s,a)}} + \sqrt{2}H\sum_{s}\sqrt{\ln(T)\sum_{t = 1}^{T}\E\brk*{\sum_{a \neq \pist(s)}q^{\pi_t}(s)\pi_t(a\mid s)}}\\
    &\leq 2\sqrt{2}H\sum_{s}\sum_{a \neq \pist(s)}\sqrt{\ln(T)\E\brk*{\sum_{t = 1}^{T}q^{\pi_t}(s,a)}}.
\end{align}
Hence, combining this with \cref{eq:reg_bound_occup_in_proof}, we obtain
\begin{align}
    \Reg_T \lesssim H\sum_{s}\sum_{a \neq \pist(s)}\sqrt{\ln(T)\E\brk*{\sum_{t = 1}^{T}q^{\pi_t}(s,a)}} + HSA\ln(T).
\end{align}
Finally, applying \cref{general_self_bounding} to the last inequality yields
\begin{align}
    \Reg_T \lesssim U + \sqrt{U\calC} + HSA\ln(T),
\end{align}
where $U = \sum_{s}\sum_{a\neq\pi^\star(s)}\frac{H^2\ln(T)}{\Delta(s,a)}$.
\end{proof}

\subsection{Proof of \protect\cref{thm:Occup_OPT2}}
Here we provide the proof of \cref{thm:Occup_OPT2}.
\begin{thm}[Restatement of \cref{thm:Occup_OPT2}]
\cref{alg:OFTRL_GO} with the loss prediction $m_t$ defined in \cref{def:predictor_sequence2} guarantees
\begin{align}
    \Reg_T \lesssim \sqrt{SA\ln(T)\min\set*{L^\star, HT - L^\star, Q_{\infty}}} + HSA\ln(T).
\end{align}
Under the stochastic regime with adversarial corruption, it simultaneously ensures
\begin{align}
    \Reg_T \lesssim \sqrt{SA\ln(T)\prn*{\V T + \calC}} + HSA\ln(T),
\end{align}
and
\begin{align}
    \Reg_T \lesssim \Uvar + \sqrt{\Uvar\calC} + \sqrt{H S^2 A^2 \calC}\ln(T) + H^{\frac12}S^{\frac32}A^{\frac32}\ln^{\frac32}(T),
\end{align}
where $\Uvar = \sum_{s}\sum_{a\neq\pi^\star(s)}\frac{H\Varmax(s)\ln(T)}{\Delta(s,a)}$.
\end{thm}

\begin{proof}
The proof follows the same argument as \cref{thm:Occup_OPT}. 
The main differences are that the stochastic gap-dependent bound becomes variance-aware, at the cost of not deriving a path-length bound.

We start from \cref{lem:occup_regterm}, which gives
\begin{align}
    \Reg_T
    \lesssim HSA\ln(T) + \sum_{s,a}\sqrt{\ln(T)\E\brk*{\sum_{t = 1}^{T}\zeta_t(s,a)}} \label{eq:reg_bound_occup_in_proof2}.
\end{align}
By the definition of $\zeta_t(s,a)$, the same argument as in \cref{eq:common_bound_occup} yields
\begin{align}
    \sum_{s,a}\sqrt{\ln(T)\E\brk*{\sum_{t = 1}^{T}\zeta_t(s,a)}}
    &\leq \sqrt{SA\ln(T)\E\brk*{\sum_{t = 1}^{T}\sum_{s,a}\I_t(s,a)(\ell_t(s,a)-m_t(s,a))^2}}, \label{eq:common_bound_occup2}
\end{align}

\paragraph{1. Bounds for the adversarial regime.}
Applying \cref{lem:general_predict_result2} to \cref{eq:common_bound_occup2} gives
\begin{align}
    &\sqrt{SA\ln(T)\E\brk*{\sum_{t = 1}^{T}\sum_{s,a}\I_t(s,a)(\ell_t(s,a)-m_t(s,a))^2}} \\
    &\lesssim \sqrt{SA\ln(T)\prn*{\min\set*{L^\star  + \Reg_T, HT - L^\star  - \Reg_T, Q_{\infty}} +  SA\ln(T) + SA}}\\
    &\lesssim \sqrt{SA\ln(T)\min\set*{L^\star  + \Reg_T, HT - L^\star  - \Reg_T, Q_{\infty}}} + SA\ln(T).
\end{align}
Absorbing the lower-order term into $HSA\ln(T)$, we obtain
\begin{align}
    \Reg_T &\lesssim \sqrt{SA\ln(T)\prn*{L^\star  + \Reg_T}} + HSA\ln(T) \\
    \Reg_T &\lesssim \sqrt{SA\ln(T)\prn*{HT - L^\star  - \Reg_T}} + HSA\ln(T) \\
    \Reg_T &\lesssim \sqrt{SA\ln(T)Q_{\infty}} + HSA\ln(T)
\end{align}
Applying the same calculation as in \cref{eq:occup_mid_regret_first1,eq:occup_mid_regret_first2,eq:occup_regret_second} gives
\begin{align}
    \Reg_T \lesssim \sqrt{SA\ln(T)\min\set*{L^\star, HT - L^\star, Q_{\infty}}} + HSA\ln(T).
\end{align}

\paragraph{2. Stochastic variance bound.}
Under the stochastic regime, \cref{lem:general_predict_result2} further implies
\begin{align}
    \Reg_T \lesssim \sqrt{SA\ln(T)\prn*{\V T + \calC}} + HSA\ln(T).
\end{align}

\paragraph{3. Stochastic gap-dependent bound.}
Moreover, by using \cref{lem:lower_order_stochastic,lem:loss_shifting_estimator3}, we have
\begin{align}
    &\sum_{s,a}\sqrt{\ln(T)\E\brk*{\sum_{t = 1}^{T}\zeta_t(s,a)}}\\
    &= \sum_{s,a}\sqrt{\ln(T)\E\brk*{\sum_{t = 1}^{T}q^{\pi_t}(s,a)^2\min\set*{(\hat\ell_t(s,a)-m_t(s,a))^2, (\hat\ell_t(s,a) + g_t(s,a)-m_t(s,a))^2}}}\\
    &\leq \sum_{s,a}\sqrt{\ln(T)\E\brk*{\sum_{t = 1}^{T}q^{\pi_t}(s,a)^2(\hat\ell_t(s,a) + g_t(s,a)-m_t(s,a))^2}} \label{eq:before_variance_occup}\\
    &\leq \sum_{s,a}\sqrt{\ln(T)\E\brk*{\sum_{t = 1}^{T}4H\Varmax(s)q^{\pi_t}(s,a)\prn*{1 - \pi_t(a\mid s)}}} \label{eq:after_variance_occup}\\
    &\qquad + \sum_{s,a}\sqrt{\ln(T)\E\brk*{4H\sum_{t = 1}^{T}\sum_{s',a'}q^{\pi_t}(s',a')(\ell_t(s', a') - \ell'_t(s',a') + \mu(s',a') -  m_t(s',a'))^2}}\tag{by \cref{lem:loss_shifting_estimator3}}\\
    &\lesssim \sum_{s,a}\sqrt{\ln(T)\E\brk*{\sum_{t = 1}^{T}H\Varmax(s)q^{\pi_t}(s,a)\prn*{1 - \pi_t(a\mid s)}}} + \sum_{s,a}\sqrt{H\ln(T)\prn*{SA\ln^2(T) + \calC\ln(T)}} \tag{by \cref{lem:lower_order_stochastic}}\\
    &\lesssim \sum_{s,a}\sqrt{\ln(T)\E\brk*{\sum_{t = 1}^{T}H\Varmax(s)q^{\pi_t}(s,a)\prn*{1 - \pi_t(a\mid s)}}} + \sqrt{H S^2 A^2 \calC}\ln(T) + H^{\frac12}S^{\frac32}A^{\frac32}\ln^{\frac32}(T).\label{eq:variance-aware_occup}
\end{align}
Further, the first term in \cref{eq:variance-aware_occup} can be rewritten as
\begin{align}
    &\sum_{s,a}\sqrt{\ln(T)\E\brk*{\sum_{t = 1}^{T}H\Varmax(s)q^{\pi_t}(s,a)\prn*{1 - \pi_t(a\mid s)}}} \\
    &\leq \sum_{s}\sqrt{H\Varmax(s)}\sum_{a \neq \pist(s)}\sqrt{\ln(T)\E\brk*{\sum_{t = 1}^{T}q^{\pi_t}(s,a)}}\\  
    &\qquad +\sum_{s}\sqrt{H\Varmax(s)}\sqrt{\ln(T)\E\brk*{\sum_{t = 1}^{T}q^{\pi_t}(s)\prn*{1 - \pi_t(\pist(s)\mid s)}}}\\
    &= \sum_{s}\sqrt{H\Varmax(s)}\sum_{a \neq \pist(s)}\sqrt{\ln(T)\E\brk*{\sum_{t = 1}^{T}q^{\pi_t}(s,a)}}\\  
    &\qquad + \sum_{s}\sqrt{H\Varmax(s)}\sqrt{\ln(T)\sum_{t = 1}^{T}\E\brk*{\sum_{a \neq \pist(s)}q^{\pi_t}(s)\pi_t(a\mid s)}}\\
    &\leq \sum_{s}2\sqrt{H\Varmax(s)}\sum_{a \neq \pist(s)}\sqrt{\ln(T)\E\brk*{\sum_{t = 1}^{T}q^{\pi_t}(s,a)}}.
\end{align}
Hence, combining the last inequality with \cref{eq:reg_bound_occup_in_proof2}, we obtain
\begin{align}
    \Reg_T \lesssim \sum_{s}\sqrt{H\Varmax(s)}\sum_{a \neq \pist(s)}\sqrt{\ln(T)\E\brk*{\sum_{t = 1}^{T}q^{\pi_t}(s,a)}} + \sqrt{H S^2 A^2 \calC}\ln(T) + H^{\frac12}S^{\frac32}A^{\frac32}\ln^{\frac32}(T).
\end{align}
Finally, applying \cref{general_self_bounding} to the last inequality yields
\begin{align}
    \Reg_T \lesssim \Uvar + \sqrt{\Uvar C} + \sqrt{H S^2 A^2 \calC}\ln(T) + H^{\frac12}S^{\frac32}A^{\frac32}\ln^{\frac32}(T),
\end{align}
where $\Uvar = \sum_{s}\sum_{a\neq\pi^\star(s)}\frac{H\Varmax(s)\ln(T)}{\Delta(s,a)}$.
\end{proof}

\begin{rem}[Restatement of \cref{rem:Occup_OPT2}] 
    In the stochastic regime with adversarial corruption, suppose that the uncorrupted losses are generated independently and are uncorrelated across layers, \cref{thm:Occup_OPT2} improves by a factor of $H$ to $\Uvar = \sum_{s}\sum_{a\neq\pi^\star(s)}\frac{\Varmax(s)\ln(T)}{\Delta(s,a)}$.
\end{rem}
\begin{proof}
    In the proof of \cref{thm:Occup_OPT2}, applying \cref{cor:loss_shifting_estimator4} to \cref{eq:before_variance_occup} yields the following inequality in place of \cref{eq:after_variance_occup}:
    \begin{align}
        &\sum_{s,a}\sqrt{\ln(T)\E\brk*{\sum_{t = 1}^{T}q^{\pi_t}(s,a)^2(\hat\ell_t(s,a) + g_t(s,a)-m_t(s,a))^2}}\\
        &\leq \sum_{s,a}\sqrt{\ln(T)\E\brk*{\sum_{t = 1}^{T}4\Varmax(s)q^{\pi_t}(s,a)\prn*{1 - \pi_t(a\mid s)}}}\\
        &\qquad + \sum_{s,a}\sqrt{\ln(T)\E\brk*{4H\sum_{t = 1}^{T}\sum_{s',a'}q^{\pi_t}(s',a')(\ell_t(s', a') - \ell'_t(s',a') + \mu(s',a') -  m_t(s',a'))^2}}.
    \end{align}
    Compared to \cref{eq:after_variance_occup}, this bound is improved by a factor of $H$, and can be interpreted as replacing the $H\Varmax(s)$ term by $\Varmax(s)$.
    The remainder of the proof follows by the same steps as in \cref{thm:Occup_OPT2}.
\end{proof}
\section{Regret Analysis of Policy Optimization (deferred from \Cref{sec:policy_optimization})}
\label{app:policy_opt_proofs}
In this section, we provide the missing details from \cref{sec:policy_optimization} and present the full regret analysis leading to the proof of \cref{thm:Policy_OPT,thm:Policy_OPT2}.

Throughout this section, sums over $t=1,\ldots,T$ are taken over the augmented sequence including both real and virtual episodes, whereas the data-dependent complexity measures in the main statements are defined only over real episodes. We denote by $\calT_r=\{t\in[T]:Y_t=1\}$ and $\calT_v=\{t\in[T]:Y_t=0\}$ the sets of real and virtual episodes, with cardinalities $\abs{\calT_r}$ and $\abs{\calT_v}$, respectively.

\subsection{Auxiliary Lemmas}
Building on the policy optimization framework of \citet{luo2021policy} and \citet{dann2023best}, we use the following key lemma to derive our regret bounds.

\begin{lem}[Restatement of \cref{lem:dilated_bonus_main}]
\label{lem:dilated_bonus}
    Suppose that $b_t(s)$ is a nonnegative loss function and that, for all $s,a$, 
    \begin{align}
        B_t(s,a) = b_t(s) + \prn*{1 + \frac{1}{H}}\E_{s'\sim P(\cdot \mid s, a), a'\sim\pi_t(\cdot \mid s')}\brk*{B_t(s',a')}
        .
    \end{align}
    Suppose also that for some $J(s) \geq 0$ it holds that
    \begin{align}
        &\E\brk*{\sum_s \qst(s)\sum_{t,a} \prn*{\pi_t(a\mid s) - \pio(a\mid s) } \prn*{Q^{\pi_t}(s,a;\ell_t) -  B_t(s,a)}}\\
        &\leq \sum_{s} \qst(s) J(s) + \E\brk*{\sum_{t = 1}^T\sum_s \qst(s) b_t(s)} + \E\brk*{\frac{1}{H} \sum_{t = 1}^T\sum_s\sum_a \qst(s) \pi_t(a\mid s)B_t(s,a)}
        .\label{eq:key_lemma_ineq}
    \end{align}
    Then, 
    \begin{align}
        \Reg_T \leq \sum_{s} \qst(s) J(s) + 3 \, \E\brk*{\sum_{t = 1}^T V^{\pi_t}(s_0;b_t)}.
    \end{align}
\end{lem}

\cref{lem:dilated_bonus} reduces the regret analysis to proving \cref{eq:key_lemma_ineq} for an appropriate bonus $b_t(s)$ and its dilated version $B_t(s,a)$.
Here $B_t(s,a)$ is the exploration bonus in $Q$-space, and $b_t(s)$ is the one-step bonus that generates it.
Once \eqref{eq:key_lemma_ineq} is established, the regret is controlled by the cumulative values $\sum_{t}V^{\pi_t}(s_0;b_t)$.

To show \cref{eq:key_lemma_ineq}, we choose $b_t$ in \cref{def:bonus_term} and decompose the LHS of \cref{eq:key_lemma_ineq} as
\begin{align}
    &\sum_s \qst(s) \sum_{t,a} \prn*{\pi_t(a\mid s) -  \pio(a\mid s) } \prn*{Q^{\pi_t}(s,a;\ell_t) -  B_t(s,a)}\\
    &= \sum_s \qst(s)  \underbrace{\sum_{t,a}  \prn*{\pi_t(a\mid s) -  \pio(a\mid s)} \prn*{\hat{Q}_t(s,a) -  B_t(s,a)}}_{\regterm(s)}\\
    &\qquad\qquad + \sum_s \qst(s) \underbrace{\sum_{t,a} \prn*{\pi_t(a\mid s) -  \pio(a\mid s)} \prn*{Q^{\pi_t}(s,a;\ell_t) -  \hat{Q}_t(s,a)}}_{\biasterm(s)}.
    \label{eq:decompose_PO}
\end{align}

\begin{lem}\label{lem:expect_hatQ_proof}
    It holds that
    \begin{align}
        \E_t\brk*{\hat{Q}_t(s,a)} =  Q^{\pi_t}(s, a; m_t)  + \frac{q^{\pi_t}(s)}{q_t(s)}Q^{\pi_t}(s, a; \ell_t-m_t)Y_t  - \frac{\gamma_t H}{q_t(s)}.
    \end{align}
    for all state-action pairs $(s,a)$.
\end{lem}
\begin{proof}
By the definition of \cref{eq:Q-estimate_policy}, we have
        \begin{align}
        \E_t\brk*{\hat{Q}_t(s,a)} &= \E_t\brk*{Q^{\pi_t}(s,a;m_t) + \frac{\I_t(s,a)(L_{t,h(s)} - M_{t,h(s)})}{q_t(s)\pi_t(a \mid s)}Y_t - \frac{\gamma_t H}{q_t(s)}}\\
        &=Q^{\pi_t}(s, a; m_t) + q_t(s,a)\E_t\brk*{\frac{L_{t,h(s)} - M_{t,h(s)}}{q_t(s)\pi_t(a \mid s)} \relmiddle| \I_t(s,a) = 1}Y_t - \frac{\gamma_t H}{q_t(s)}\\
        &=Q^{\pi_t}(s, a; m_t) + \frac{q^{\pi_t}(s,a)}{q_t(s)\pi_t(a\mid s)}Q^{\pi_t}(s, a; \ell_t-m_t)Y_t  - \frac{\gamma_t H}{q_t(s)}\\
        &= Q^{\pi_t}(s, a; m_t)  + \frac{q^{\pi_t}(s)}{q_t(s)}Q^{\pi_t}(s, a; \ell_t-m_t)Y_t  - \frac{\gamma_t H}{q_t(s)}. \label{eq:expect_hatQ}
    \end{align}
\end{proof}

\begin{lem}\label{lem:bound_eta_pi_B_known}
The variables $b_t(s)$ in \cref{def:bonus_term} and $B_t(s,a)$ in \cref{def:dilated_bonus} satisfy
\begin{align}
    \eta_t(s,a)\pi_t(a\mid s)B_t(s,a) \leq \frac{1}{5H}, \quad
    B_t(s,a) \leq \frac{2\sqrt{HS}}{\gamma_t} + 15H^2
    \label{eq:bound_eta_pi_B_known}
\end{align}
for any episode $t$ and state-action pair $(s,a)$.
\end{lem}

\begin{proof}
Let $R_t = \max_{s,a}\frac{\eta_t(s,a)}{q_t(s)}$.
We first note that the dilated bonus-to-go $B_t(s,a)$ is bounded via the dilated recursion.
Unrolling it for at most $H$ steps and using $(1+1/H)^H \le 3$, we obtain
\begin{align}
    B_t(s,a) \leq 3\sum_{h=h(s)}^{H-1}\sum_{s'\in\mathcal S_h} q^{\pi_t}(s'\mid s,a)\,b_t(s'). \label{eq:dilated_bonus_upper}
\end{align}

Then, we first consider the case when $t$ is a real episode.
In real episodes, by the definition of $b_t$ and the learning-rate update,
\begin{align}
    b_t(s) 
    &= 6\sum_a  \prn*{\frac{1}{\eta_{t+1}(s, a)} - \frac{1}{\eta_t(s, a)}}\ln(T) + 5\frac{\gamma_t H}{q_t(s)}\\
    &= 6\sum_{a} \frac{\eta_t(s,a)\zeta_t(s,a)}{q_t(s)^2} + 5H \\
    &= 6\sum_{a}  \frac{\eta_t(s,a)\prn*{\I_t(s,a)-\pi_t(a\mid s)\I_t(s)}^2\prn*{L_{t,h(s)}-M_{t,h(s)}}^2}{q_t(s)^2} + 5H \\
    &\leq \frac{6H^2}{q_t(s)^2}\max_a\eta_t(s,a)
     \sum_a\prn*{\I_t(s,a)-\pi_t(a\mid s)\I_t(s)}^2 + 5H\\ 
    &\leq \frac{12H^2}{q_t(s)^2}\max_a\eta_t(s,a) + 5H\\ 
    &\leq \frac{12H^2}{q_t(s)}\max_{a}\frac{\eta_t(s, a)}{q_t(s)} + 5H. \label{eq:bonus_term_upper_real}
\end{align}

Using \cref{eq:bonus_term_upper_real,eq:dilated_bonus_upper}, we have
\allowdisplaybreaks
\begin{align}
    B_t(s,a) &\leq 3\sum_{h = h(s)}^{H-1}\sum_{s' \in \calS_h} q^{\pi_t}(s'\mid s,a)b_t(s')\\ 
    &\leq 36H^2R_t \prn*{\sum_{h = h(s)}^{H-1}\sum_{s' \in \calS_h}  q^{\pi_t}(s'\mid s,a)\frac{1}{q_t(s')}} + 15H^2 \\
    &\leq 36H^2R_t \prn*{\sum_{h = h(s)}^{H-1}\sum_{s' \in \calS_h} q^{\pi_t}(s'\mid s,a)\frac{1}{q^{\pi_t}(s,a) q^{\pi_t}(s'\mid s,a) + \gamma_t}} + 15H^2\\
    &\leq 36H^2R_t \prn*{\sum_{h = h(s)}^{H-1}\sum_{s' \in \calS_h}\frac{1}{q^{\pi_t}(s,a) + \gamma_t}} + 15H^2 \\
    &\leq 36H^2SR_t\cdot\frac{1}{q^{\pi_t}(s,a) + \gamma_t} + 15H^2 \label{eq:bound_B_real}\\
    & \leq \frac{2\sqrt{HS}}{\gamma_t} + 15H^2,
\end{align}
where in the last inequality we used $R_t \leq \frac{1}{18\sqrt{H^3S}}$ that holds in real episodes.
This is the desired second inequality in \cref{eq:bound_eta_pi_B_known}.

By using \cref{eq:bound_B_real}, we also have
\begin{align}
    \eta_t(s,a)\pi_t(a \mid s)B_t(s,a) 
    &\leq 36H^2SR_t\cdot\frac{\eta_t(s,a)\pi_t(a\mid s)}{q^{\pi_t}(s,a) + \gamma_t} + 15\eta_1 H^2\\ 
    &\leq 36H^2SR_t\cdot\frac{\eta_t(s,a)}{q_t(s)} + 15\eta_1 H^2\\ 
    &\leq 36H^2SR_t^2 + 15\eta_1 H^2\\ 
    &\leq \frac{1}{9H} + \frac{1}{12H} \leq \frac{1}{5H},
\end{align}
where we used $\frac{\eta_t(s,a)}{q_t(s)}\leq R_t$ and $\eta_1\leq \frac{1}{180H^3}$.
This is the desired first inequality in \cref{eq:bound_eta_pi_B_known}.

We next consider the case when $t$ is a virtual episode.
In a virtual episode, only the single pair $(s_t^\dagger,a_t^\dagger)$ is updated, and thus
\begin{align}
        b_t(s) &= 6\sum_a  \prn*{\frac{1}{\eta_{t+1}(s, a)} - \frac{1}{\eta_t(s, a)}}\ln(T) + 5\frac{\gamma_t H}{q_t(s)}\\
        &= \sum_{a} \frac{\ind\{(s^\dagger_t, a^\dagger_t) = (s,a)\}}{54\eta_t(s,a)H} + 5H \\
        &= \sum_{a} \frac{\ind\{(s^\dagger_t, a^\dagger_t) = (s,a)\}}{54Hq_t(s)} \cdot \frac{1}{\max_{s', a'}\frac{\eta_t(s', a')}{q_t(s')}} + 5H 
        \tag{since $(s^\dagger_t, a^\dagger_t) \in \argmax_{s,a} \frac{\eta_t(s,a)}{q_t(s)}$}
        \\
        &= \frac{\ind\{s^\dagger_t = s\}}{q_t(s)}\cdot \frac{1}{54H\max_{s', a'}\frac{\eta_t(s', a')}{q_t(s')}} + 5H. \label{eq:bonus_term_upper_virtual}
\end{align}

Using \cref{eq:bonus_term_upper_virtual,eq:dilated_bonus_upper}, we have
\begin{align}
    B_t(s,a) &\leq 3\sum_{h = h(s)}^{H-1}\sum_{s' \in \calS_h} q^{\pi_t}(s'\mid s,a)b_t(s')\\ 
    &\leq \frac{1}{18HR_t} {\sum_{h = h(s)}^{H-1}\sum_{s' \in \calS_h}  q^{\pi_t}(s'\mid s,a)\frac{\ind\set{s^\dagger_t = s'}}{q_t(s')}} + 15H^2\\
    &\leq \frac{1}{18HR_t} {\sum_{h = h(s)}^{H-1}\sum_{s' \in \calS_h} q^{\pi_t}(s'\mid s,a)\frac{\ind\set{s_t^\dagger = s'}}{q^{\pi_t}(s,a) q^{\pi_t}(s'\mid s,a) + \gamma_t}} + 15H^2\\
    &\leq \frac{1}{18HR_t} {\sum_{h = h(s)}^{H-1}\sum_{s' \in \calS_h}\frac{\ind\set{s_t^\dagger = s'}}{q^{\pi_t}(s,a) + \gamma_t}} + 15H^2 \\
    &\leq \frac{1}{18HR_t} \frac{1}{q^{\pi_t}(s)\pi_t(a\mid s) + \gamma_t} + 15H^2 \label{eq:bound_B_virtual}\\
    & \leq \frac{\sqrt{HS}}{\gamma_t} + 15H^2 , \\
\end{align}
where in the last inequality we used $R_t > \frac{1}{18\sqrt{H^3S}}$ in a virtual episode.
This is the desired second inequality in \cref{eq:bound_eta_pi_B_known}.

By using \cref{eq:bound_B_virtual}, we also have
\begin{align}
    \eta_t(s,a)\pi_t(a \mid s)B_t(s,a) 
    &\leq \frac{1}{18HR_t} \frac{\eta_t(s,a)\pi_t(a\mid s)}{q^{\pi_t}(s)\pi_t(a\mid s) + \gamma_t} + 15\eta_1 H^2\\ 
    &\leq \frac{1}{18HR_t} \frac{\eta_t(s,a)}{q_t(s)} + 15\eta_1 H^2\\ 
    &\leq \frac{1}{18H} + \frac{1}{12H} \leq \frac{1}{5H},
\end{align}
where we used $\frac{\eta_t(s,a)}{q_t(s)}\leq R_t$ and $\eta_1= \frac{1}{180H^3}$.
This is the desired first inequality in \cref{eq:bound_eta_pi_B_known}.
\end{proof}

\begin{lem}[{\citealt[Lemma G.3]{dann2023best}}]
\label{lem:eta_lemma_normal}
    Let $\eta_1>0, \eta_2, \eta_3, \ldots$ be updated by 
    \begin{align}
        \frac{1}{\eta_{t+1}} = \frac{1}{\eta_t} + \eta_t \phi_t\qquad \forall t\geq 1
    \end{align}
    with $0\leq \phi_t\leq \eta_t^{-2}$. 
    Then,
    \begin{align}
        \frac{1}{\eta_{t+1}} \geq \frac{1}{2}\sqrt{\sum_{\tau=1}^{t+1} \phi_\tau}. 
    \end{align}
\end{lem}

\begin{lem}\label{lem:learning_rate_policy}
Suppose that the learning rates are updated according to \cref{eq:eta_update_policy}. Then, it holds
    \begin{align}
        \eta_t(s,a) \leq \frac{2\sqrt{\ln(T)}}{\sqrt{\sum_{\tau\leq t: \tau\in\mathcal{T}_r} \frac{\zeta_\tau(s,a)}{q_\tau(s)^2} }}
    \end{align}
for any episode $t$ and state-action pair $(s,a)$.

\end{lem}
\begin{proof}
    Let $\phi_t(s,a)=\frac{\zeta_t(s,a)}{q_t(s)^2\ln(T)}$ in real episodes and $\phi_t(s,a)=\frac{\I\{(s_t^\dagger,a_t^\dagger)=(s,a)\}}{324\eta_t(s,a)^2H\ln(T)}$ in virtual episodes. Then the update rule of learning rates can be written as
    \begin{align}
        \frac{1}{\eta_{t+1}(s,a)} = \frac{1}{\eta_t(s,a)} + \eta_t(s,a) \phi_t(s,a).
    \end{align}
    To apply \cref{lem:eta_lemma_normal}, it suffices to show that $\phi_t(s,a)\le \frac{1}{\eta_t(s,a)^2}$. 
    This is clear for virtual episodes. For real episodes, 
    \begin{align}
        \phi_t(s,a)\eta_t(s,a)^2 
        = \frac{\eta_t(s,a)^2\zeta_t(s,a)}{q_t(s)^2\ln(T)}
        \leq \frac{H^2}{\ln(T)}\prn*{\frac{\eta_t(s,a)}{q_t(s)}}^2
        \leq \frac{H^2}{\ln(T)}\cdot \frac{1}{18^2 H^3S} 
        \leq 1,
    \end{align}
    which follows from $\zeta_t(s,a) \leq H^2$ and $\frac{\eta_t(s,a)}{q_t(s)}\leq \frac{1}{18\sqrt{H^3S}}$ in real episodes.

    Then, by \cref{lem:eta_lemma_normal}, we have
    \begin{align}
        \eta_t(s,a) \leq \frac{2}{\sqrt{\sum_{\tau\leq t} \phi_\tau }}
        \leq \frac{2\sqrt{\ln(T)}}{\sqrt{\sum_{\tau\leq t: \tau\in\mathcal{T}_r} \frac{\zeta_\tau(s,a)}{q_\tau(s)^2} }}
        .
    \end{align}
\end{proof}

\begin{lem}
\label{lem:bound_virtual_episode}
The number of virtual episodes $\abs{\mathcal{T}_v}$ is upper bounded by
\begin{align}
        \abs{\mathcal{T}_v} \lesssim  HSA \ln^2(T).
    \end{align}
\end{lem}

\begin{proof}
By the definition of virtual episodes, whenever $t\in\mathcal{T}_v$, there exists a pair $(s,a)$ such that $\frac{\eta_t(s,a)}{q_t(s)} > \frac{1}{18\sqrt{H^3S}}$.
Moreover, in virtual episodes, the corresponding learning rate will shrink by a factor of $\prn*{1+\frac{1}{324H\ln (T)}}$ for a state-action pair $(s_t^\dagger,a_t^\dagger)$.
Hence, for each fixed $(s,a)$, the number of virtual updates on this pair is at most the number of multiplicative shrink steps needed to reduce $\eta_t(s,a)$ from its initial value $\eta_1$ to $\frac{\max_{t \in [T]} q_t(s)}{18\sqrt{H^3S}} \geq \frac{\gamma_T}{18\sqrt{H^3S}}$. 
Hence,
\begin{align}
    |\mathcal{T}_v|
    \lesssim
    SA \cdot
    \frac{\log \prn*{\frac{\eta_1}{\gamma_T / (18\sqrt{H^3 S}) }}}
    {\log (1 + \frac{1}{H \log T}) }
    \lesssim 
    SA
    \cdot \frac{\log \frac{\sqrt{H^3S}T\eta_1}{\sqrt{HS}}}{\log\prn*{1+\frac{1}{H\ln T}}} \lesssim HSA \ln^2(T)
    ,
\end{align}
where we used $\ln(1+x)\ge x/2$ for $x\in(0,1]$ and $\gamma_T=\frac{\sqrt{HS}}{T}$.
\end{proof}

\begin{lem}
\label{lem:zeta_decompose}
It holds that 
\begin{align}
    &\E_t\brk*{\zeta_t(s,a)} \\
    &\leq 2H\Varmax(s) q^{\pi_t}(s, a)(1 - \pi_t(a\mid s)) + 2H\sum_{s',a'}q^{\pi_t}(s',a')\E_t\brk*{(\ell_t(s', a') - \ell'_t(s',a') + \mu(s',a') - m_t(s',a'))^2}
\end{align}
for all state-action pairs $(s,a)$.
\end{lem}
\begin{proof}
Fix any $(s,a)$. Define 
$\kappa_t(s,a) = \ell'_t(s,a) - \mu(s,a), \lambda_t(s,a) = \ell_t(s, a) - \ell'_t(s,a) + \mu(s,a) -  m_t(s,a)$ so that 
\begin{align}
    \ell_t(s,a) - m_t(s,a) = \kappa_t(s,a) + \lambda_t(s,a).
\end{align}
Conditioning on which action is taken at state $s$ in episode $t$, we write
    \begin{align}
        \E_t\brk*{\zeta_t(s,a)}
        &= \E_t\brk*{(\I_t(s,a) - \pi_t(a\mid s) \I_t(s))^2(L_{t,h(s)} - M_{t,h(s)})^2}\\
        &= q^{\pi_t}(s,a)\E_t\brk*{(\I_t(s,a) - \pi_t(a\mid s)\I_t(s))^2(L_{t,h(s)} - M_{t,h(s)})^2 \relmiddle| \I_t(s,a) = 1}\\
        &\qquad + \sum_{b\neq a}q^{\pi_t}(s,b)\E_t\brk*{(\I_t(s,a) - \pi_t(a\mid s)\I_t(s))^2(L_{t,h(s)} - M_{t,h(s)})^2 \relmiddle| \I_t(s,b) = 1}\\
        &= q^{\pi_t}(s,a)(1 - \pi_t(a\mid s))^2\E_t\brk*{(L_{t,h(s)} - M_{t,h(s)})^2 \relmiddle| \I_t(s,a) = 1} \\
        &\qquad + \sum_{b\neq a}q^{\pi_t}(s,b)\pi_t(a\mid s)^2\E_t\brk*{(L_{t,h(s)} - M_{t,h(s)})^2\relmiddle| \I_t(s,b) = 1}. \label{eq:zeta_decomp_policy}
    \end{align}
By the definitions of $L_{t,h(s)}$ and $M_{t,h(s)}$, we have
\begin{align}
    L_{t,h(s)} - M_{t,h(s)}
    &= \sum_{h'=h(s)}^{H-1}\prn*{\ell_t(s_{t,h'}, a_{t,h'}) - m_t(s_{t,h'}, a_{t,h'})}\\
    &= \sum_{h'=h(s)}^{H-1}\prn*{\kappa_t(s_{t,h'}, a_{t,h'}) + \lambda_t(s_{t,h'}, a_{t,h'})}.
\end{align}
Then, for any $b\in\calA$ we obtain
    \begin{align}
         &\E_t\brk*{\prn*{L_{t,h(s)} - M_{t,h(s)}}^2 \relmiddle| \I_t(s,b) = 1}\\ &= \E_t\brk*{\prn*{\sum_{h'=h(s)}^{H-1}\prn*{\kappa_t(s_{t,h'}, a_{t,h'}) + \lambda_t(s_{t,h'}, a_{t,h'})}}^2\relmiddle| \I_t(s,b) = 1}\\
         &\leq 2\E_t\brk*{\prn*{\sum_{h'=h(s)}^{H-1}\kappa_t(s_{t,h'}, a_{t,h'})}^2 \relmiddle| \I_t(s,b) = 1} + 2\E_t\brk*{\prn*{\sum_{h'=h(s)}^{H-1}\lambda_t(s_{t,h'}, a_{t,h'})}^2\relmiddle| \I_t(s,b) = 1}, \label{eq:L-M_decomp}
    \end{align}
where we used $(x+y)^2\le 2(x^2+y^2)$ for $x,y \in\R$.
By the Cauchy--Schwarz inequality, the first term in \cref{eq:L-M_decomp} is evaluated as
\begin{align}
    2\E_t\brk*{\prn*{\sum_{h'=h(s)}^{H-1}\kappa_t(s_{t,h'}, a_{t,h'})}^2 \relmiddle| \I_t(s,b) = 1} 
    &\leq 2H\E_t\brk*{\sum_{h'=h(s)}^{H-1}\kappa_t(s_{t,h'}, a_{t,h'})^2 \relmiddle| \I_t(s,b) = 1}\\
    &= 2H\sum_{s',a'}q^{\pi_t}(s',a'\mid s, b) \sigma^2(s', a')\\
    &\leq 2H\Varmax(s). \label{eq:zeta_correlate}
\end{align}
For the second term in \cref{eq:L-M_decomp}, the same argument yields
    \begin{align}
        2\E_t\brk*{\prn*{\sum_{h'=h(s)}^{H-1}\lambda_t(s_{t,h'}, a_{t,h'})}^2\relmiddle| \I_t(s,b) = 1} 
        &\leq 2H\E_t\brk*{\sum_{h'=h(s)}^{H-1}\lambda_t(s_{t,h'}, a_{t,h'})^2\relmiddle| \I_t(s,b) = 1}\\
        &\leq 2H\sum_{s',a'}q^{\pi_t}(s',a'\mid s, b)\E_t\brk*{\lambda_t(s', a')^2}.
    \end{align}
Combining the above two bounds, we obtain
    \begin{align}
         \E_t\brk*{\prn*{L_{t,h(s)} - M_{t,h(s)}}^2 \relmiddle| \I_t(s,b) = 1}
         \leq 2H\Varmax(s) + 2H\sum_{s',a'}q^{\pi_t}(s',a'\mid s, b)\E_t\brk*{\lambda_t(s', a')^2}. \label{eq:L-M_cond_bound}
    \end{align}
Thus, combining \cref{eq:L-M_cond_bound} with \cref{eq:zeta_decomp_policy} yields
        \begin{align}
        \E_t\brk*{\zeta_t(s,a)} 
        &\leq q^{\pi_t}(s,a)(1 - \pi_t(a\mid s))^2\prn*{2H\Varmax(s) + 2H\sum_{s',a'}q^{\pi_t}(s',a'\mid s, a)\E_t\brk*{\lambda_t(s', a')^2}} \\
        &\qquad + \sum_{b\neq a}q^{\pi_t}(s,b)\pi_t(a\mid s)^2\prn*{2H\Varmax(s) + 2H\sum_{s',a'}q^{\pi_t}(s',a'\mid s, b)\E_t\brk*{\lambda_t(s', a')^2}}\\
        &= 2H\Varmax(s) q^{\pi_t}(s,a)(1 - \pi_t(a\mid s))^2 + 2H(1 - \pi_t(a\mid s))^2\sum_{s',a'}q^{\pi_t}(s,a)q^{\pi_t}(s',a'\mid s, a)\E_t\brk*{\lambda_t(s', a')^2} \\
        &\qquad + 2H\Varmax(s)\sum_{b\neq a}q^{\pi_t}(s,b)\pi_t(a\mid s)^2+ 2H\pi_t(a\mid s)^2\sum_{s',a'}\sum_{b\neq a}q^{\pi_t}(s,b)q^{\pi_t}(s',a'\mid s, b)\E_t\brk*{\lambda_t(s', a')^2}\\
        &\leq 2H\Varmax(s) q^{\pi_t}(s)\pi_t(a\mid s)(1 - \pi_t(a\mid s)) + 2H\sum_{s',a'}q^{\pi_t}(s',a')\E_t\brk*{\lambda_t(s', a')^2},
    \end{align}
    which completes the proof.
\end{proof}

\begin{cor}\label{cor:zeta_decompose2} 
    Under the stochastic regime with adversarial corruption, suppose that the uncorrupted losses are generated independently and are uncorrelated across layers. Then, it holds that
    \begin{align}
        &\E_t\brk*{\zeta_t(s,a)} \\
        &\leq 2\Varmax(s) q^{\pi_t}(s, a)(1 - \pi_t(a\mid s)) + 2H\sum_{s',a'}q^{\pi_t}(s',a')\E_t\brk*{(\ell_t(s', a') - \ell'_t(s',a') + \mu(s',a') - m_t(s',a'))^2}
\end{align}
	for all state-action pairs $(s,a)$.
\end{cor}
\begin{proof}
    This corollary can be viewed as a simple variant of \cref{lem:zeta_decompose}.
    Let $\kappa_t(s,a) \coloneqq \ell'_t(s,a)-\mu(s,a)$.
    Since the uncorrupted losses are generated independently and are uncorrelated
    across layers, it holds that for any $(s_1,a_1)\neq(s_2,a_2)$,
    \begin{align}
        \E_t\brk*{\I_t(s_1,a_1)\I_t(s_2,a_2)\kappa_t(s_1,a_1)\kappa_t(s_2,a_2)} = 0.
    \end{align}
    Then, for any function $\alpha : \calS\times \calA\to\R$, we have
    \begin{align}
    \prn*{\sum_{s,a}\alpha(s,a)\I_t(s,a)\kappa_t(s,a)}^2
    &= \sum_{s_1,a_1}\sum_{s_2,a_2} \alpha(s_1,a_1)\alpha(s_2,a_2)\I_t(s_1,a_1)\I_t(s_2,a_2)\kappa_t(s_1,a_1)\kappa_t(s_2,a_2)\\
    &= \sum_{s,a} \alpha(s,a)^2\I_t(s,a)\kappa_t(s,a)^2. \label{eq:uncorrupt_kappa_policy}
    \end{align}
    Thus, for the first term in \cref{eq:L-M_decomp} is evaluated as
    \begin{align}
        2\E_t\brk*{\prn*{\sum_{h'=h(s)}^{H-1}\kappa_t(s_{t,h'}, a_{t,h'})}^2 \relmiddle| \I_t(s,b) = 1} 
        &\leq 2\E_t\brk*{\sum_{h'=h(s)}^{H-1}\kappa_t(s_{t,h'}, a_{t,h'})^2 \relmiddle| \I_t(s,b) = 1}\\
        &= 2\sum_{s',a'}q^{\pi_t}(s',a'\mid s, b) \sigma^2(s', a')\\
        &\leq 2\Varmax(s). \label{eq:zeta_uncorrelate}
    \end{align}
Therefore, compared with \cref{lem:zeta_decompose}, we obtain an
$H$-times sharper bound in \cref{eq:zeta_uncorrelate}
than \cref{eq:zeta_correlate}. As a consequence, the corresponding $\Varmax$ term is also improved by a factor of $H$.
\end{proof}

\begin{lem}
\label{lem:L-M_to_l-m}
For each state-action pair $(s,a)$, it holds that
\begin{align}
    \E\brk*{\sum_{t=1}^T\sum_{s,a}\I_t(s,a)(L_{t,h(s)} - M_{t,h(s)})^2} \leq H^2\E\brk*{\sum_{t=1}^T\sum_{s,a}\I_t(s,a)(\ell_t(s,a) - m_t(s,a))^2}
    .
\end{align}
\end{lem}
\begin{proof}
By the definitions of $L_{t,h(s)}$ and $M_{t,h(s)}$, we have
\begin{align}
    L_{t,h(s)} - M_{t,h(s)}
    = \sum_{h'=h(s)}^{H-1}\prn*{\ell_t(s_{t,h'}, a_{t,h'}) - m_t(s_{t,h'}, a_{t,h'})}.
\end{align}
Hence, we have
\begin{align}
    &\E\brk*{\sum_{t=1}^T\sum_{s,a}\I_t(s,a)\prn*{L_{t,h(s)} - M_{t,h(s)}}^2} \\
    &= \E\brk*{\sum_{t=1}^T\sum_{s,a}\I_t(s,a)\prn*{\sum_{h'=h(s)}^{H-1}\prn*{\ell_t(s_{t,h'}, a_{t,h'}) - m_t(s_{t,h'}, a_{t,h'})}}^2} \\
    &\leq H\E\brk*{\sum_{t=1}^T\sum_{s,a}\I_t(s,a)\sum_{h'=h(s)}^{H-1}\prn*{\ell_t(s_{t,h'}, a_{t,h'}) - m_t(s_{t,h'}, a_{t,h'})}^2} \tag{by the Cauchy–Schwarz inequality}\\
    &\leq H\E\brk*{\sum_{t=1}^T\sum_{h=0}^{H-1}\sum_{(s,a) \in \calS_h \times \calA}\I_t(s,a)\sum_{h'=0}^{H-1}\prn*{\ell_t(s_{t,h'}, a_{t,h'}) - m_t(s_{t,h'}, a_{t,h'})}^2}\\
    &= H\E\brk*{\sum_{t=1}^T\sum_{h=0}^{H-1}\sum_{h'=0}^{H-1}\prn*{\ell_t(s_{t,h'}, a_{t,h'}) - m_t(s_{t,h'}, a_{t,h'})}^2\sum_{(s,a)\in \calS_h \times \mathcal {A}}\I_t(s,a)},
\end{align}
where the last equality rearranges the summations.

Since for each fixed $(t,h)$ exactly one state-action pair is visited, we have
$\sum_{(s,a)\in \calS_h \times \calA}\I_t(s,a) = 1$, and then, 
\begin{align}
    & H\E\brk*{\sum_{t=1}^T\sum_{h=0}^{H-1}\sum_{h'=0}^{H-1}\prn*{\ell_t(s_{t,h'}, a_{t,h'}) - m_t(s_{t,h'}, a_{t,h'})}^2\sum_{(s,a)\in \calS_h \times \mathcal {A}}\I_t(s,a)}\\
    &= H\E\brk*{\sum_{t=1}^T\sum_{h=0}^{H-1}\sum_{h'=0}^{H-1}\prn*{\ell_t(s_{t,h'}, a_{t,h'}) - m_t(s_{t,h'}, a_{t,h'})}^2} \tag{$\sum_{(s,a)\in \calS_h \times \mathcal {A}}\I_t(s,a) = 1$}\\
    &= H^2\E\brk*{\sum_{t=1}^T\sum_{h'=0}^{H-1}\prn*{\ell_t(s_{t,h'}, a_{t,h'}) - m_t(s_{t,h'}, a_{t,h'})}^2}\\
    &= H^2\E\brk*{\sum_{t=1}^T\sum_{s,a}\I_t(s,a)(\ell_t(s,a) - m_t(s,a))^2},
\end{align}
which completes the proof.
\end{proof}

\subsection{Common Regret Analysis}

Now we are ready to upper bound the RHS of \cref{eq:decompose_PO}.
We first consider the bias term, $\biasterm(s)$.

\begin{lem}\label{lem:biasterm_policy}
    For each state $s \in \calS$, it holds that 
    \begin{align}
        \mathbb{E} \brk*{ \biasterm(s) } \leq 2\sum_{t=1}^T \frac{\gamma_t H}{q_t(s)} + H^2SA\ln^2(T).
    \end{align}
\end{lem}

\begin{proof}
From the definition of the $Q$-function estimator, we have
\begin{align}
   &\E_t\brk*{Q^{\pi_t}(s,a;\ell_t) -  \hat{Q}_t(s,a)}\\
   &= \E_t\brk*{Q^{\pi_t}(s,a;\ell_t) - \prn*{Q^{\pi_t}(s,a; m_t) + \frac{ q^{\pi_t}(s)}{q_t(s)}(Q^{\pi_t}(s,a;\ell_t) - Q^{\pi_t}(s,a; m_t))Y_t - \frac{\gamma_t}{q_t(s)}H}} 
   \tag{by \cref{lem:expect_hatQ_proof}} \\
   &= \E_t\brk*{\frac{ q_t(s) -  q^{\pi_t}(s)Y_t}{q_t(s)}\prn*{Q^{\pi_t}(s,a;\ell_t) - Q^{\pi_t}(s,a; m_t)} + \frac{\gamma_t}{q_t(s)}H}\\
   &= \begin{cases}
        - Q^{\pi_t}(s,a; m_t) + \dfrac{\gamma_t}{q_t(s)}H& \text {if } Y_t = 0\\[2ex]
       \dfrac{\gamma_t}{q_t(s)}\prn*{Q^{\pi_t}(s,a;\ell_t) - Q^{\pi_t}(s,a; m_t) + H}& \text {if } Y_t = 1
   \end{cases}.
\end{align}
When $Y_t = 1$, since $\ell_t,m_t\in[0,1]^{S \times A}$, we have
\begin{align}
    0 \leq Q^{\pi_t}(s,a;\ell_t - m_t) + H \leq 2H.
\end{align}
Therefore,
\begin{align}
    0 \leq \E_t\brk*{Q^{\pi_t}(s,a;\ell_t) -  \hat{Q}_t(s,a)} \leq \frac{2\gamma_t H}{q_t(s)}.
\end{align}
When $Y_t = 0$, we use $Q^{\pi_t}(s,a; m_t) \leq H$ and obtain
\begin{align}
    -H \leq \E_t\brk*{Q^{\pi_t}(s,a;\ell_t) -  \hat{Q}_t(s,a)} \leq \frac{\gamma_t H}{q_t(s)}.
\end{align}
Using this bound, we obtain
\begin{align}
    \mathbb{E} \brk*{\biasterm(s) }
    &= \mathbb{E}\brk*{\sum_{t=1}^T \sum_{a} \prn*{\pi_t(a\mid s) -  \pio(a\mid s)} \prn*{Q^{\pi_t}(s,a;\ell_t) -  \hat{Q}_t(s,a)}} \\
    &\leq \mathbb{E}\brk*{2\sum_{t=1}^T \sum_{a} \pi_t(a\mid s) \frac{\gamma_t H}{q_t(s)} } + \mathbb{E}\brk*{\sum_{t\in \calT_v} \sum_{a} \pio(a\mid s) H }\\
    &= 2\sum_{t=1}^T \frac{\gamma_t H}{q_t(s)} + H\abs*{\calT_v} \leq 2\sum_{t=1}^T \frac{\gamma_t H}{q_t(s)} + H^2SA\ln^2(T),
\end{align}
where the last inequality follows from \cref{lem:bound_virtual_episode}, which guarantees that the number of virtual episodes satisfies $\abs{\calT_v}\le HSA\ln^2(T)$.
\end{proof}

We next consider $\regterm(s)$.
\begin{lem}\label{lem:regterm_policy}
    For each state $s \in \calS$, it holds that 
    \begin{align}
        \mathbb{E} \brk*{\regterm(s) } &\leq O\prn*{H^3A\ln(T)} + 6\,\mathbb{E}\brk*{\sum_{t=1}^T\sum_a  \prn*{\frac{1}{\eta_{t+1}(s, a)} - \frac{1}{\eta_t(s, a)}}\ln(T)} \\
        &\qquad + \mathbb{E}\brk*{\frac{1}{H}\sum_{t=1}^T\sum_a \pi_t(a \mid s) B_t(s,a)} + 3\sum_{t=1}^T \frac{\gamma_t H}{q_t(s)}
        .
    \end{align}
\end{lem}

\begin{proof}
We will apply \cref{lem:OFTRL_log_barrier_policy} with $p_t = \pi_t(\cdot \mid s)$ and $\ell_t = \hat{Q}_t(s,a) - B_t(s,a)$ for each $s \in \calS$.
To do so, in what follows, we will check the conditions of \Cref{lem:OFTRL_log_barrier_policy}.
Let
\begin{align}
    \tilQ(s,a) = Q^{\pi_t}(s,a;m_t) + \frac{\I_t(s,a)(L_{t,h(s)} - M_{t,h(s)})}{q_t(s)\pi_t(a\mid s)}Y_t.
\end{align}
Then, we have 
$\hat{Q}_t(s,a) = \tilQ(s,a) - \frac{\gamma_t H}{q_t(s)}$.
Define
\begin{align}
x_t = \inpr*{-\pi_t(\cdot \mid s) , \tilQ_t(s,\cdot) - Q^{\pi_t}(s, \cdot;m_t)} = -\dfrac{\I_t(s)(L_{t,h(s)} - M_{t,h(s)})}{q_t(s)}Y_t
\end{align}
and verify that for all $(s,a)$, $\eta_t(s,a)\pi_t(a\mid s)\prn*{\hat{Q}_t(s,a)-B_t(s,a) - Q^{\pi_t}(s, a; m_t) + x_t}\geq -{1}/{2}$.
Recall that in a virtual episode we have $Y_t=0$ and $\ell_t(s,a)=0$ for all state-action pairs $(s,a)$. Hence,
\begin{align}
    &\eta_t(s,a)\pi_t(a\mid s)\prn*{\hat{Q}_t(s,a)-B_t(s,a)-Q^{\pi_t}(s, a; m_t) + x_t} \\ 
    &= \eta_t(s,a)\pi_t(a\mid s) \prn*{\frac{\I_t(s,a)(L_{t,h(s)} - M_{t,h(s)})}{q_t(s)\pi_t(a\mid s)}Y_t-B_t(s,a) - \frac{\gamma_t H}{q_t(s)} -\dfrac{\I_t(s)(L_{t,h(s)} - M_{t,h(s)})}{q_t(s)}Y_t} \\
    &\geq -\frac{\eta_t(s,a)}{q_t(s)}M_{t,h(s)}Y_t -\eta_t(s, a)\pi_t(a\mid s) B_t(s,a) - \eta_t(s,a)H -\frac{\eta_t(s,a)}{q_t(s)}L_{t,h(s)}Y_t\\
    &\geq -\frac{2\eta_t(s,a)}{q_t(s)}HY_t-\eta_t(s,a)\pi_t(a\mid s) B_t(s,a) - \eta_1 H \\
    &\geq -\frac{1}{9\sqrt{HS}} - \frac{1}{5H} - \frac{1}{180H^2} \\
    &\geq -\frac{1}{2}, 
\end{align}
where the bounds in the third and fourth lines use $0\leq L_{t,h(s)}\leq H$ and $0\leq M_{t,h(s)}\leq H$, and the fifth line uses
$\frac{\eta_t(s,a)}{q_t(s)} \leq \frac{1}{18\sqrt{H^3S}}$ in real episodes together with
$\eta_t(s,a)\pi_t(a\mid s)B_t(s,a)\leq \frac{1}{5H}$ from \cref{lem:bound_eta_pi_B_known}.

Hence, by \cref{lem:OFTRL_log_barrier_policy}, we obtain
\begin{align}
        & \E[\regterm(s)]\\
        &\leq \frac{A\ln(AT^2)}{\eta_1} + \E\brk*{\sum_{t=1}^T \sum_a \prn*{\frac{1}{\eta_{t+1}(s, a)} - \frac{1}{\eta_t(s, a)}}\ln(AT^2)} \\
        &\qquad + \E\brk*{\sum_{t=1}^T \sum_a  \eta_t(s, a)\pi_t(a \mid s)^{2}\prn*{\prn*{\tilQ_t(s, a) - Q^{\pi_t}(s, a; m_t)} - B_t(s, a) - \frac{\gamma_t H}{q_t(s)} + x_t}^2} \\
        &\qquad + \E\brk*{\frac{1}{T^2}\sum_{t=1}^T \inpr*{ -\pio(\cdot \mid s) + \frac{1}{A}\one, \hatQ_t(s, \cdot) - B_t(s, \cdot) }} + 2\E\brk*{\nrm*{Q^{\pi_t}(s, \cdot; m_{T+1})}_{\infty}} \\
        &\leq \frac{3A\ln(T)}{\eta_1} + 4 + \frac{30H^2}{T} + 2H + 3\E\brk*{\sum_{t=1}^T\sum_a \prn*{\frac{1}{\eta_{t+1}(s, a)} - \frac{1}{\eta_t(s, a)}}\ln(T)} \tag{by $T \geq A$ and $\nrm*{Q^{\pi_t}(s, \cdot; m_{T+1})}_{\infty} \leq H$} \\
        &\qquad + 3\E\brk*{\sum_{t=1}^T \sum_a  \eta_t(s, a)\pi_t(a\mid s)^{2}\prn*{\tilQ_t(s, a)  - Q^{\pi_t}(s, a; m_t) +x_t}^2}\\
        &\qquad + 3\E\brk*{\sum_{t=1}^T \sum_a  \eta_t(s, a)\pi_t(a\mid s)^{2} \prn*{B_t(s, a)^2 + \frac{\gamma_t^2 H^2}{q_t(s)^2}}} \\ 
        &\leq O(H^3A\ln(T)) + 3\E\brk*{\sum_{t=1}^T\sum_a \prn*{\frac{1}{\eta_{t+1}(s, a)} - \frac{1}{\eta_t(s, a)}}\ln(T)}  \\
        &\qquad + 3\E\brk[\Bigg]{\underbrace{\sum_{t=1}^T \sum_a  \eta_t(s, a)\pi_t(a\mid s)^{2}\prn*{\tilQ_t(s, a)  - Q^{\pi_t}(s, a; m_t) +x_t}^2}_{\text{stability-term-2}}}\\
        &\qquad + \mathbb{E}\brk*{\frac{1}{H}\sum_{t=1}^T\sum_a \pi_t(a \mid s) B_t(s,a)} + 3\mathbb{E}\brk*{\sum_{t=1}^T \sum_a \eta_t(s,a)\pi_t(a\mid s) \frac{\gamma_t H^2}{q_t(s)}}.
        \label{eq:regterm_policy_mid}
    \end{align}
    Here, the second inequality follows from 
    \begin{align}
        &\E\brk*{\frac{1}{T^2}\sum_{t=1}^T \inpr*{ -\pio(\cdot \mid s) + \frac{1}{A}\one, \hatQ_t(s, \cdot) - B_t(s, \cdot) }} \\
        &= \frac{1}{T^2}\inpr*{ -\pio(\cdot \mid s) + \frac{1}{A}\one, \E\brk*{\sum_{t=1}^T\E_t\brk*{\hatQ_t(s, \cdot) - B_t(s, \cdot)} }} \\
        &\leq \frac{1}{T^2}\nrm*{-\pio(\cdot \mid s) + \frac{1}{A}\one}_1\nrm*{\E\brk*{\sum_{t=1}^T\E_t\brk*{\hatQ_t(s, \cdot) - B_t(s, \cdot)} }}_\infty \\
        &\leq \frac{2T(2T + 15H^2)}{T^2} = 4 + \frac{30H^2}{T}
        ,
    \end{align}
    where we used $\nrm*{-\pio(\cdot \mid s) + \frac{1}{A}\one}_1\leq 2$,
    $\abs*{\E_t[\hatQ_t(s,a)]} \leq H$ and $B_t(s,a) \leq \frac{2\sqrt{HS}}{\gamma_t} + 15H^2$ from \cref{lem:bound_eta_pi_B_known}, which together imply $\nrm*{\E\brk*{\sum_{t=1}^T\E_t\brk*{\hatQ_t(s, \cdot) - B_t(s, \cdot)} }}_\infty \leq T(2T + 15H^2)$,
    and the last inequality follows from $\gamma_t \leq q_t(s)$ and \cref{lem:bound_eta_pi_B_known}.
    We further evaluate the \text{stability-term-2} in the last inequality as
    \begin{align}
        &\eta_t(s, a)\pi_t(a\mid s)^{2}\prn*{\tilQ_t(s, a)  - Q^{\pi_t}(s, a; m_t) + x_t}^2   \\
        &= \eta_t(s,a)\pi_t(a\mid s)^2\prn*{\frac{\I_t(s,a) (L_{t,h(s)} - M_{t,h(s)})}{q_t(s) \pi_t(a\mid s)} - \frac{\I_t(s) (L_{t,h(s)} - M_{t,h(s)})}{q_t(s)}}^2Y_t  \\
        &= \eta_t(s,a)\prn*{\frac{\I_t(s,a) (L_{t,h(s)} - M_{t,h(s)})}{q_t(s)} - \frac{\pi_t(a\mid s) \I_t(s) (L_{t,h(s)} - M_{t,h(s)})}{q_t(s)}}^2Y_t  \\
        &= \frac{\eta_t(s,a)}{q_t(s)^2}(\I_t(s,a) - \pi_t(a\mid s) \I_t(s))^2 (L_{t,h(s)} - M_{t,h(s)})^2Y_t\\
        &= \frac{\eta_t(s,a)\zeta_t(s,a)}{q_t(s)^2} Y_t
        ,
    \end{align}
    where the last equality follows from the definition of $\zeta_t(s,a) = (\I_t(s,a) - \pi_t(a\mid s) \I_t(s))^2 (L_{t,h(s)} - M_{t,h(s)})^2$.
    Then, \text{stability-term-2} is evaluated as
    \begin{align}
        \text{stability-term-2} &= \sum_{t=1}^T \sum_a  \eta_t(s, a)\pi_t(a\mid s)^{2}\prn*{\tilQ_t(s, a)  - Q^{\pi_t}(s, a; m_t) + x_t}^2 \\
        &= \sum_{t=1}^T \sum_a \frac{\eta_t(s,a)\zeta_t(s,a)}{q_t(s)^2}Y_t \\
        &\leq \sum_{t=1}^T \sum_a\prn*{\frac{1}{\eta_{t+1}(s,a)} - \frac{1}{\eta_t(s,a)}}\ln(T),
    \end{align}
    where the last inequality follows from \cref{eq:eta_update_policy}.
    Then, together with \cref{eq:regterm_policy_mid} and $\eta_t(s,a) \leq \frac{1}{H}$, we obtain
    \begin{align}
        \E[\regterm(s)] &\leq O\prn*{H^3A\ln(T)} + 6\mathbb{E}\brk*{\sum_{t=1}^T\sum_a  \prn*{\frac{1}{\eta_{t+1}(s, a)} - \frac{1}{\eta_t(s, a)}}\ln(T)} \\
        &\qquad + \mathbb{E}\brk*{\frac{1}{H}\sum_{t=1}^T\sum_a \pi_t(a \mid s) B_t(s,a)} + 3\E\brk*{\sum_{t=1}^T \sum_a \pi_t(a\mid s) \frac{\gamma_t H}{q_t(s)}}\\
        &= O\prn*{H^3A\ln(T)} + 6\mathbb{E}\brk*{\sum_{t=1}^T\sum_a  \prn*{\frac{1}{\eta_{t+1}(s, a)} - \frac{1}{\eta_t(s, a)}}\ln(T)} \\
        &\qquad + \mathbb{E}\brk*{\frac{1}{H}\sum_{t=1}^T\sum_a \pi_t(a \mid s) B_t(s,a)} + 3\sum_{t=1}^T \frac{\gamma_t H}{q_t(s)}.
    \end{align}
\end{proof}

\begin{lem}\label{lem:regret_to_bt}
\cref{alg:OFTRL_PO} guarantees
    \begin{align}
        \Reg_T \leq O\prn*{H^3S A\ln^2(T)} + 3\E\brk*{\sum_{t = 1}^T V^{\pi_t}(s_0;b_t)},
    \end{align}
    where $b_t$ is defined in \cref{def:bonus_term}.
\end{lem}
\begin{proof}
By the definition of the regret decomposition in \cref{eq:decompose_PO},
\begin{align}
    &\E\brk*{\sum_s \qst(s)\sum_{t,a} \prn*{\pi_t(a\mid s) -  \pio(a\mid s) } \prn*{Q^{\pi_t}(s,a;\ell_t) -  B_t(s,a)}}\\
    &= \E\brk*{\sum_s \qst(s) \cdot \regterm(s)} + \E\brk*{\sum_s \qst(s) \cdot \biasterm(s)}\\
    &\leq O\prn*{H^4A\ln(T)} + 6\mathbb{E}\brk*{\sum_{t=1}^T\sum_s \qst(s)\sum_a  \prn*{\frac{1}{\eta_{t+1}(s, a)} - \frac{1}{\eta_t(s, a)}}\ln(T)} \\
    &\qquad + \mathbb{E}\brk*{\frac{1}{H}\sum_{t=1}^T\sum_s \qst(s)\sum_a \pi_t(a \mid s) B_t(s,a)} + 5\sum_{t=1}^T \sum_s \qst(s) \frac{\gamma_t H}{q_t(s)} + H^3SA\ln^2(T)\\
    &= O\prn*{H^3S A\ln^2(T)} + \E\brk*{\sum_{t = 1}^T\sum_s \qst(s) b_t(s)} + \E\brk*{\frac{1}{H} \sum_{t = 1}^T\sum_s \qst(s)\sum_a \pi_t(a\mid s)B_t(s,a)},
\end{align}
where we used \cref{lem:biasterm_policy,lem:regterm_policy} and the definition of $b_t$.
Combining the last inequality with \cref{lem:dilated_bonus} completes the proof.
\end{proof}

\begin{lem}\label{lem:log_barrier_bt_known}
    It holds that
    \begin{align}
        &\E\brk*{\sum_{t=1}^T V^{\pi_t}(s_0;b_t)}\lesssim \ln(T)\sum_{s,a}\sqrt{\E\brk*{\sum_{t\in\mathcal{T}_r} \zeta_t(s,a)}} + H^{\frac{3}{2}}S^{\frac{3}{2}}A\ln^2(T).  
    \end{align}
\end{lem}
\begin{proof}[Proof of \cref{lem:log_barrier_bt_known}]
    We use $\mathcal{T}_r$ and $\mathcal{T}_v$ to denote the set of real and virtual episodes, respectively. Then we have
    \allowdisplaybreaks
    \begin{align}
        &\sum_{t=1}^T  V^{\pi_t}(s_0;b_t)\\
        &= \sum_{t=1}^T\sum_s q^{\pi_t}(s)b_t(s)\\
        &= 6\sum_{t\in\mathcal{T}_r}\sum_{s,a} q^{\pi_t}(s) \prn*{\frac{1}{\eta_{t+1}(s,a)}-\frac{1}{\eta_t(s,a)}}\ln(T) 
        + 6\sum_{t\in\mathcal{T}_v}\sum_{s,a} q^{\pi_t}(s) \prn*{\frac{1}{\eta_{t+1}(s,a)}-\frac{1}{\eta_t(s,a)}}\ln(T) \\
        &\qquad + 5\sum_{t=1}^T\sum_{s} q^{\pi_t}(s)\frac{\gamma_t H}{q_t(s)}\\
        &\lesssim \sum_{t\in\mathcal{T}_r}\sum_{s,a} q^{\pi_t}(s) \frac{\eta_t(s,a)\zeta_t(s,a)}{q_t(s)^2} 
        + \sum_{t\in\mathcal{T}_v}\sum_{s,a} q^{\pi_t}(s) \frac{\I\{(s^\dagger_t, a^\dagger_t) = (s,a)\}}{\eta_t(s,a)H} + \sum_{t=1}^T\sum_{s} \gamma_t H\\
        &\leq \sum_{t\in\mathcal{T}_r}\sum_{s,a} \frac{\eta_t(s,a)\zeta_t(s,a)}{q_t(s)} + \sum_{t\in\mathcal{T}_v} \frac{q^{\pi_t} (s_t^\dagger)}{\eta_t(s_t^\dagger,a_t^\dagger)H} + HS\sum_{t=1}^T \gamma_t\\
        &\lesssim \sum_{t\in\mathcal{T}_r}\sum_{s,a}\frac{\eta_t(s,a)\zeta_t(s,a)}{q_t(s)} + \sum_{t\in\mathcal{T}_v} \sqrt{HS}+ H^{\frac{3}{2}}S^{\frac{3}{2}}\ln(T)
        \tag{by $\frac{\eta_t(s_t^\dagger, a_t^\dagger)}{q_t(s_t^\dagger)} > \frac{1}{18\sqrt{H^3S}}$ in virtual episodes} \\
        &\lesssim \sqrt{\ln(T)}\sum_{t\in\mathcal{T}_r}\sum_{s,a} \frac{\frac{\sqrt{\zeta_t(s,a)}}{q_t(s)}\times \sqrt{\zeta_t(s,a)}}{\sqrt{\sum_{\tau\leq t: \tau\in\mathcal{T}_r} \frac{\zeta_\tau(s,a)}{q_\tau(s)^2} }} + \sqrt{HS}\abs{\mathcal{T}_v} + H^{\frac{3}{2}}S^{\frac{3}{2}}\ln(T)  \tag{by \cref{lem:learning_rate_policy}}\\
        &\leq \sqrt{\ln(T)}\sum_{s,a}\sqrt{\sum_{t\in\mathcal{T}_r}\frac{\frac{\zeta_t(s,a)}{q_t(s)^2}}{ \sum_{\tau\leq t: \tau\in\mathcal{T}_r} \frac{\zeta_\tau(s,a)}{q_\tau(s)^2}}}\sqrt{\sum_{t\in\mathcal{T}_r} \zeta_t(s,a)} + H^{\frac{3}{2}}S^{\frac{3}{2}}A\ln^2(T) + H^{\frac{3}{2}}S^{\frac{3}{2}}\ln(T) \tag{by the Cauchy–Schwarz inequality and \cref{lem:bound_virtual_episode}}\\
        &\lesssim \ln(T)\sum_{s,a}\sqrt{\sum_{t\in\calT_{r}} \zeta_t(s,a)} + H^{\frac{3}{2}}S^{\frac{3}{2}}A\ln^2(T), 
    \end{align}
    where the last inequality follows from
    \begin{align}
        \sqrt{\sum_{t\in\mathcal{T}_r}\frac{\frac{\zeta_t(s,a)}{q_t(s)^2}}{ \sum_{\tau\leq t:\tau\in\mathcal{T}_r} \frac{\zeta_\tau(s,a)}{q_\tau(s)^2}}} 
        \leq \sqrt{1 + \ln \prn*{ \sum_{\tau\in\mathcal{T}_r} \frac{\zeta_\tau(s,a)}{q_\tau(s)^2}}}
        \leq \sqrt{1 + \ln \prn*{\sum_{\tau\in\mathcal{T}_r} \frac{H^2 T}{H}}}
        \lesssim \sqrt{\ln(T)}.
    \end{align}
    Combining the above arguments with $H\leq S$, \cref{lem:regret_to_bt,lem:log_barrier_bt_known}, we obtain
    \begin{align}
        \Reg_T \lesssim \sum_{s,a}\sqrt{\ln^2(T)\mathbb{E}\brk*{\sum_{t=1}^T \zeta_t(s,a)}} + H^{\frac52}S^{\frac32}A\ln^2(T). \label{eq:policy_reg_general}
    \end{align}
\end{proof}

\subsection{Proof of \protect\cref{thm:Policy_OPT}}

Here we provide the proof of \cref{thm:Policy_OPT}.

\begin{thm}[Restatement of \cref{thm:Policy_OPT}]
\cref{alg:OFTRL_PO} with the loss prediction $m_t$ defined in \cref{def:predictor_sequence} guarantees
\begin{align}
    \Reg_T &\lesssim \sqrt{H^2SA\ln^2(T)\min\set*{L^\star, HT - L^\star, Q_{\infty}, V_1}} + H^{\frac52}S^{\frac32}A\ln^2(T).
\end{align}
Under the stochastic regime with adversarial corruption, it simultaneously ensures
\begin{align}
    \Reg_T \lesssim \sqrt{H^2SA\ln^2(T)\prn*{\V T + \calC}} + H^{\frac52}S^{\frac32}A\ln^2(T),
\end{align}
and
\begin{align}
    \Reg_T \lesssim U + \sqrt{U\calC} + H^{\frac52}S^{\frac32}A\ln^2(T),
\end{align}
where $U = \sum_{s}\sum_{a\neq\pist(s)}\frac{H^2\ln^2(T)}{\Delta(s,a)}$.
\end{thm}

\begin{proof}
We start from \cref{eq:policy_reg_general},
\begin{align}
    \Reg_T
    \lesssim \sum_{s,a}\sqrt{\ln^2(T)\mathbb{E}\brk*{\sum_{t=1}^T \zeta_t(s,a)}} + H^{\frac52}S^{\frac32}A\ln^2(T). \label{eq:reg_bound_policy_in_proof}
\end{align}
By the definition of $\zeta_t(s,a)$, we have
\begin{align}   
        &\sum_{s,a}\sqrt{\ln^2(T)\E\brk*{\sum_{t=1}^T \zeta_t(s,a)} } \\
        &= \sum_{s,a}\sqrt{\ln^2(T)\E\brk*{\sum_{t=1}^T (\I_t(s,a) - \pi_t(a\mid s) \I_t(s))^2(L_{t,h(s)} - M_{t,h(s)})^2}}\\
        &\leq \sqrt{SA\ln^2(T)\E\brk*{\sum_{t=1}^T\sum_{s,a} (\I_t(s,a) - \pi_t(a\mid s) \I_t(s))^2(L_{t,h(s)} - M_{t,h(s)})^2}}\tag{by the Cauchy–Schwarz inequality} \\
        &\leq \sqrt{SA\ln^2(T)\E\brk*{\sum_{t \in \calT_r}\sum_{s,a} 2\I_t(s,a)(L_{t,h(s)} - M_{t,h(s)})^2}} \\
        &\lesssim \sqrt{H^2SA\ln^2(T)\E\brk*{\sum_{t \in \calT_r}\sum_{s,a}\I_t(s,a)(\ell_t(s,a) - m_t(s,a))^2}}, \label{eq:common_bound_policy}
\end{align}
where the third line uses 
$\sum_a(\I_t(s,a)-\pi_t(a\mid s)\I_t(s))^2\le 2\I_t(s)$
for each fixed state $s$, and the last inequality follows from \cref{lem:L-M_to_l-m}.

\paragraph{1. Bounds for the adversarial regime.}
Applying \cref{lem:general_predict_result} to \eqref{eq:common_bound_policy} yields
\begin{align}
    &\sqrt{H^2SA\ln^2(T)\E\brk*{\sum_{t = 1}^{T}\sum_{s,a}\I_t(s,a)(\ell_t(s,a)-m_t(s,a))^2}} \\
    &\lesssim \sqrt{H^2SA\ln^2(T)\prn*{\min\set*{L^\star  + \Reg_T, HT - L^\star  - \Reg_T, Q_{\infty}, V_1} +  SA}}\\
    &\leq \sqrt{H^2SA\ln^2(T)\min\set*{L^\star  + \Reg_T, HT - L^\star  - \Reg_T, Q_{\infty}, V_1}} + HSA\ln(T).
\end{align}
Then, we obtain
\begin{align}
    \Reg_T &\lesssim \sqrt{H^2SA\ln^2(T)\prn*{L^\star  + \Reg_T}} + H^{\frac52}S^{\frac32}A\ln^2(T) , \label{eq:policy_mid_regret_first1}\\
    \Reg_T &\lesssim \sqrt{H^2SA\ln^2(T)\prn*{HT - L^\star  - \Reg_T}} + H^{\frac52}S^{\frac32}A\ln^2(T) , \label{eq:policy_mid_regret_first2}\\
    \Reg_T &\lesssim \sqrt{H^2SA\ln^2(T)Q_{\infty}} + H^{\frac52}S^{\frac32}A\ln^2(T) ,  \label{eq:policy_regret_second}\\
    \Reg_T &\lesssim \sqrt{H^2SA\ln^2(T)V_1} + H^{\frac52}S^{\frac32}A\ln^2(T). \label{eq:policy_regret_path}
\end{align}
From \eqref{eq:policy_mid_regret_first1},
\begin{align}
    \Reg_T &\leq c\sqrt{H^2SA\ln^2(T)L^\star}+c\sqrt{H^2SA\ln^2(T)\Reg_T} + cH^{\frac52}S^{\frac32}A\ln^2(T) \tag{for some absolute constant $c$}\\ 
    &\leq c\sqrt{H^2SA\ln^2(T)L^\star} + \frac{c^2}{2}H^2SA\ln^2(T) + \frac{1}{2}\Reg_T + cH^{\frac52}S^{\frac32}A\ln^2(T) \\
    &\leq \frac{1}{2}\Reg_T + O\prn*{\sqrt{H^2SA\ln^2(T)L^\star} + H^{\frac52}S^{\frac32}A\ln^2(T)},
\end{align}
where the second line follows from the AM--GM inequality. Therefore, 
\begin{align}
    \Reg_T &\lesssim \sqrt{H^2SA\ln^2(T)L^\star} + H^{\frac52}S^{\frac32}A\ln^2(T). \label{eq:policy_regret_first1}
\end{align}
From \cref{eq:policy_mid_regret_first2}, we also have
\begin{align}
    \Reg_T &\lesssim \sqrt{H^2SA\ln^2(T)\prn*{HT - L^\star  - \Reg_T}} + H^{\frac52}S^{\frac32}A\ln^2(T)\\
    &\leq \sqrt{H^2SA\ln^2(T)\prn*{HT - L^\star}} + H^{\frac52}S^{\frac32}A\ln^2(T). \label{eq:policy_regret_first2}
\end{align}

Combining \cref{eq:policy_regret_first1,eq:policy_regret_first2,eq:policy_regret_second,eq:policy_regret_path}, we obtain
\begin{align}
    \Reg_T &\lesssim \sqrt{H^2SA\ln^2(T)\min\set*{L^\star, HT - L^\star, Q_{\infty}, V_1}} + H^{\frac52}S^{\frac32}A\ln^2(T).
\end{align}

\paragraph{2. Stochastic variance bound.}
Under the stochastic regime, \cref{lem:general_predict_result} further implies
\begin{align}
    \Reg_T \lesssim \sqrt{H^2SA\ln^2(T)\prn*{\V T + \calC}} + H^{\frac52}S^{\frac32}A\ln^2(T).
\end{align}

\paragraph{3. Stochastic gap-dependent bound.}
Moreover, we have
\begin{align}   
        &\sum_{s,a}\sqrt{\ln^2(T)\E\brk*{\sum_{t=1}^T \zeta_t(s,a)}} \\
        &= \sum_{s,a}\sqrt{\ln^2(T)\E\brk*{\sum_{t=1}^T (\I_t(s,a) - \pi_t(a\mid s) \I_t(s))^2(L_{t,h(s)} - M_{t,h(s)})^2}}\\
        &\leq H\sum_{s,a}\sqrt{\ln^2(T)\E\brk*{\sum_{t=1}^T (\I_t(s,a) - \pi_t(a\mid s) \I_t(s))^2}}\\
        &\leq H\sum_{s,a}\sqrt{\ln^2(T)\E\brk*{\sum_{t=1}^T q^{\pi_t}(s)\pi_t(a\mid s)(1 - \pi_t(a\mid s))}}\\
        &\leq 2H\sum_{s}\sum_{a\neq\pist(s)}\sqrt{\ln^2(T)\E\brk*{\sum_{t=1}^T q^{\pi_t}(s, a)}}
\end{align}

Therefore, together with \cref{eq:reg_bound_policy_in_proof}, we obtain
\begin{align}
    \Reg_T \lesssim H\sum_{s}\sum_{a \neq \pist(s)}\sqrt{\ln^2(T)\E\brk*{\sum_{t = 1}^{T}q^{\pi_t}(s,a)}} +  H^{\frac52}S^{\frac32}A\ln^2(T).
\end{align}
Finally, applying \cref{general_self_bounding} yields
\begin{align}
    \Reg_T \lesssim U + \sqrt{UC} + H^{\frac52}S^{\frac32}A\ln^2(T),
\end{align}
where $U = \sum_{s}\sum_{a\neq\pist(s)}\frac{H^2\ln^2(T)}{\Delta(s,a)}$.
\end{proof}

\subsection{Proof of \protect\cref{thm:Policy_OPT2}}
Here we provide the proof of \cref{thm:Policy_OPT2}.

\begin{thm}[Restatement of \cref{thm:Policy_OPT2}]
\cref{alg:OFTRL_PO} with the loss prediction $m_t$ defined in \cref{def:predictor_sequence2} guarantees
\begin{align}
    \Reg_T &\lesssim \sqrt{H^2SA\ln^2(T)\min\set*{L^\star, HT - L^\star, Q_{\infty}}} + H^{\frac52}S^{\frac32}A\ln^2(T).
\end{align}
Under the stochastic regime with adversarial corruption, it simultaneously ensures
\begin{align}
    \Reg_T \lesssim \sqrt{H^2SA\ln^2(T)\prn*{\V T + \calC}} + H^{\frac52}S^{\frac32}A\ln^2(T) ,
\end{align}
and
\begin{align}
    \Reg_T \lesssim \Uvar + \sqrt{\Uvar\calC}  + \sqrt{H S^2 A^2 \calC}\ln^{\frac32}(T) + H^{\frac12}S^{\frac32}A(H^2 + A^{\frac{1}{2}})\ln^2(T),
\end{align}
where $\Uvar = \sum_{s}\sum_{a\neq\pist(s)}\frac{H\Varmax(s)\ln^2(T)}{\Delta(s,a)}$.
\end{thm}

\begin{proof}
The proof follows the same template as \cref{thm:Policy_OPT}. The main differences are that we do not derive a path-length bound, and the stochastic gap-dependent bound is variance-aware.

We start from \cref{eq:policy_reg_general}, which gives
\begin{align}
    \Reg_T
    \lesssim \sum_{s,a}\sqrt{\ln^2(T)\mathbb{E}\brk*{\sum_{t=1}^T \zeta_t(s,a)}} + H^{\frac52}S^{\frac32}A\ln^2(T). \label{eq:reg_bound_policy_in_proof2}
\end{align}
By the definition of $\zeta_t(s,a)$, the same argument as in \cref{eq:common_bound_policy} yields
\begin{align}
    \sum_{s,a}\sqrt{\ln^2(T)\E\brk*{\sum_{t = 1}^{T}\zeta_t(s,a)}}
    &\lesssim \sqrt{H^2SA\ln^2(T)\E\brk*{\sum_{t=1}^T\sum_{s,a}\I_t(s,a)(\ell_t(s,a) - m_t(s,a))^2}}. \label{eq:common_bound_policy2}
\end{align}

\paragraph{1. Bounds for the adversarial regime.}
Applying \cref{lem:general_predict_result2} to \eqref{eq:common_bound_policy2} gives
\begin{align}
    &\sqrt{H^2SA\ln^2(T)\E\brk*{\sum_{t = 1}^{T}\sum_{s,a}\I_t(s,a)(\ell_t(s,a)-m_t(s,a))^2}} \\
    &\lesssim \sqrt{H^2SA\ln^2(T)\prn*{\min\set*{L^\star  + \Reg_T, HT - L^\star  - \Reg_T, Q_{\infty}} +  SA\ln(T) + SA}}\\
    &\lesssim \sqrt{H^2SA\ln^2(T)\min\set*{L^\star  + \Reg_T, HT - L^\star  - \Reg_T, Q_{\infty}}} + HSA\ln^{\frac32}(T).
\end{align}
Then, we obtain
\begin{align}
    \Reg_T &\lesssim \sqrt{H^2SA\ln^2(T)\prn*{L^\star  + \Reg_T}} + H^{\frac52}S^{\frac32}A\ln^2(T), \\
    \Reg_T &\lesssim \sqrt{H^2SA\ln^2(T)\prn*{HT - L^\star  - \Reg_T}} + H^{\frac52}S^{\frac32}A\ln^2(T), \\
    \Reg_T &\lesssim \sqrt{H^2SA\ln^2(T)Q_{\infty}} + H^{\frac52}S^{\frac32}A\ln^2(T).
\end{align}
Applying the same calculation as in \cref{eq:policy_mid_regret_first1,eq:policy_mid_regret_first2,eq:policy_regret_second} gives
\begin{align}
    \Reg_T \lesssim \sqrt{H^2SA\ln^2(T)\min\set*{L^\star, HT - L^\star, Q_{\infty}}} + H^{\frac52}S^{\frac32}A\ln^2(T).
\end{align}

\paragraph{2. Stochastic variance bound.}
Under the stochastic regime, \cref{lem:general_predict_result2} further implies
\begin{align}
    \Reg_T \lesssim \sqrt{H^2SA\ln^2(T)\prn*{\V T + \calC}} + H^{\frac52}S^{\frac32}A\ln^2(T).
\end{align}

\paragraph{3. Stochastic gap-dependent bound.}
Moreover, by using \cref{lem:zeta_decompose} and \cref{lem:lower_order_stochastic}, we have
\begin{align}   
        &\sum_{s,a}\sqrt{\ln^2(T)\E\brk*{\sum_{t=1}^T \zeta_t(s,a)}} \label{eq:before_variance_policy}\\
        &\leq \sum_{s,a}\sqrt{\ln^2(T)\E\brk*{\sum_{t=1}^T 2H\Varmax(s) q^{\pi_t}(s, a)(1 - \pi_t(a\mid s))}}\label{eq:after_variance_policy}\\
        &\qquad + \sum_{s,a}\sqrt{\ln^2(T)\E\brk*{\sum_{t=1}^T 2H\sum_{s',a'}q^{\pi_t}(s',a')\E_t\brk*{(\ell_t(s', a') - \ell'_t(s',a') + \mu(s',a') - m_t(s',a'))^2}}} \tag{by \cref{lem:zeta_decompose}}\\
        &\lesssim \sum_{s,a}\sqrt{\ln^2(T)\E\brk*{\sum_{t = 1}^{T}H\Varmax(s)q^{\pi_t}(s,a)\prn*{1 - \pi_t(a\mid s)}}} + \sum_{s,a}\sqrt{H\ln^2(T)\prn*{SA\ln^2(T) + \calC\ln(T)}} \tag{by \cref{lem:lower_order_stochastic}}\\
        &\leq \sum_{s,a}\sqrt{\ln^2(T)\E\brk*{\sum_{t = 1}^{T}H\Varmax(s)q^{\pi_t}(s,a)\prn*{1 - \pi_t(a\mid s)}}}  + \sqrt{H S^2 A^2 \calC}\ln^{\frac32}(T) + H^{\frac12}S^{\frac32}A^{\frac32}\ln^2(T).\\
        &\leq \sum_{s}2\sqrt{H\Varmax(s)}\sum_{a \neq \pist(s)}\sqrt{\ln^2(T)\E\brk*{\sum_{t = 1}^{T}q^{\pi_t}(s,a)}}  + \sqrt{H S^2 A^2 \calC}\ln^{\frac32}(T) + H^{\frac12}S^{\frac32}A^{\frac32}\ln^2(T), \label{eq:variance-aware_policy}
\end{align}
where the last inequality follows by the same argument as in \cref{eq:variance-aware_occup}.

Therefore, together with \cref{eq:reg_bound_policy_in_proof}, we obtain
\begin{align}
    \Reg_T \lesssim \sum_{s}\sqrt{H\Varmax(s)}\sum_{a \neq \pist(s)}\sqrt{\ln^2(T)\E\brk*{\sum_{t = 1}^{T}q^{\pi_t}(s,a)}} + \sqrt{H S^2 A^2 \calC}\ln^{\frac32}(T) + H^{\frac12}S^{\frac32}A(H^{2} + A^{\frac{1}{2}})\ln^2(T).
\end{align}
Finally, applying \cref{general_self_bounding} yields
\begin{align}
    \Reg_T \lesssim \Uvar + \sqrt{\Uvar\calC} + \sqrt{H S^2 A^2 \calC}\ln^{\frac32}(T) + H^{\frac12}S^{\frac32}A(H^{2} + A^{\frac{1}{2}})\ln^2(T),
\end{align}
where $\Uvar = \sum_{s}\sum_{a\neq\pist(s)}\frac{H\Varmax(s)\ln^2(T)}{\Delta(s,a)}$.
\end{proof}

\begin{rem}[Restatement of \cref{rem:Policy_OPT2}] 
    In the stochastic regime with adversarial corruption, suppose that the uncorrupted losses are generated independently and are uncorrelated across layers, \cref{thm:Policy_OPT2} improves by a factor of $H$ to $\Uvar = \sum_{s}\sum_{a\neq\pist(s)}\frac{\Varmax(s)\ln^2(T)}{\Delta(s,a)}$.
\end{rem}
\begin{proof}
    In the proof of \cref{thm:Policy_OPT2}, applying \cref{cor:zeta_decompose2} to \cref{eq:before_variance_policy} yields the following inequality in place of \cref{eq:after_variance_policy}:
    \begin{align}
        &\sum_{s,a}\sqrt{\ln^2(T)\E\brk*{\sum_{t=1}^T \zeta_t(s,a)}} \\
        &\leq \sum_{s,a}\sqrt{\ln^2(T)\E\brk*{\sum_{t=1}^T 2\Varmax(s) q^{\pi_t}(s, a)(1 - \pi_t(a\mid s))}}\\
        &\qquad + \sum_{s,a}\sqrt{\ln^2(T)\E\brk*{\sum_{t=1}^T 2H\sum_{s',a'}q^{\pi_t}(s',a')\E_t\brk*{(\ell_t(s', a') - \ell'_t(s',a') + \mu(s',a') - m_t(s',a'))^2}}}
    \end{align}
    Compared to \cref{eq:after_variance_policy}, this bound is improved by a factor of $H$, and can be interpreted as replacing the $H\Varmax(s)$ term by $\Varmax(s)$.
    The remainder of the proof follows by the same steps as in \cref{thm:Policy_OPT2}.
\end{proof}

\section{Auxiliary Lemmas}
This section provides auxiliary lemmas used in \cref{app:global_opt_proofs,app:policy_opt_proofs}.

\subsection{Concentration Bounds in the Stochastic Regime}
This section establishes the key properties of the loss prediction \cref{def:predictor_sequence2} under the stochastic regime with adversarial corruption.
In particular, we choose $m_t$ to be the empirical mean of the observed losses as follows:
\begin{align}
    m_t(s,a)=\frac{\sum_{\tau=1}^{t-1}\I_\tau(s,a)\,\ell_\tau(s,a)}{\max\{1,N_{t-1}(s,a)\}} .
\end{align}
Here, $N_{t-1}(s,a) = \sum_{\tau=1}^{t-1}\I_\tau(s,a)$ denotes the number of visits to the state-action pair $(s,a)$ up to episode $t-1$.
For the analysis, we also introduce the following corresponding empirical mean computed from the uncorrupted losses $\ell'$.
\begin{align}
    m'_t(s,a) = \frac{\sum_{\tau=1}^{t-1}\I_\tau(s,a)\,\ell'_\tau(s,a)}{\max\{1,N_{t-1}(s,a)\}}. \label{eq:uncorrupted_predictor}
\end{align}

\begin{lem}[{Bennett’s inequality, \citealt[Theorem 3]{maurer2009empirical}}]
\label{lem:bennett_normal}
Let $X_1, X_2, \dots, X_n$ be i.i.d.~random variables with values in $[0,1]$.
Then, with probability at least $1 - 2\delta$, it holds that
\begin{align}
    \left|\E\brk*{X_1} - \frac{1}{n}\sum_{t=1}^n X_i\right|
    \leq \sqrt{\frac{2\Var(X_1)\ln(1/\delta)}{n}} +\frac{\ln(1/\delta)}{3n},
\end{align}
where $\Var(X_1)$ is the variance of $X_1$.
\end{lem}

\begin{lem}\label{lem:bennett_union}
We have with probability at least $1 - 2\delta$,
\begin{align}
    \abs{\mu(s,a) - m'_t(s,a)}  \leq \sqrt{\frac{2\sigma^2(s,a)\ln(SAT/\delta)}{\max\set{1, N_{t-1}(s,a)}}} +\frac{\ln(SAT/\delta)}{3\max\set{1, N_{t-1}(s,a)}}
\end{align}
for all state-action pairs $(s,a)$ and $t \leq T$.
\end{lem}
\begin{proof}
Apply \cref{lem:bennett_normal} with $\delta'=\delta/(SAT)$ and take a union bound over all state-action pairs $(s,a)\in\calS\times\calA$ and all $t\leq T$, which completes the proof.
\end{proof}

\begin{dfn}\label{dfn:good_event}
     Define $\calE$ to be the event that \cref{lem:bennett_union} holds.
     In this case, $\Pr(\calE) \geq 1 - 2\delta$.
\end{dfn}

\begin{lem}\label{lem:bennett_square}
On the event $\calE$ defined in \cref{dfn:good_event}, it holds that
\begin{align}
    \prn{\mu(s,a) - m'_t(s,a)}^2 \leq \frac{2\ln(SAT/\delta)}{\max\set{1, N_{t-1}(s,a)}}
\end{align}
for all state-action pairs $(s,a)$ and $t \leq T$.
\end{lem}
\begin{proof}
Let $x = \sqrt{\frac{\ln(SAT/\delta)}{\max\set{1, N_{t-1}(s,a)}}}$. Since $\sigma^2(s,a) \leq \frac14$, \cref{lem:bennett_union} implies
\begin{align}
    \abs{\mu(s,a) - m'_t(s,a)}  \leq x +\frac{x^2}{3}.
\end{align}
Hence,
\begin{align}
    \prn{\mu(s,a) - m'_t(s,a)}^2 &\leq \prn*{x +\frac{x^2}{3}}^2 \leq x^2\prn*{\frac19 x^2 + \frac23 x + 1}.
\end{align}
Moreover, by using $(\mu(s,a)-m'_t(s,a))^2\le 1$, we have
\begin{align}
    \prn{\mu(s,a) - m'_t(s,a)}^2 &\leq \min\set*{x^2\prn*{\frac19 x^2 + \frac23 x + 1}, 1}\\
    &\leq \begin{cases}
        1 & \text{ if } x \geq 1\\
        2x^2 & \text{ if } x < 1
    \end{cases}\\
    &\leq 2x^2,
\end{align}
which concludes the proof.
\end{proof}

\begin{lem}\label{lem:sum_q_N}
It holds that
\begin{align}
\E\brk*{\sum_{t = 1}^{T}\sum_{s,a}\frac{q^{\pi_t}(s,a)}{\max\set{1, N_{t-1}(s,a)}}}\leq SA\ln(T)+ 2SA.
\end{align}
\end{lem}
\begin{proof}
Using $\E_t\brk*{\mathbb{I}_t(s,a)}=q^{\pi_t}(s,a)$, we have
\begin{align}
\E\brk*{\sum_{t = 1}^{T}\sum_{s,a}\frac{q^{\pi_t}(s,a)}{\max\set{1, N_{t-1}(s,a)}}}
&= \E\brk*{\sum_{t = 1}^{T}\sum_{s,a}\frac{\I_t(s,a)}{\max\set{1, N_{t-1}(s,a)}}}\\
&= \E\brk*{\sum_{s,a}\sum_{i = 0}^{N_{T}(s,a) - 1}\frac{1}{\max\set{1, i}}}\\
&\leq \sum_{s,a}\prn*{2 + \int_{1}^{T}\frac{1}{x} \mathrm{d}x}\\
&= \sum_{s,a}\prn*{2 + \ln (T)}\\
&\leq SA\ln(T)+ 2SA.
\end{align}
\end{proof}

\begin{lem}
\label{lem:bound_bias_uncorrupt}
It holds that
\begin{align}
\E\brk*{\sum_{t = 1}^{T}\sum_{s,a}q^{\pi_t}(s,a)(\mu(s,a) -  m'_t(s,a))^2} \lesssim SA\ln^2(T).
\end{align}
\end{lem}
\begin{proof}
Let $\bar{\calE}$ denote the complement of the event $\calE$ in \cref{dfn:good_event}. Then,
\begin{align}
   \E\brk*{\sum_{t = 1}^{T}\sum_{s,a}q^{\pi_t}(s,a)(\mu(s,a) -  m'_t(s,a))^2} 
   &= \Pr(\calE)\E\brk*{\sum_{t = 1}^{T}\sum_{s,a}q^{\pi_t}(s,a)(\mu(s,a) -  m'_t(s,a))^2 \relmiddle| \calE}\\
   &\qquad + \Pr(\bar\calE) \E\brk*{\sum_{t = 1}^{T}\sum_{s,a}q^{\pi_t}(s,a)(\mu(s,a) -  m'_t(s,a))^2 \relmiddle| \bar\calE}\\
   &\leq \E\brk*{\sum_{t = 1}^{T}\sum_{s,a}q^{\pi_t}(s,a)(\mu(s,a) -  m'_t(s,a))^2 \relmiddle| \calE} + 2HT\delta. \label{eq:split_E_bias}
\end{align}
On the event $\calE$, we have
\begin{align}
    \E\brk*{\sum_{t = 1}^{T}q^{\pi_t}(s,a)(\mu(s,a) -  m'_t(s,a))^2 \relmiddle| \calE}
    &\leq \E\brk*{\sum_{t = 1}^{T}\sum_{s,a}q^{\pi_t}(s,a)\frac{2\ln(SAT/\delta)}{\max\set{1, N_{t-1}(s,a)}}} \tag{by \cref{lem:bennett_square}}\\
   &\leq 2(SA\ln(T)+ 2SA)\ln(SAT/\delta) \tag{by \cref{lem:sum_q_N}}\\
   &\lesssim SA\ln(T)\ln(SAT/\delta).
\end{align}
Choosing $\delta = 1/T$ and combining with \eqref{eq:split_E_bias} yields
\begin{align}
   \E\brk*{\sum_{t = 1}^{T}\sum_{s,a}q^{\pi_t}(s,a)(\mu(s,a) -  m'_t(s,a))^2}\lesssim SA\ln^2(T).
\end{align}
\end{proof}

\begin{lem}
\label{lem:bound_bias_corrupt}
Suppose that $m_t$ is defined as \cref{def:predictor_sequence2}.
Then, it holds that
\begin{align}
\E\brk*{\sum_{t = 1}^{T}\sum_{s,a}q^{\pi_t}(s,a)\prn*{m'_t(s,a) -  m_t(s,a)}^2}
\leq \calC\ln(T) + 2\calC.
\end{align}
\end{lem}
\begin{proof}
By the definitions of $m_t$ in \cref{def:predictor_sequence2} and $m'_t$ in \cref{eq:uncorrupted_predictor}, we have
\begin{align}
m'_t(s,a) - m_t(s,a) 
=  \frac{\sum_{\tau=1}^{t - 1} \I_\tau(s,a)\prn*{\ell'_\tau(s,a) - \ell_\tau(s,a)}}{\max\set{1, N_{t-1}(s,a)}} . \label{eq:mprime_minus_m}
\end{align}

Thus,
\begin{align}
\prn*{m'_t(s,a) - m_t(s,a)}^2
&= \frac{\prn*{\sum_{\tau=1}^{t - 1} \I_\tau(s,a)\prn*{\ell'_\tau(s,a) - \ell_\tau(s,a)}}^2}{\prn*{\max\set{1, N_{t-1}(s,a)}}^2} \\
&\leq \frac{N_{t-1}(s,a)\sum_{\tau=1}^{t - 1} \I_\tau(s,a)\prn*{\ell'_\tau(s,a) - \ell_\tau(s,a)}^2}{\prn*{\max\set{1, N_{t-1}(s,a)}}^2} \\
&\leq \frac{\sum_{\tau=1}^{t - 1} \I_\tau(s,a)\prn*{\ell'_\tau(s,a) - \ell_\tau(s,a)}^2}{\max\set{1, N_{t-1}(s,a)}} , \label{eq:mprime_minus_m_square}
\end{align}
where the first inequality follows from the Cauchy--Schwarz inequality.

From \cref{eq:mprime_minus_m_square}, we obtain
\begin{align}
&\E\brk*{\sum_{t = 1}^{T}\sum_{s,a}q^{\pi_t}(s,a)\prn*{m'_t(s,a) - m_t(s,a)}^2} \\
&\leq \E\brk*{\sum_{t = 1}^{T}\sum_{s,a}q^{\pi_t}(s,a)\frac{\sum_{\tau=1}^{t - 1} \I_\tau(s,a)\prn*{\ell'_\tau(s,a) - \ell_\tau(s,a)}^2}{\max\set{1, N_{t-1}(s,a)}}} \\
&= \E\brk*{\sum_{s,a}\sum_{\tau=1}^{T-1}\I_\tau(s,a)\prn*{\ell'_\tau(s,a) - \ell_\tau(s,a)}^2
\sum_{t=\tau+1}^{T}\frac{\I_t(s,a)}{\max\set{1, N_{t-1}(s,a)}}} \\
&\leq \prn*{2+\ln T}\E\brk*{\sum_{\tau=1}^{T}\sum_{s,a}\I_\tau(s,a)\abs*{\ell'_\tau(s,a) - \ell_\tau(s,a)}} \label{eq:mprime_m_logT}\\
&\leq \prn*{2+\ln T}\calC, \label{eq:mprime_m_to_abs}
\end{align}
where \cref{eq:mprime_m_logT} uses $-1\leq \ell'_\tau(s,a) - \ell_\tau(s,a)\leq 1$ and
\begin{align}
\sum_{t=\tau+1}^{T}\frac{\I_t(s,a)}{\max\set{1, N_{t-1}(s,a)}}
\leq \sum_{i=0}^{N_T(s,a) - 1}\frac{1}{\max\set{1,i}}
\leq 2 + \int_{1}^{T}\frac{1}{x} \mathrm{d}x= 2+\ln (T).
\end{align}
The last inequality follows from the definition of the corruption budget 
$\E\brk*{\sum_{\tau=1}^{T}\sum_{s,a}\I_\tau(s,a)\abs{\ell'_\tau(s,a) - \ell_\tau(s,a)}}\leq \calC$,
which completes the proof.
\end{proof}

\begin{lem}\label{lem:lower_order_stochastic}
Suppose $m_t$ is defined in \cref{def:predictor_sequence2}. It holds that
\begin{align}
\E\brk*{\sum_{t = 1}^{T}\sum_{s,a}q^{\pi_t}(s,a)\prn*{\ell_t(s,a) - \ell'_t(s,a) + \mu(s,a) - m_t(s,a)}^2}
\lesssim SA\ln^2(T) + \calC\ln(T).
\end{align}
\end{lem}

\begin{proof}
It holds that
\begin{align}
&\E\brk*{\sum_{t = 1}^{T}\sum_{s,a}q^{\pi_t}(s,a)\prn*{\ell_t(s,a) - \ell'_t(s,a) + \mu(s,a) - m_t(s,a)}^2} \\
&= \E\brk*{\sum_{t = 1}^{T}\sum_{s,a}q^{\pi_t}(s,a)(\prn*{\ell_t(s,a) - \ell'_t(s,a)} + \prn*{\mu(s,a) - m'_t(s,a)} + \prn*{m'_t(s,a) - m_t(s,a)} )^2} \\ 
&\leq 3\E\brk*{\sum_{t = 1}^{T}\sum_{s,a}q^{\pi_t}(s,a)\prn{\ell_t(s,a) - \ell'_t(s,a)}^2} + 3\E\brk*{\sum_{t = 1}^{T}\sum_{s,a}q^{\pi_t}(s,a)(\mu(s,a) - m'_t(s,a))^2} \\
&\qquad + 3\E\brk*{\sum_{t = 1}^{T}\sum_{s,a}q^{\pi_t}(s,a)(m'_t(s,a) - m_t(s,a))^2} \tag{by $(x+y+z)^2\leq 3(x^2+y^2+z^2)$ for $x,y,z\in\R$}\\
&\leq 3\E\brk*{\sum_{t = 1}^{T}\sum_{s,a}q^{\pi_t}(s,a)\abs{\ell_t(s,a) - \ell'_t(s,a)}} \\
&\qquad + 3\E\brk*{\sum_{t = 1}^{T}\sum_{s,a}q^{\pi_t}(s,a)\prn*{\mu(s,a) - m'_t(s,a)}^2}
+ 3\E\brk*{\sum_{t = 1}^{T}\sum_{s,a}q^{\pi_t}(s,a)\prn*{m'_t(s,a) - m_t(s,a)}^2}. \label{eq:lower_order_decomp}
\end{align}
The first term is bounded by the corruption budget as
\begin{align}
\E\brk*{\sum_{t = 1}^{T}\sum_{s,a}q^{\pi_t}(s,a)\abs{\ell_t(s,a) - \ell'_t(s,a)}} \leq \calC.
\end{align}
The second term is bounded by \cref{lem:bound_bias_uncorrupt} and the third term is bounded by \cref{lem:bound_bias_corrupt}.
Combining these bounds, we have
\begin{align}
\E\brk*{\sum_{t = 1}^{T}\sum_{s,a}q^{\pi_t}(s,a)\prn*{\ell_t(s,a) - \ell'_t(s,a) + \mu(s,a) - m_t(s,a)}^2}
\lesssim \calC + SA\ln^2(T) + \calC\ln(T),
\end{align}
which completes the proof.
\end{proof}

\subsection{General Lemmas for Data-Dependent and Best-of-Both-Worlds Bounds}
In this section, we present general tools for deriving data-dependent bounds and for establishing self-bounding inequalities, which together yield best-of-both-worlds guarantees.

The first lemma is a standard tool for deriving path-length bounds when $m_t$ is updated as in \cref{def:predictor_sequence}. It appears in \citet[Proposition 13]{ito2021parameter} and \citet[Lemma 7]{tsuchiya2023further}. Here, we extend it to the MDP setting.

\begin{lem}
    \label{lem:m_to_mstar}
    Suppose $m_t$ is defined in \cref{def:predictor_sequence}. Then, for any sequence $m^*_t \in [0,1]^{S\times A}$ and any $\xi\in\prn*{0,\frac12}$, we have
    \begin{align}
        &\sum_{t = 1}^T\sum_{s,a} \I_t(s,a)(\ell_t(s,a) - m_t(s,a))^2 \\
        &\leq \frac{1}{1-2\xi}\sum_{t = 1}^T \sum_{s,a}\I_t(s,a)(\ell_t(s,a) - m^*_t(s,a))^2 + \frac{1}{\xi(1-2\xi)}\prn*{\frac{SA}{4}+2\sum_{t = 1}^{T-1}\nrm*{m^*_{t + 1} - m^*_t}_1}.
    \end{align}
\end{lem}
\begin{proof}
    Fix any $(s,a)$. For episodes $t$ with $\I_t(s,a)=1$, the update rule of $m_t$ \cref{def:predictor_sequence} implies that
    \begin{align}
        &(\ell_t(s,a) - m_t(s,a))^2 - (\ell_t(s,a) - m^*_t(s,a))^2 \\
        &\leq 2(\ell_t(s,a) - m_t(s,a))(m^*_t(s,a) - m_t(s,a)) \\
        &= 2(\ell_t(s,a) - m_t(s,a))(m_{t+1}(s,a) - m_t(s,a) + m^*_t(s,a) - m_{t + 1}(s,a)) \\
        &= 2\xi(\ell_t(s,a) - m_t(s,a))^2 + \frac{2}{\xi}(m_{t+1}(s,a) - m_t(s,a))(m^*_t(s,a) - m_{t+1}(s,a)) \\
        &\leq 2\xi(\ell_t(s,a) - m_t(s,a))^2 + \frac{1}{\xi}((m^*_t(s,a) - m_t(s,a))^2 - (m^*_t(s,a) - m_{t+1}(s,a))^2 ),
    \end{align}
    where the inequalities follow from $x^2 - y^2 = 2x(x - y) - (x - y)^2 \leq 2x(x - y)$ for $x,y \in \R$. 
    Hence, we have
    \begin{align}
        (\ell_t(s,a) - m_t(s,a))^2 &\leq \frac{1}{1 - 2\xi}(\ell_t(s,a) - m^*_t(s,a))^2 \\
        &\qquad + \frac{1}{\xi(1 - 2\xi)}((m^*_t(s,a) - m_t(s,a))^2 - (m^*_t(s,a) - m_{t+1}(s,a))^2) .
    \end{align}
    Then, for any $s, a$ and for $m_t = m_{t + 1}$ when $\I_t(s,a) = 0$, we obtain
    \begin{align}
        &\sum_{t = 1}^T \I_t(s,a) (\ell_t(s,a) - m_t(s,a))^2 \\
        &\leq \frac{1}{1 - 2\xi} \sum_{t=1}^T \I_t(s,a) (\ell_t(s,a) - m^*_t(s,a))^2 \\
        &\qquad + \frac{1}{\xi(1 - 2\xi)}\sum_{t=1}^T ((m^*_t(s,a) - m_t(s,a))^2 - (m^*_t(s,a) - m_{t+1}(s,a))^2) \\
        &= \frac{1}{1 - 2\xi} \sum_{t=1}^T \I_t(s,a) (\ell_t(s,a) - m^*_t(s,a))^2 \\
        &\qquad + \frac{1}{\xi(1 - 2\xi)}\set*{\sum_{t=1}^{T-1} ((m^*_{t+1}(s,a) - m_{t+1}(s,a))^2 - (m^*_t(s,a) - m_{t+1}(s,a))^2) + (m^*_{1}(s,a) - m_{1}(s,a))^2}\\
        &\leq \frac{1}{1 - 2\xi} \sum_{t=1}^T \I_t(s,a) (\ell_t(s,a) - m^*_t(s,a))^2 \\
        &\qquad + \frac{1}{\xi(1 - 2\xi)}\set*{\sum_{t=1}^{T-1} (m^*_{t+1}(s,a) + m^*_t(s,a) - 2m_{t+1}(s,a))(m^*_{t + 1}(s,a) - m^*_t(s,a)) + \frac{1}{4}}\\
        &\leq \frac{1}{1 - 2\xi} \sum_{t=1}^T \I_t(s,a) (\ell_t(s,a) - m^*_t(s,a))^2 + \frac{1}{\xi(1 - 2\xi)}\set*{2\sum_{t=1}^{T-1} \abs*{m^*_{t + 1}(s,a) - m^*_t(s,a)} + \frac{1}{4}}.
    \end{align}
    Therefore, 
    \begin{align}
        &\sum_{t = 1}^T\sum_{s,a} \I_t(s,a)(\ell_t(s,a) - m_t(s,a))^2 \\
        &\leq \frac{1}{1-2\xi}\sum_{t = 1}^T \sum_{s,a} \I_t(s,a)(\ell_t(s,a) - m^*_t(s,a))^2 + \frac{1}{\xi(1-2\xi)}\prn*{\frac{SA}{4}+2\sum_{t = 1}^{T-1}\nrm*{m^*_{t + 1} - m^*_t}_1}.
    \end{align}
\end{proof}

The next two lemmas concern the loss prediction $m_t$, which is updated as in \cref{def:predictor_sequence2}.
\begin{lem}
\label{lem:m_to_m_star_comp}
Let $\ell_1,\dots,\ell_T\in[0,1]$ be any sequence and let $m^*\in[0,1]$ be arbitrary. Define
\begin{align}
    m_t = \frac{1}{\max\set{1, t - 1}} \sum_{\tau = 1}^{t-1}\ell_\tau
\end{align}
Then, it holds that
\begin{align}
        \sum_{t=1}^T(\ell_t - m_t)^2 \leq \sum_{t=1}^T(\ell_t - m^*)^2 + \ln(T) + 1.
\end{align}
\end{lem}
\begin{proof}
    For $t\ge 2$, $m_t$ is expressed as
    \begin{align}
        m_t \in \argmin_{m\in \mathbb{R}} \left\{\sum_{\tau = 1}^{t-1} (m - \ell_\tau)^2\right\}.
    \end{align}
    Then, since $\sum_{\tau = 1}^{t-1} (m - \ell_\tau)^2$ is a quadratic function of $m$, for any $m \in \R$,
    \begin{align}
        \sum_{\tau = 1}^{t-1}(m - \ell_\tau)^2 
        = \sum_{\tau=1}^{t-1}(m_t - \ell_\tau)^2 + (t - 1)(m - m_t)^2. \label{eq:m_quad}
    \end{align}
    Thus, by applying \cref{eq:m_quad}, we obtain
    \begin{align}
        \sum_{t=1}^T(\ell_t - m^*)^2
        &= \sum_{t=1}^T(m_{T+1} - \ell_t)^2 + T(m^* - m_{T+1})^2 \tag{by \cref{eq:m_quad}} \\
        &\geq \sum_{t=1}^T (m_{T+1} - \ell_t)^2 \\
        &= \sum_{t=1}^{T-1} (m_{T+1} - \ell_t)^2 + (m_{T+1} - \ell_T)^2 \\
        &= \sum_{t=1}^{T-1} (m_T - \ell_t)^2 + (T - 1)(m_{T+1} - m_T)^2 + (m_{T+1} - \ell_T)^2 \tag{by repeatedly using \cref{eq:m_quad}} \\
        &= \sum_{t=1}^T (t - 1)(m_{t+1} - m_t)^2 + \sum_{t=1}^T(m_{t+1} - \ell_t)^2.\label{eq:m_star_decomp}
    \end{align}
Then, we have
    \begin{align}
    \sum_{t=1}^T(\ell_t - m_t)^2 - \sum_{t=1}^T(\ell_t - m^*)^2
    &\leq \sum_{t=1}^T(\ell_t - m_t)^2 -\sum_{t=1}^T (t -1)(m_{t+1} - m_t)^2 - \sum_{t=1}^T(m_{t+1} - \ell_t)^2 \tag{by \cref{eq:m_star_decomp}}\\
    &= \sum_{t=1}^T\prn*{(2\ell_t - m_t - m_{t+1})(m_{t+1} - m_t)  - (t - 1)(m_{t+1} - m_t)^2} \\
    &= \sum_{t=1}^T\prn*{\frac{2t - 1}{t^2}(\ell_t - m_t)^2  - \frac{t - 1}{t^2}(\ell_t - m_t)^2} \tag{by $m_{t+1} - m_t = \frac{1}{t}(\ell_t - m_t)$}\\
    &\leq \sum_{t=1}^T\frac{1}{t} \tag{by $(\ell_t(s,a) - m_t(s,a))^2 \leq 1$}\\
    &\leq 1 + \int_{1}^{T}\frac{1}{x} \mathrm{d} x \leq 1 + \ln(T),
\end{align}
which completes the proof.
\end{proof}

\begin{lem}
    \label{lem:m_to_mstar2}
    Suppose $m_t$ is defined in \cref{def:predictor_sequence2}. Then, for any $m^* \in [0,1]^{S\times A}$, we have
    \begin{align}
        \sum_{t = 1}^T\sum_{s,a} \I_t(s,a)(\ell_t(s,a) - m_t(s,a))^2 
        &\leq \sum_{t = 1}^T\sum_{s,a} \I_t(s,a)\prn*{\ell_t(s,a) - m^*(s,a)}^2 + SA\ln(T)+SA.
    \end{align}
\end{lem}
\begin{proof}
Fix any $(s,a)$. Since $m_t$ is defined as $m_t(s,a)=\frac{\sum_{\tau=1}^{t-1}\I_\tau(s,a)\,\ell_\tau(s,a)}{\max\{1,N_{t-1}(s,a)\}}$, we use \cref{lem:m_to_m_star_comp} and obtain
\begin{align}
\sum_{t=1}^T \I_t(s,a)\prn*{\ell_t(s,a)-m_t(s,a)}^2
&\leq \sum_{t=1}^T \I_t(s,a)\prn*{\ell_t(s,a)-m^*}^2 +\ln(N_T(s,a))+1\\
&\leq \sum_{t=1}^T \I_t(s,a)\prn*{\ell_t(s,a)-m^*}^2 +\ln(T)+1.
\end{align}
Summing the above inequality over all $(s,a)\in\calS\times\calA$ completes the proof.
\end{proof}

The following lemma serves a variety of data-dependent bounds, and (in the stochastic regime) variance-dependent bounds.

\begin{lem}\label{lem:general_predict_result}
Suppose $m_t$ is defined in \cref{def:predictor_sequence}. Then, it holds that
\begin{align}
    \E\brk*{\sum_{t=1}^T\sum_{s,a}\I_t(s,a)\prn*{\ell_t(s,a) - m_t(s,a)}^2}
    \lesssim \min\set*{L^\star  + \Reg_T, HT - L^\star  - \Reg_T, Q_{\infty}, V_1} + SA.
\end{align}
Simultaneously, under the stochastic regime with adversarial corruption (\cref{subsec:regimes}), it holds that
\begin{align}
    \E\brk*{\sum_{t=1}^T\sum_{s,a}\I_t(s,a)\prn*{\ell_t(s,a) - m_t(s,a)}^2}
    \lesssim \V T + \calC + SA.
\end{align}
\end{lem}
\begin{proof}
    By using \cref{lem:m_to_mstar}, for any $m^*_t \in [0,1]^{S\times A}$ and any $\xi\in(0,\frac12)$, we obtain
    \begin{align}
        &\E\brk*{\sum_{t=1}^T\sum_{s,a}\I_t(s,a)(\ell_t(s,a) - m_t(s,a))^2} \\
        &\leq \underbrace{\frac{1}{1-2\xi}\E\brk*{\sum_{t = 1}^T \sum_{s,a}\I_t(s,a)(\ell_t(s,a) - m^*_t(s,a))^2}}_{\term_1} + \underbrace{\frac{1}{\xi(1-2\xi)}\prn*{\frac{SA}{4}+2\E\brk*{\sum_{t = 1}^{T-1}\nrm*{m^*_{t + 1} - m^*_t}_1}}}_{\term_2}
    \end{align}
In particular, if $m^*_t$ is time-invariant, then $\sum_{t=1}^{T-1}\nrm*{m^*_{t+1}-m^*_t}_1=0$ and
$\term_2 = \dfrac{SA}{4\xi(1-2\xi)}$.

\paragraph{1. First-order bound.}
Taking $m^*(s,a)\equiv 0$, we obtain
\begin{align}
    \term_1
    &= \frac{1}{1-2\xi}\E\brk*{\sum_{t=1}^T\sum_{s,a}\I_t(s,a)\ell_t(s,a)^2}\\
    &\leq \frac{1}{1-2\xi}\E\brk*{\sum_{t=1}^T\sum_{s,a}\I_t(s,a)\ell_t(s,a)}\\
    &= \frac{1}{1-2\xi}\E\brk*{\sum_{t=1}^TV^{\pi_t}(s_0; \ell_t)}\\
    &= \frac{1}{1-2\xi}\prn*{L^\star  + \Reg_T}. \label{eq:general_first-order_1}
\end{align}
Similarly, taking $m^*(s,a)\equiv 1$ yields
\begin{align}
    \term_1
    &= \frac{1}{1-2\xi}\E\brk*{\sum_{t=1}^T\sum_{s,a}\I_t(s,a)(\ell_t(s,a) - 1)^2}\\
    &\leq \frac{1}{1-2\xi}\E\brk*{\sum_{t=1}^T\sum_{s,a}\I_t(s,a)(1 - \ell_t(s,a))}\\
    &= \frac{1}{1-2\xi}\E\brk*{\sum_{t=1}^T(H - V^{\pi_t}(s_0; \ell_t))}\\
    &= \frac{1}{1-2\xi}\prn*{HT - L^\star  - \Reg_T}.\label{eq:general_first-order_2}
\end{align}

\paragraph{2. Second-order bound.}
For any time-invariant $m^*\in[0,1]^{S\times A}$,
\begin{align}
    \term_1
    &= \frac{1}{1-2\xi}\E\brk*{\sum_{t=1}^T\sum_{h=0}^{H-1}\sum_{(s,a) \in \calS_h \times \calA}\I_t(s,a)(\ell_t(s,a) - m^*(s,a))^2}\\
    &\leq \frac{1}{1-2\xi}\E\brk*{\sum_{t=1}^T\sum_{h=0}^{H-1} \|\ell_t(h) - m^*(h)\|_\infty^2}\\
    &\leq \frac{1}{1-2\xi}Q_{\infty}. \label{eq:general_second-order}
\end{align}

\paragraph{3. Path-length bound.}
Taking $m^*_t(s,a)=\ell_t(s,a)$ yields $\term_1=0$ and
\begin{align}
        \term_2 = \frac{1}{\xi(1-2\xi)}\prn*{\frac{SA}{4}+2\sum_{t = 1}^{T-1}\nrm*{\ell_{t + 1} - \ell_t}_1}
    = \frac{1}{\xi(1-2\xi)}\prn*{\frac{SA}{4}+2V_1}. \label{eq:general_path_length}
\end{align}

Combining \cref{eq:general_path_length,eq:general_first-order_1,eq:general_first-order_2,eq:general_second-order}, 
we get 
\begin{align}
    &\E\brk*{\sum_{t=1}^T\sum_{s,a}\I_t(s,a)(\ell_t(s,a) - m_t(s,a))^2} \\
    &\leq \frac{1}{1-2\xi}\min\set*{L^\star  + \Reg_T, HT - L^\star  - \Reg_T, Q_{\infty}, \frac{V_1}{\xi}} +  \frac{SA}{4\xi(1 - 2\xi)}\\
    &\lesssim \min\set*{L^\star  + \Reg_T, HT - L^\star  - \Reg_T, Q_{\infty}, V_1} +  SA,
\end{align}
where we absorb the $\xi$-dependent constants into $\lesssim$.

\paragraph{4. Stochastic variance bound.}
Under the stochastic regime with adversarial corruption, recall that $\mu(s,a)$ and $\sigma^2(s,a)$ denote the mean and variance of the uncorrupted losses $\ell'_t$, respectively.

We set the predictor to the mean $m^\star\equiv \mu$, and we obtain
\begin{align}
    \term_1 &= \frac{1}{1-2\xi}\E\brk*{\sum_{t = 1}^T \sum_{s,a}\I_t(s,a)(\ell_t(s,a) - \mu(s,a))^2}\\
    &= \frac{1}{1-2\xi}\E\brk*{\sum_{t=1}^T\sum_{s,a}\I_t(s,a)(\ell_t(s,a) - \ell_t'(s,a) + \ell_t'(s,a) - \mu(s,a))^2}\\
    &\leq \frac{2}{1-2\xi}\E\brk*{\sum_{t=1}^T\sum_{s,a}\I_t(s,a)\prn*{(\ell_t(s,a) - \ell_t'(s,a))^2 + (\ell_t'(s,a) - \mu(s,a))^2}} \\
    &\leq \frac{2}{1-2\xi}\E\brk*{\sum_{t=1}^T\sum_{h = 0}^{H-1}\sum_{(s,a) \in \calS_h \times \calA}\I_t(s,a)\abs*{\ell_t(s,a) - \ell_t'(s,a)} + \sum_{t=1}^T\sum_{s,a}q^{\pi_t}(s,a)\sigma^2(s,a)}\\
    &= \frac{2}{1-2\xi}\E\brk*{\sum_{t=1}^T\sum_{h = 0}^{H-1} \nrm*{\ell_t'(h) - \ell_t(h)}_\infty + \sum_{t=1}^T\sum_{s,a}q^{\pi_t}(s,a)\sigma^2(s,a)}\\
    &= \frac{2}{1-2\xi}\prn*{\calC + \V T}, \label{eq:general_variance_bound}
\end{align}

Therefore, in the stochastic regime with adversarial corruption, we have
\begin{align}
     \E\brk*{\sum_{t=1}^T\sum_{s,a}\I_t(s,a)(\ell_t(s,a) - m_t(s,a))^2} &\leq \frac{2}{1-2\xi}\prn*{\V T + \calC} + \frac{SA}{4\xi(1 - 2\xi)} \\
     &\lesssim \V T + \calC + SA.
\end{align}
\end{proof}

\begin{lem}\label{lem:general_predict_result2}
 Suppose $m_t$ is defined in \cref{def:predictor_sequence2}. Then, it holds that
\begin{align}
    \E\brk*{\sum_{t=1}^T\sum_{s,a}\I_t(s,a)\prn*{\ell_t(s,a) - m_t(s,a)}^2}
    \leq \min\set*{L^\star  + \Reg_T, HT - L^\star  - \Reg_T, Q_\infty}  + SA\ln(T)+SA.
\end{align}
Simultaneously, under the stochastic regime with adversarial corruption, it holds that
\begin{align}
    \E\brk*{\sum_{t=1}^T\sum_{s,a}\I_t(s,a)\prn*{\ell_t(s,a) - m_t(s,a)}^2}
    \leq \V T + \calC + SA\ln(T)+SA.
\end{align}
\end{lem}
\begin{proof}
The argument follows the same lines as \cref{lem:general_predict_result}.

By using \cref{lem:m_to_mstar2}, for any $m^*_t \in [0,1]^{S\times A}$, we obtain
\begin{align}
    \E\brk*{\sum_{t=1}^T\sum_{s,a}\I_t(s,a)(\ell_t(s,a) - m_t(s,a))^2}
    &\leq \E\brk*{\sum_{t = 1}^T\sum_{s,a} \I_t(s,a)\prn*{\ell_t(s,a) - m^*(s,a)}^2} + SA\ln(T)+SA.\label{eq:mean_base}
\end{align}

\paragraph{1. First-order bound.}
Taking $m^*(s,a)\equiv 0$ in \cref{eq:mean_base} and proceeding as in \cref{eq:general_first-order_1} yields
\begin{align}
    \E\brk*{\sum_{t=1}^T\sum_{s,a}\I_t(s,a)(\ell_t(s,a) - m_t(s,a))^2}
    &\leq \prn*{L^\star + \Reg_T} + SA\ln(T)+SA. \label{eq:general_first-order_1_mean}
\end{align}
Similarly, taking $m^*(s,a)\equiv 1$ and proceeding as in \cref{eq:general_first-order_2} yields
\begin{align}
    \E\brk*{\sum_{t=1}^T\sum_{s,a}\I_t(s,a)(\ell_t(s,a) - m_t(s,a))^2}
    &\leq \prn*{HT - L^\star  - \Reg_T} + SA\ln(T)+SA.\label{eq:general_first-order_2_mean}
\end{align}

\paragraph{2. Second-order bound.}
Since \cref{eq:mean_base} holds for any $m^\star$, the same argument as in \cref{eq:general_second-order} gives
\begin{align}
    \E\brk*{\sum_{t=1}^T\sum_{s,a}\I_t(s,a)(\ell_t(s,a) - m_t(s,a))^2}
    &\leq Q_{\infty} + SA\ln(T)+SA.\label{eq:general_second-order_mean}
\end{align}

Combining \cref{eq:general_first-order_1_mean,eq:general_first-order_2_mean,eq:general_second-order_mean}, we get 
\begin{align}
    \E\brk*{\sum_{t=1}^T\sum_{s,a}\I_t(s,a)(\ell_t(s,a) - m_t(s,a))^2} 
    &\leq \min\set*{L^\star  + \Reg_T, HT - L^\star  - \Reg_T, Q_\infty}  + SA\ln(T)+SA.
\end{align}

\paragraph{3. Stochastic variance bound.}
In the stochastic regime with adversarial corruption, we take $m^\star\equiv \mu$ and proceed as in~\cref{eq:general_variance_bound} to obtain
\begin{align}
     \E\brk*{\sum_{t=1}^T\sum_{s,a}\I_t(s,a)(\ell_t(s,a) - m_t(s,a))^2} 
    &\leq \prn*{\calC + \V T}  + SA\ln(T)+SA,
\end{align}
which completes the proof.
\end{proof}

Finally, we generalize a self-bounding argument that appears in
\citet[Appendix H]{dann2023best} and \citet[Appendix A.1]{jin2020simultaneously},
which is useful in deriving gap-dependent bounds in the stochastic regime with adversarial corruption.
We first note that in the stochastic regime with adversarial corruption, the regret is lower bounded as follows:
\begin{lem}[{\citealt[Section 2.1]{jin2020learning}}]
\label{lem:corruption_self-bounding}
Under the stochastic regime with adversarial corruption, for any sequence of policies $\{\pi_t\}_{t=1}^T$, the regret satisfies the following $(\Delta,2\calC,T)$ self-bounding constraint:
\begin{equation}
    \Reg_T\geq\E\brk[\Bigg]{\sum_{t=1}^T \sum_{s}\sum_{a\neq \pist(s)} q^{\pi_t}(s,a)\,\Delta(s,a)} - 2\,\calC. \label{eq:adversarial_regime_with_self-bounding_constraints}
\end{equation}
\end{lem}

We now use \Cref{lem:corruption_self-bounding} to prove the following lemma based on the self-bounding argument.
\begin{lem}
\label{general_self_bounding}
Let $G(s,a)$ be any nonnegative function and $J>0$. Suppose that
    \begin{align}
        \Reg_T \lesssim \sum_{s}\sum_{a\neq\pi^\star(s)}G(s,a)\sqrt{\E\brk*{\sum_{t=1}^T q^{\pi_t}(s,a)}} + J.
    \end{align}
Then, under the stochastic regime with adversarial corruption, it holds that
    \begin{align}
        \Reg_T \lesssim U + \sqrt{U\calC} + J
        \quad
        \mbox{for}
        \quad
        U = \sum_{s}\sum_{a\neq\pi^\star(s)} \frac{G(s,a)^2}{\Delta(s,a)}
        .
    \end{align}
\end{lem}
\begin{proof}
By \cref{lem:corruption_self-bounding}, in the stochastic regime with adversarial corruption it holds that $\Reg_T \geq \E\brk*{\sum_{t=1}^T \sum_s \sum_{a\neq\pist(s)}q^{\pi_t}(s, a)\Delta(s,a) } - 2\mathcal{C}$. Then, for any $\alpha\in(0,1/2]$, we obtain
    \begin{align}
        \Reg_T &\leq c\sum_{s}\sum_{a\neq\pi^\star(s)}G(s,a)\sqrt{\E\brk*{\sum_{t=1}^T q^{\pi_t}(s,a)}} + cJ \tag{for some absolute constant $c$}\\
        &\leq  \sum_{s}\sum_{a\neq\pi^\star(s)}G(s,a)\prn*{\frac{\alpha}{G(s,a)}\E\brk*{\sum_{t=1}^T q^{\pi_t}(s, a)\Delta(s,a)} + \frac{c^2G(s,a)}{\alpha\Delta(s,a)}} + cJ \tag{by the AM--GM inequality}\\
        &\leq \alpha\E\brk*{\sum_{t=1}^T\sum_{s}\sum_{a\neq\pi^\star(s)} q^{\pi_t}(s, a)\Delta(s,a)} + \sum_{s}\sum_{a\neq\pist(s)}\frac{c^2G(s,a)^2}{\alpha\Delta(s,a)} + cJ\\
        &\leq \alpha\prn*{\Reg_T + 2\calC} + \sum_{s}\sum_{a\neq\pist(s)}\frac{c^2G(s,a)^2}{\alpha\Delta(s,a)} + cJ.
    \end{align}
    Choosing $\alpha=\min\set*{\frac12,\sqrt{\frac{U}{\calC}}}$ with $U = \sum_s\sum_{a\neq\pi^\star(s)}\frac{G(s,a)^2}{\Delta(s,a)}$ and absorbing the $\alpha\Reg_T$ term yields
    \begin{align}
        \Reg_T \lesssim U + \sqrt{U\calC} + J.
    \end{align}
\end{proof}
\section{Proofs of Regret Lower Bounds (deferred from \Cref{sec:lower_bounds})}
\label{app:lower_bounds}

In this section, we provide complete proofs of the lower bounds stated in Section~\ref{sec:lower_bounds}.
We first establish an information-theoretic lower bound under a convenient stochastic loss model on a layered MDP with uniform transitions (Lemma~\ref{lem:normal_lower_bound_ber}).
We then prove Theorems~\ref{thm:lower_bound1}--\ref{thm:lower_bound3} and Theorem~\ref{thm:lower_bound4}, which are detailed versions of Theorem~\ref{thm:adversarial_data_lb} and Theorem~\ref{thm:variance_lower_bound}, respectively.

Here, we write $\ber(p)$ for the Bernoulli distribution with mean $p$
and $\unif(\calA)$ for the uniform distribution over $\calA$. We also use $\KL(\P,\P')$ to denote the Kullback--Leibler (KL) divergence between distributions $\P$ and $\P'$, and use $\kl(p,q)$ to denote the KL divergence between Bernoulli distributions with means $p$ and $q$.
We also define the regret without expectation given by
\begin{equation}
\Regg_T(\pi) \coloneqq \sum_{t=1}^T V^{\pi_t}(s_0;\ell_t) - \sum_{t=1}^T V^{\pi}(s_0;\ell_t).
\end{equation}
Note that it holds that $\mathsf{Reg}_T = \max_{\pi \in \Pi} \E[\Regg_T(\pi)]$.

\subsection{General Regret Lower Bound for Tabular MDPs}
We first state a general regret lower bound for tabular MDPs due to \citet{zimin2013online,tsuchiya2025reinforcement}. We include the proof to make the constants and the dependence on $H$, $S$, $A$, and $T$ explicit. The hard instance constructed for this lower bound will also be used in the proof of \Cref{thm:variance_lower_bound}.

\begin{lem}[{\citealt{tsybakov2009non}}]
\label{lem: kl-ber}
    Let $p,q \in [0,1]$. Then the KL divergence between Bernoulli distributions with parameters $p, q$ satisfies
    \begin{align}
        \mathrm{kl}(p,q) \leq \chi^2(p,q) = \frac{(p-q)^2}{q(1-q)}.
    \end{align}
\end{lem}

We now consider the following instance of online episodic tabular MDPs to prove a lower bound.
Let $\tilde{\calS} = \calS \setminus \set{s_0}$ (note that, for simplicity, we define the state space $\calS$ to exclude the terminal state $s_H$).
We assume that $(S-1)/(H-1)$ is an integer and that each non-initial layer has the same number of states, i.e.,
$S' := |\mathcal{S}_h| = (S-1)/(H-1)$ for all $h \neq 0,H$.
The non-integral case can be treated by a standard floor/ceiling adjustment, which changes the bounds only by constant factors.
We use
\begin{align}
  N_T(s, a) =
  \sum_{t = 1}^T \frac{1}{S'}\pi_t(a \mid s)
\end{align}
to denote the expected number of times the state-action pair $(s,a) \in \tilde{\calS} \times \calA$ is visited.

We then define the following episodic MDP with stochastic loss models. 
The models are specified as follows:
\begin{itemize}
  \item Transitions occur uniformly at random to states in the next layer. Specifically, for any $(s, a) \in \calS_h \times \calA$, it holds that $P(s' \mid s, a) = 1 / \abs{\calS_{h+1}}$ for all $s' \in \calS_{h+1}$.
  \item All random losses $\ell_t(s,a)$ are assumed to be independent. For policy $\tilpi\in\Pidet$ and $\epsilon\in(0,1/2]$, we consider the following two stochastic loss models:
    \begin{align}
        \ell^{(\tilpi)}_t(s,a)
        &\sim
        \begin{cases}
          \ber(1/2) & \text{if } a = \tilpi(s), \\
          \ber(1/2 + \epsilon) & \text{otherwise},
        \end{cases} \\[6pt]
        \ell^{(\tilpi,\stil)}_t(s,a)
        &\sim
        \begin{cases}
          \ber(1/2) & \text{if } a = \tilpi(s) \text{ and } s \neq \stil, \\
          \ber(1/2 + \epsilon) & \text{otherwise}.
        \end{cases}
    \end{align}
\end{itemize}
These specifications define two episodic MDP instances, denoted by $\calM(\tilpi)$ and $\calM(\tilpi,\stil)$, respectively.

Let $\P_{\tilpi}$ and $\P_{\tilpi,\stil}$ be the probability distribution induced by $\calM(\tilpi)$ and $\calM(\tilpi,\stil)$, respectively.  
We also denote by $\E_{\prn*{\P_{\tilpi}}}[\cdot]$ and $\E_{\prn*{\P_{\tilpi,\stil}}}[\cdot]$ the expectations under the MDPs induced by $\calM(\tilpi)$ and $\calM(\tilpi,\stil)$, respectively.

\begin{lem}
\label{lem:normal_lower_bound_ber}
    Suppose that $H \geq 3$, $A \geq 3$ and $T \geq \frac{SA}{8H}$. Then, for any policy $\{\pi_t\}_{t=1}^T$, there exists $\tilpi\in\Pidet$ such that
\begin{align}
     \max_{\pi \in \Pi} \E_{\prn*{\P_{\tilpi}}}\brk*{ \Regg_T(\pi) } 
    \geq  c\sqrt{HSAT}.
\end{align}
Here the expectation is with respect to $\P_{\tilpi}$ and $c = \frac{\sqrt{2}}{16}\prn*{\frac{1}{2} - \frac{1}{A}}^2$.
\end{lem}

\begin{proof}
We can lower bound the regret under $\calM(\tilpi)$ as
\begin{align}
  \max_{\pi \in \Pi}\E_{\prn*{\P_{\tilpi}}} \brk{\Regg_T(\pi)}
  &\geq \epsilon (H-1)T  - \epsilon \sum_{s \in \tilde{\calS}} \E_{\prn*{\P_{\tilpi}}}  \brk*{ N_T(s, \tilpi(s)) }\\
  &= \frac{\epsilon T}{S'} \prn*{S'(H-1) - \frac{S'}{T}\sum_{s \in \tilde{\calS}} \E_{\prn*{\P_{\tilpi}}}  \brk*{ N_T(s, \tilpi(s)) }}\\
  &= \frac{\epsilon T}{S'} \prn*{S-1 - \frac{S'}{T}\sum_{s \in \tilde{\calS}} \E_{\prn*{\P_{\tilpi}}}  \brk*{ N_T(s, \tilpi(s)) }}.
  \label{eq:regret_decompose_lb}
\end{align}

In what follows, we will upper bound $\E_{\P_{\tilpi}} \brk*{ N_T(s, \tilpi(s)) }$.
Note that the only difference between $\calM(\tilpi)$ and $\calM(\tilpi, \stil)$ lies in the expected value of the loss at the state-action pair $(\stil, \tilpi(\stil))$.

Then, using the fact that $\frac{S'}{T} N_T(s, \tilpi(s)) \in [0,1]$ for $s \neq s_0$ and Pinsker's inequality,
for any $\stil \in \calS \setminus \set{s_0}$ we have
\begin{align}
  \frac{S'}{T} \E_{\prn*{\P_{\tilpi}}}  \brk*{ N_T(\stil, \tilpi(\stil)) }
  &\leq \frac{S'}{T} \E_{\prn*{\P_{\tilpi, \stil}}} \brk*{ N_T(\stil, \tilpi(\stil)) } + \|\P_{\tilpi,\stil} - \P_{\tilpi}\|_{\mathrm{TV}} \\
  &\leq \frac{S'}{T} \E_{\prn*{\P_{\tilpi, \stil}}} \brk*{ N_T(\stil, \tilpi(\stil)) } + \sqrt{\frac{1}{2}\KL(\P_{\tilpi,\stil},\P_{\tilpi})}
  .
  \label{eq:bretagnolle_huber}
\end{align}
Then, from the chain rule of the KL divergence,
we can evaluate the KL divergence in the last inequality as
\begin{align}
  \KL(\P_{\tilpi,\stil},\P_{\tilpi}) 
  &= \E_{\prn*{\P_{\tilpi, \stil}}} \brk*{N_T(\stil, \tilpi(\stil))} \, \mathrm{kl}(1/2 + \epsilon, 1/2)\\
  &\leq 4\epsilon^2 \E_{\prn*{\P_{\tilpi, \stil}}} \brk*{N_T(\stil, \tilpi(\stil))} 
  .
  \label{eq:kl_chain_ber}
\end{align}
where we used \cref{lem: kl-ber}.
Taking the uniform average over $\Pidet$ for the RHS of \cref{eq:kl_chain_ber}, for any $\stil \in \tilde{\calS}$ we have
\begin{align}
    &\E_{\tilpi \sim \unif(\Pidet)} \brk*{\E_{\prn*{\P_{\tilpi, \stil}}} \brk*{N_T(\stil, \tilpi(\stil))}}\\
    &= \sum_{a \in \calA}\Pr\brk{ \tilpi(\stil) = a } \E_{\tilpi \sim \unif(\Pidet)}\brk*{ \E_{\prn*{\P_{\tilpi, \stil}}}\brk*{N_T(\stil, \tilpi(\stil))}\middle|\tilpi(\stil) = a}\\
    &=\frac{1}{A}\sum_{a \in \calA}\E_{\tilpi \sim \unif(\Pidet)}\brk*{\E_{\prn*{\P_{\tilpi, \stil}}}\brk*{N_T(\stil, a)}}\\
    &= \frac{T}{S' A},
    \label{eq:unif_Pidet_N}
\end{align}
where the last equality follows from the definition of $N_T$.
By summing over $\stil \in \tilde{\calS}$ in \cref{eq:unif_Pidet_N},
\begin{equation}\label{eq:unif_pi_sum_s_N}
  \sum_{\stil \in \tilde{\calS}}\E_{\tilpi \sim \unif(\Pidet)} \brk*{\E^{(\tilpi,\stil)} \brk*{N_T(\stil, \tilpi(\stil))}}
  =\frac{(S-1)T}{S' A} 
  \leq\frac{ST}{S'A}
  .
\end{equation}
Using the last inequality, we also have
\begin{align}
  &\sum_{\stil \in \tilde{\calS}}\E_{\tilpi \sim \unif(\Pidet)}\brk*{\sqrt{\frac{1}{2} \KL(\P_{\tilpi,\stil},\P_{\tilpi})}}\\
  &\leq \epsilon\sum_{\stil \in \tilde{\calS}}\E_{\tilpi \sim \unif(\Pidet)}\brk*{\sqrt{2 \E_{\prn*{\P_{\tilpi, \stil}}} \brk*{N_T(\stil, \tilpi(\stil))}}}\\
  &\leq \epsilon\sqrt{2(S - 1)\sum_{\stil \in \tilde{\calS}}  \E_{\tilpi \sim \unif(\Pidet)}  \brk* { \E_{\prn*{\P_{\tilpi, \stil}}} \brk*{N_T(\stil, \tilpi(\stil))}}}\\
  &\leq \epsilon\sqrt{\frac{2HST}{A}}
  ,
  \label{eq:lb_upper_ber1}
\end{align}
where the first inequality follows from \cref{eq:kl_chain_ber},
the second inequality follows from the Cauchy--Schwarz inequality and Jensen's inequality,
and the last inequality follows from \cref{eq:unif_pi_sum_s_N}.

Therefore, by \cref{eq:bretagnolle_huber,eq:unif_pi_sum_s_N,eq:lb_upper_ber1},
\begin{align}
    \sum_{\stil \in \tilde{\calS}}\frac{S'}{T} \E_{\tilpi \sim \unif(\Pidet)} \brk*{\E_{\P_{\tilpi}} \brk*{ N_T(\stil, \tilpi(\stil)) }} \leq \frac{S}{A} + \epsilon\sqrt{\frac{2HST}{A}}
    .
    \label{eq:lb_upper_ber2}
\end{align}
Finally, combining everything together, we have
\begin{align}
  \max_{\tilpi \in \Pidet}\set*{\max_{\pi \in \Pi} \E_{\prn*{\P_{\tilpi}}} \brk*{\Regg_T(\pi)}}
  &\geq\E_{\tilpi \sim \unif(\Pidet)}\brk*{ \set*{\max_{\pi \in \Pi}\E_{\prn*{\P_{\tilpi}}} \brk*{\Regg_T(\pi)}}}\\
  &\geq \frac{\epsilon T}{S'}\prn*{\frac{S}{2} - \frac{S'}{T} \sum_{s \in \tilde{\calS}}\E_{\tilpi \sim \unif(\Pidet)} \brk*{ \E_{\prn*{\P_{\tilpi}}} \brk*{ N_T(s, \tilpi(s))}}}
  \tag{by \cref{eq:regret_decompose_lb} and $S - 1 \geq \frac{S}{2}$ when $H \geq 3$} \\
  &\geq \frac{\epsilon T}{S'}\prn*{\frac{S}{2} - \frac{S}{A} - \epsilon \sqrt{\frac{2HST}{A}}} \tag{by \cref{eq:lb_upper_ber2}}\\
  &\geq \frac{\epsilon HT}{2} \prn*{\frac{1}{2} - \frac{1}{A} - \epsilon \sqrt{\frac{2HT}{SA}}} 
  .
  \tag{$\frac{S}{S'} \geq \frac{H}{2}$ when $H \geq 3$}
\end{align}
Choosing the optimal $\epsilon = (\frac{1}{4}-\frac{1}{2A})\sqrt{\frac{SA}{2HT}}$, which lies in $(0, 1/2)$ whenever $A \geq 3$ and $T \geq \frac{SA}{8H}$, we obtain

\begin{align}
    \max_{\tilpi \in \Pidet}\set*{\max_{\pi \in \Pi} \E_{\prn*{\P_{\tilpi}}} \brk*{\Regg_T(\pi)}} \geq c\sqrt{HSAT},
\end{align}
where $c = \frac{\sqrt{2}}{16}\prn*{\frac{1}{2} - \frac{1}{A}}^2$.
\end{proof}
This lower bound is useful for deriving the regret lower bound
$\max_{\pi \in \Pi}\E\brk*{\Regg_T(\pi)}\geq\Omega\prn{\sqrt{HSAT}}$
for online episodic tabular MDPs with adversarial losses.

\subsection{Proof of \Cref{thm:adversarial_data_lb}}
Here we provide the proof of \Cref{thm:adversarial_data_lb}.

\begin{thm}[First-order lower bound]
    \label{thm:lower_bound1}
    Suppose that $H \geq 3$, $A \geq 3$, $T \geq \frac{SA}{8H}$ and $\alpha \in \brk*{\frac{\ceil*{SA/8H}}{T}, 1}$.
    Then, for any policy $\{\pi_t\}_{t=1}^T$, there exists an episodic MDP with adversarial losses satisfying
    \begin{align}
        L^\star = \min_{\pi \in \Pi} \sum_{t=1}^T V^\pi(s_0;\ell_t) \leq \alpha HT
    \end{align}
    such that
     \begin{align}
        \max_{\pi \in \Pi}\E\brk{\Regg_T(\pi)} \geq  \Omega\prn{\sqrt{\alpha HSAT} } =  \Omega\prn{\sqrt{SAL^\star}}.
    \end{align}
\end{thm}
\begin{proof}
    Fix any $\alpha \in \brk*{\frac{\ceil*{SA/8H}}{T}, 1}$ and split the horizon into
    an active phase $t = 1,\dots,\floor{\alpha T}$ and an inactive phase $t = \floor{\alpha T} + 1,\dots,T$.
    In the inactive phase, we set all losses to zero. As a result, it holds that
    \begin{align}
        L^\star \leq \alpha HT.
    \end{align}
    On the other hand, by applying \cref{lem:normal_lower_bound_ber} to the active phase, we obtain
    \begin{align}
        \max_{\pi\in\Pi}\E\brk{\Regg_T(\pi)} \geq  \Omega\prn{\sqrt{\alpha HSAT}} + 0 = \Omega\prn{\sqrt{\alpha HSAT}}.
    \end{align}
\end{proof}
\begin{thm}[Second-order lower bound]
    \label{thm:lower_bound2}
    Suppose that $H \geq 3$, $A \geq 3$, $T \geq \frac{SA}{8H}$ and $\alpha \in \brk*{\frac{\ceil*{SA/8H}}{T}, 1}$.
    Then, for any policy $\{\pi_t\}_{t=1}^T$, there exists an episodic MDP with adversarial losses satisfying
    \begin{align}
        Q_\infty = \min_{\ell^\star\in[0,1]^{S\times A}} \sum_{t = 1}^T\sum_{h = 0}^{H-1} \|\ell_t(h) - \ell^\star(h)\|_\infty^2 \leq \alpha HT
    \end{align}
    such that
     \begin{align}
        \max_{\pi\in\Pi}\E\brk{\Regg_T(\pi)} \geq  \Omega\prn{\sqrt{\alpha HSAT} } =  \Omega\prn{\sqrt{SAQ_\infty }}.
    \end{align}
\end{thm}
\begin{proof}
    Fix any $\alpha \in \brk*{\frac{\ceil*{SA/8H}}{T}, 1}$ and split the horizon into
    an active phase $t = 1,\dots,\floor{\alpha T}$ and an inactive phase $t = \floor{\alpha T} + 1,\dots,T$.
    In the inactive phase, we set all losses to zero. As a result, it holds that
    \begin{align}
        Q_\infty = \min_{\ell^\star\in[0,1]^{S\times A}} \sum_{t = 1}^T\sum_{h = 0}^{H-1} \|\ell_t(h) - \ell^\star(h)\|_\infty^2 \leq \sum_{t = 1}^T\sum_{h = 0}^{H-1} \|\ell_t(h)\|_\infty^2
        \leq \alpha HT.
    \end{align}
    On the other hand, by applying \cref{lem:normal_lower_bound_ber} to the active phase, we obtain
    \begin{align}
        \max_{\pi\in\Pi}\E\brk{\Regg_T(\pi)} \geq  \Omega\prn{\sqrt{\alpha HSAT}} + 0 = \Omega\prn{\sqrt{\alpha HSAT}}.
    \end{align}
\end{proof}

\begin{thm}[Path-length lower bound]
    \label{thm:lower_bound3}
    Suppose that $H \geq 3$, $A \geq 3$, $T \geq \frac{SA}{8H}$ and $\alpha \in \brk*{\frac{\ceil*{SA/8H}}{T}, 1}$.
    Then, for any policy $\{\pi_t\}_{t=1}^T$, there exists an episodic MDP with adversarial losses satisfying
    \begin{align}
        V_1 = \E\brk*{\sum_{t = 1}^{T-1}\nrm*{\ell_{t + 1} - \ell_t}_1} \leq \alpha SAT
    \end{align}
    such that
     \begin{align}
        \max_{\pi\in\Pi}\E\brk{\Regg_T(\pi)} \geq  \Omega\prn{\sqrt{\alpha HSAT} } =  \Omega\prn{\sqrt{HV_1}}.
    \end{align}
\end{thm}
\begin{proof}
    Fix any $\alpha \in \brk*{\frac{\ceil*{SA/8H}}{T}, 1}$ and split the horizon into
    an active phase $t = 1,\dots,\floor{\alpha T}$ and an inactive phase $t = \floor{\alpha T} + 1,\dots,T$.
    In the inactive phase, we set all losses to zero. As a result, it holds that
    \begin{align}
        V_1 = \E\brk*{\sum_{t = 1}^{T-1}\nrm*{\ell_{t + 1} - \ell_t}_1}
        \leq \alpha SAT.
    \end{align}
    On the other hand, by applying \cref{lem:normal_lower_bound_ber} to the active phase, we obtain
    \begin{align}
        \max_{\pi\in\Pi}\E\brk{\Regg_T(\pi)} \geq  \Omega\prn{\sqrt{\alpha HSAT}} + 0 = \Omega\prn{\sqrt{\alpha HSAT}}.
    \end{align}
\end{proof}

\subsection{Proof of \Cref{thm:variance_lower_bound}}
Here we provide the proof of \Cref{thm:variance_lower_bound}.

\begin{thm}
\label{thm:lower_bound4}
Suppose that $H \geq 3$, $A \geq 3$, $T \geq \frac{SA}{8H}$, and $\alpha \in (0, \frac{1}{4}]$.
Then, for any policy $\{\pi_t\}_{t=1}^T$, there exists an episodic MDP with stochastic losses satisfying
\begin{align}
    \V \coloneq \max_\pi\E\brk*{\sum_{s,a} q^{\pi}(s,a)\sigma^2(s,a)} \leq \alpha H
\end{align}
such that
 \begin{align}
    \max_{\pi\in\Pi}\E\brk{\Regg_T(\pi)} \geq \Omega\prn{\sqrt{\alpha HSAT}} = \Omega\prn{\sqrt{SA \V T}}.
\end{align}
\end{thm}

\begin{proof}
Let $\beta = 2\sqrt{\alpha} \in (0,1]$. We take the hard instance $\calM(\tilpi)$ in Lemma~\ref{lem:normal_lower_bound_ber} and scale all losses by $\beta$. We define
\begin{align}
    \Lambda^{(\tilpi)}_t(s,a) = \beta \ell^{(\tilpi)}_t(s,a).
\end{align}
Let $\calM^{(\beta)}(\tilpi)$ be the MDP obtained from $\calM(\tilpi)$ by replacing each loss $\ell^{(\tilpi)}_t(s,a)$ with $\Lambda^{(\tilpi)}_t(s,a)$ while keeping the transition kernel unchanged. Let $\P_{\tilpi}^{(\beta)}$ be the probability distribution induced by $\calM^{(\beta)}(\tilpi)$. We also denote by $\E_{\P_{\tilpi}^{(\beta)}}[\cdot]$ the expectation under the MDP induced by $\calM^{(\beta)}(\tilpi)$.

Then, for any state-action pair $(s,a)$, we have
\begin{align}
    \Var\left(\Lambda^{(\tilpi)}_t(s,a)\right)
    =\beta^2 \Var\left(\ell^{(\tilpi)}_t(s,a)\right)
    \leq\frac{\beta^2}{4}
    =\alpha.
\end{align}
Therefore, the occupancy-weighted variance satisfies
\begin{align}
    \V
    &= \max_{\pi}\E_{\P_{\tilpi}^{(\beta)}}\brk*{\sum_{s,a} q^{\pi}(s,a)\sigma^2(s,a)} 
    \leq \alpha\max_{\pi}\E_{\P_{\tilpi}^{(\beta)}}\brk*{\sum_{s,a} q^\pi(s,a)} = \alpha H .
    \label{eq:var_condition_beta}
\end{align}
Thus, the constructed instance satisfies the variance condition $\V \leq \alpha H$.

Finally, since scaling all losses by $\beta$ scales the regret by the same factor, Lemma~\ref{lem:normal_lower_bound_ber} implies that there exists $\tilpi\in\Pidet$ such that
\begin{align}
    \max_{\pi\in\Pi}\E_{\P_{\tilpi}^{(\beta)}}\brk*{\Regg_T(\pi)}
    \geq 2c\sqrt{\alpha HSAT},
\end{align}
where $c = \frac{\sqrt{2}}{16}\left(\frac{1}{2}-\frac{1}{A}\right)^2$.

Finally, using $\alpha H \geq \V$ from \cref{eq:var_condition_beta}, we have
\begin{align}
    \max_{\pi\in\Pi}\E_{\P_{\tilpi}^{(\beta)}}\brk*{\Regg_T(\pi)}
    \geq 2c\sqrt{\alpha HSAT}
    \geq 2c\sqrt{SA\V T}.
\end{align}
This completes the proof.
\end{proof}

\end{document}